\documentclass{article}
\usepackage[preprint]{neurips_2026}
\makeatletter
\renewcommand{\@noticestring}{Preprint. Under review.}
\makeatother
\usepackage[utf8]{inputenc}
\usepackage[T1]{fontenc}
\usepackage{newtxtext}
\usepackage{hyperref}
\hypersetup{colorlinks=true, linkcolor=black, citecolor=black, urlcolor=black}
\usepackage{url}
\usepackage{booktabs}
\usepackage{amsfonts}
\usepackage{amsmath}
\usepackage{amssymb}
\usepackage{amsthm}

\usepackage{newtxmath}
\usepackage{nicefrac}
\usepackage{microtype}
\usepackage{xcolor}
\usepackage{graphicx}
\usepackage{multirow}
\usepackage{makecell}
\usepackage{array}
\usepackage{adjustbox}
\usepackage{enumitem}
\usepackage{pifont}
\usepackage{threeparttable}

\newcolumntype{C}[1]{>{\centering\arraybackslash}p{#1}}
\newcolumntype{L}[1]{>{\raggedright\arraybackslash}p{#1}}

\newtheorem*{thmpartition}{Theorem 2A (Partition variance gap)}
\newtheorem*{thmtwoprime}{Theorem 2$'$ (Minimality of the metadata basis)}
\newtheorem*{thmfour}{Theorem 4 (Calibration feasibility)}

\title{Fortunate Recall: Ontology-Driven Memory Lifecycle Management\\for Persistent Coherence in LLMs}

\author{
  Ansuman Mullick \\
  Bilkent University \\
  \And
  Eray T\"uz\"un \\
  Bilkent University \\
}

\begin{document}
\maketitle

\begin{abstract}
Current LLM memory systems treat all personal facts identically, storing them in flat vector stores or knowledge graphs with uniform retention, so stores grow unboundedly while retrieval precision degrades. The core challenge is lifecycle management: which memories should persist, which should be replaced, and at what rate, conditioned on the behavioral type of each fact. \textbf{Fortunate Recall} (FR) is a composable policy layer that classifies personal facts into a 10+1 behavioral ontology and applies category-specific lifecycle policies (differential temporal decay, slot-key supersession, event-time validity, and category-aware retrieval routing) as deterministic functions over LLM-extracted metadata. The lifecycle policy layer is pure math; the LLM is used at ingestion and for a single lightweight candidate-selection step at retrieval. FR-Bank, our infrastructure-independent implementation, achieves 76.9\% pass rate on \textbf{LifecycleBench}, a new 516-question temporal-disambiguation benchmark, outperforming Mem0 (61\%), A-MEM (65.3\%), Memory-R1 (66.9\%), and MemoryOS (70.5\%) across seven system configurations, and reaches 75.2\% on the full 500-question LongMemEval-S benchmark under the canonical Wu et al.\ (ICLR 2025) judge protocol, demonstrating that lifecycle policies impose no measurable aggregate cost on standard retrieval tasks. A pre-registered ablation locates where the gains originate. Replacing the typed layer with three generic lifecycle primitives over a semantic-only ranker leaves correctness statistically unchanged ($\Delta = -1.7$pp, 95\% CI $[-6.0, +2.7]$, n.s.), so the \emph{generic lifecycle metadata} carries the correctness advantage; what the behavioral ontology carries is calibration, halving downstream confabulation (12.0\% vs.\ 24.2\% over all queries, $p < 0.001$) by making per-category parameterization feasible where no single global weight is. End-to-end, FR-Bank reduces confabulation from Mem0's 45.1\% to 22.4\% over answered queries and from 32.2\% to 13.0\% over all 516 queries, while answering more of them correctly (31.2\% vs.\ 18.6\%); the hierarchy replicates on the open-weight Kimi K2.5 generator, and upgrading Mem0's extractor to gpt-4.1-mini narrows the retrieval gap without touching its structural deficits on supersession and retraction. The decomposition transfers to \textbf{BEAM} (ICLR 2026), a benchmark constructed independently of this work: 46.8\% correct vs.\ Mem0's 32.9\% over 280 questions, splitting into $+23$ to $+26$ from the generic metadata and $+10$ to $+13$ from the ontology, concentrated in contradiction resolution. A six-point granularity sweep places the ontology's benefit on a plateau from roughly seven effective policy clusters, so 10+1 is an interpretable point on that plateau rather than a uniquely necessary granularity. The ontology, benchmark, and code are released as supplementary material and will be made publicly available upon publication.
\end{abstract}

\section{Introduction}
\label{sec:intro}

Large language models lose coherence over extended conversations because existing memory systems treat all personal facts uniformly. A user's ethnicity, current food preference, and six-month-old dentist appointment coexist in the same store with identical retrieval priority, yet their appropriate persistence dynamics differ by orders of magnitude. The core challenge is not retrieval but \emph{lifecycle management}: determining which memories should persist, which should be replaced, and at what rate, conditioned on the behavioral type of each fact.

Building on early work on LLM agent memory~\cite{genagents2023}, recent systems have established strong infrastructure: temporally-aware knowledge graphs (Zep/Graphiti~\cite{zep2025}), Zettelkasten organization (A-MEM~\cite{amem2025}), OS-inspired hierarchies (MemoryOS~\cite{memoryos2025}), RL-learned operations (Memory-R1~\cite{memoryr1_2025}), and adaptive structure selection (FluxMem~\cite{fluxmem2026}). None distinguishes a current preference from a superseded one, recognizes expired logistics, or models an approaching deadline, and existing benchmarks~\cite{longmemeval2024} likewise do not measure whether a system can tell current from outdated state.

We introduce \textbf{Fortunate Recall} (FR), a composable lifecycle policy layer that classifies each extracted fact into one of 10+1 behavioral categories (Table~\ref{tab:ontology})---not cognitive types (episodic/semantic/procedural) but behavioral domains that directly determine temporal dynamics---and applies category-specific decay, supersession, event-time validity, and routing policies as deterministic functions. Every retrieval decision is a named term of a single closed-form log-score, with no learned black box mediating lifecycle behavior, at a median $47\,\mu\mathrm{s}$ and empirically linear $O(k)$ scaling.

\textbf{Contributions.} (1)~A behavioral ontology derived from how personal facts change over time, with per-category deterministic lifecycle policies (Table~\ref{tab:ontology}). (2)~A lifecycle non-identifiability theorem with sufficiency, minimality, and individual-necessity results for the metadata basis $(c, \kappa, \xi, h)$; these are representation-level results about the \emph{basis} and do not claim the 10+1 partition, the rates, or the pipeline are optimal (\S\ref{sec:theory} scopes each claim). (3)~\textbf{LifecycleBench}, a 516-question temporal-disambiguation benchmark over 40 personas with 9 attack vectors, positioned as a dense lifecycle stress test rather than the first benchmark to touch temporal updates. (4)~\textbf{FR-Bank}, the infrastructure-independent reference implementation, at $76.9\%$ on LifecycleBench and $75.2\%$ on full LongMemEval-S under the Wu et al.\ protocol, with end-to-end validation across seven configurations and two generator families under both the answered-query and all-query denominators. (5)~\textbf{An attribution and transfer study} separating the stack's two layers: a pre-registered untyped ablation (\S\ref{sec:attribution}), transfer to the independently built BEAM benchmark (\S\ref{sec:beam}), a six-point granularity sweep, a full metadata-noise grid, and two machine audits of our own judge. Ontology, benchmark, code, pre-registration files, and raw per-question judgments are released and will be made public upon publication.

\section{Related Work and Gap Analysis}
\label{sec:gaps}

Existing systems exhibit complementary gaps, mapped feature by feature in Table~\ref{tab:gap-mapping} (Appendix~\ref{app:gaps}); we summarize the pattern rather than enumerate it. Flat vector stores (Mem0~\cite{mem0_2025}) and dynamic organizers (A-MEM~\cite{amem2025}) provide no structured lifecycle. Cognitive-type categorizations (MemoryOS~\cite{memoryos2025}, MIRIX~\cite{mirix2025}) describe memory \emph{form} rather than \emph{behavior}: a chronic illness and a lunch order are both ``semantic'' yet need very different dynamics. Memory-R1~\cite{memoryr1_2025} learns operations by RL over a flat \{ADD, UPDATE, DELETE\} space containing no expiry or retraction action, so its gap is vocabulary rather than optimization. MemoryBank~\cite{memorybank2024} applies Ebbinghaus~\cite{ebbinghaus1885} curves with one uniform decay for all fact types. MemGPT~\cite{memgpt2024} and, at query time, APEX-MEM~\cite{banerjee2026apex} delegate lifecycle decisions to LLM reasoning---expensive and opaque where FR's policies are deterministic---and neither types facts, applies differential decay, or models event-time validity. Zep/Graphiti~\cite{zep2025} offers production temporal knowledge graphs with no forgetting mechanism and only binary supersession; FluxMem~\cite{fluxmem2026} adapts memory \emph{structure} rather than \emph{policy}. FR's bi-temporal representation parallels the valid-time / transaction-time distinction in temporal databases~\cite{snodgrass2000}, extended with behavioral conditioning.

\textbf{Benchmarks.} LongMemEval~\cite{longmemeval2024} and LoCoMo~\cite{locomo2024} measure long-horizon retrieval without testing whether a system distinguishes current from outdated state, and BEAM~\cite{beam2026} spans ten memory abilities at contexts up to 10M tokens, several of them lifecycle-sensitive. We therefore position LifecycleBench not as the first benchmark to touch temporal updates or forgetting but as a \emph{dense lifecycle stress test}, every question constructed so the correct answer depends on lifecycle state. Since we introduce both benchmark and method, \S\ref{sec:beam} transfers the same comparison onto BEAM, which we did not build.

\textbf{Distinction from cognitive typologies.} MemoryOS and MIRIX~\cite{mirix2025} categorize by cognitive type (episodic/semantic/procedural), which describes the \emph{form} of memory. Our ontology describes the \emph{behavioral domain}: a parent's chronic illness is ``semantic'' in cognitive typology but ``Health \& Wellbeing'' in ours; a meeting time is also ``semantic'' but ``Logistical Context.'' The cognitive type does not determine appropriate lifecycle dynamics; the behavioral domain does.

\section{Fortunate Recall}
\label{sec:architecture}

\begin{figure}[t]
  \centering
  \includegraphics[width=\textwidth]{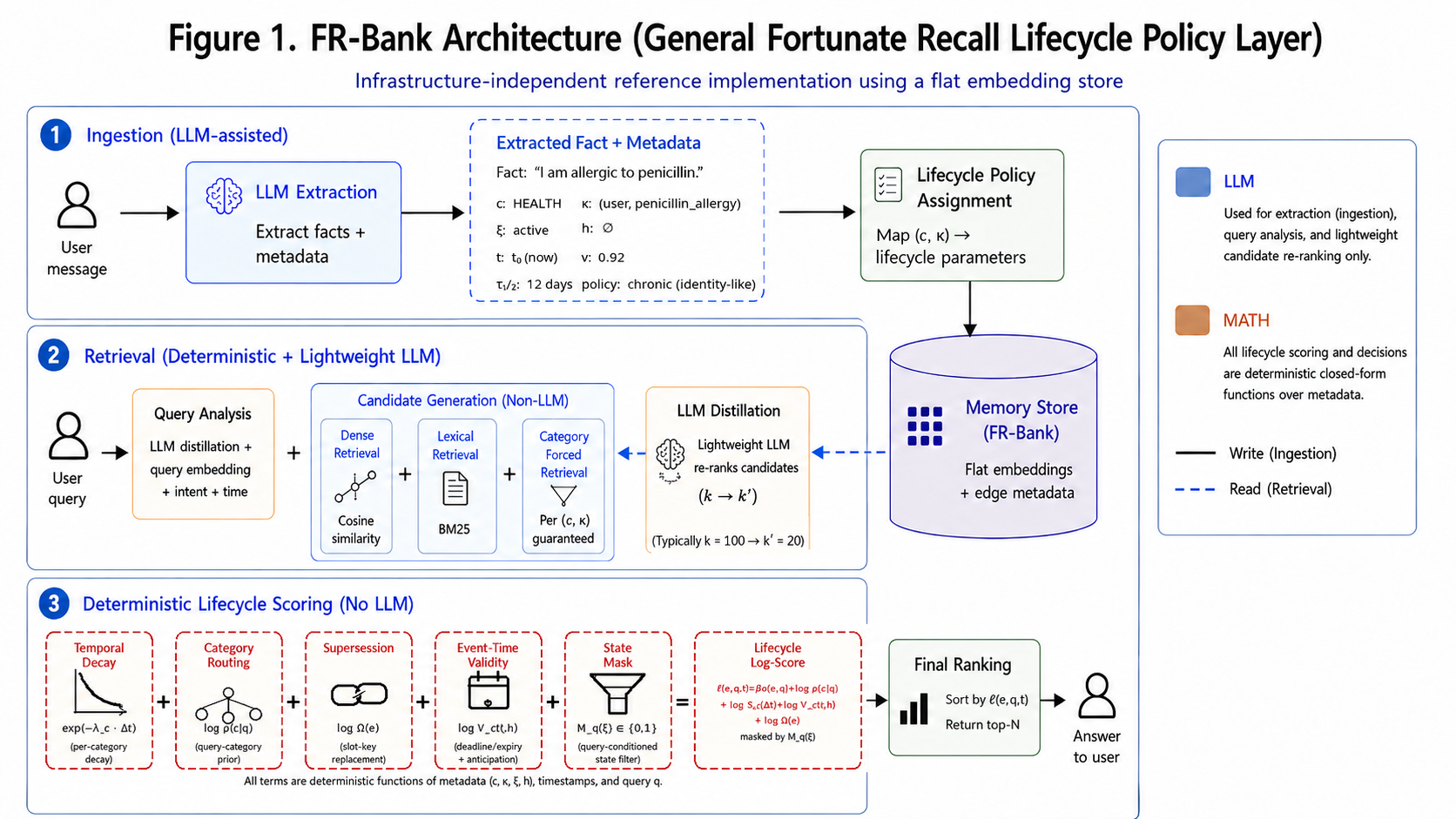}
  \caption{\textbf{Fortunate Recall architecture and the LLM / lifecycle-math boundary.} \emph{Ingestion} (top, asynchronous): one \textsc{gpt-4.1-mini} call per turn extracts each edge $e$ and tags it with behavioral category $c$, slot key $\kappa$, lifecycle state $\xi$, event-time anchor $h$, ingest time $t$, and confidence $v$; $t$ and $h$ are two independent clocks (bi-temporal). \emph{Retrieval} (bottom, per query): hybrid candidate generation is distilled by a single small-model call to top-20, which the \emph{deterministic lifecycle layer} ranks by the closed-form log-score of Eq.~\eqref{eq:logscore}. The pipeline uses exactly two LLM calls per turn, and every lifecycle decision is a named term of $\ell$, inspectable per edge and per term at millisecond latency.}
  \label{fig:architecture}
\end{figure}

The architecture (Figure~\ref{fig:architecture}) comprises an infrastructure layer adopting established patterns (bi-temporal validity, hybrid retrieval) and a lifecycle policy layer providing behavioral dynamics.

\subsection{The Behavioral Ontology}

Each extracted fact is classified into one of 10+1 behavioral categories (Table~\ref{tab:ontology}), by how it changes over time rather than by cognitive form. The category determines the fact's full lifecycle---decay rate, supersession semantics, expiry logic, and retrieval routing---so that a chronic illness and a lunch order, both ``semantic'' in cognitive typology, receive the policies their temporal behavior demands.

\textbf{Empirically grounded design.} An ``identity gravity well'' in our original 8+1 design (Identity absorbing ${\sim}$38\% of all facts) motivated the three-way split into Identity, Hobbies, and Preferences; emotional state is excluded as a category because mood modulates other categories rather than constituting a memory type. Facts are soft-clustered with membership weights summing to 1.0, so the effective decay rate is the weighted harmonic mean of category-specific rates (Lemma~1; soft-membership statistics, the granularity sweep, and the within-category variance comparison are in Appendix~\ref{app:ontology-empirical}, with Figure~\ref{fig:mechanisms} in Appendix~\ref{app:architecture-ext} illustrating supersession and event-time validity end-to-end).

\begin{table}[t]
\caption{The 10+1 behavioral ontology. Each category determines lifecycle dynamics: decay rate (half-life), supersession semantics, and expiry logic. Half-lives govern base decay; actual persistence also depends on access frequency and soft category membership (Appendix~\ref{app:architecture-ext}).}
\label{tab:ontology}
\centering
\setlength{\tabcolsep}{3pt}
\renewcommand{\arraystretch}{1.05}
\footnotesize
\begin{adjustbox}{max width=\textwidth}
\begin{tabular}{@{}c l L{3.4cm} c L{4.8cm}@{}}
\toprule
\textbf{\#} & \textbf{Category} & \textbf{Key Policy} & \textbf{Half-life} & \textbf{Example Transition} \\
\midrule
1  & Identity \& Self-Concept   & Accumulate; high-confidence supersession & 19 days & ``Works at Google'' $\to$ ``Joined Anthropic'' \\
2  & Relational Bonds           & Accumulate; explicit dissolution         & 19 days & ``Dating Alex'' $\to$ ``Broke up with Alex'' \\
3  & Intellectual Interests     & Accumulate; reactivatable                & 14 days & Dormant interest resurfaces after months \\
4  & Health \& Wellbeing        & Chronic: identity-like; Acute: expire    & 12 days & Flu resolves; diabetes persists \\
5  & Projects \& Endeavors      & State machine: active/done/paused        & 12 days & ``Writing thesis'' $\to$ ``Thesis defended'' \\
6  & Hobbies \& Recreation      & Slow decay; dormancy detection           & 8 days  & No fishing mentions for 6 months \\
7  & Preferences \& Habits      & Slot-key supersession                    & 6 days  & ``Likes pizza'' superseded by ``Likes sushi'' \\
8  & Financial \& Material      & State tracking; supersession             & 5 days  & ``Rent \$950/mo'' $\to$ ``Rent \$1{,}100/mo'' \\
9  & Obligations                & Event-time validity; anticipatory        & 5 days  & Deadline activates as it approaches \\
10 & Logistical Context         & Fastest decay; event-time expiry         & 4 days  & ``Flight at 6am tomorrow'' $\to$ expires \\
11 & Other / Open-Set           & Moderate default                         & 10 days & Facts resisting classification \\
\bottomrule
\end{tabular}
\end{adjustbox}
\end{table}

\subsection{Classify the Fact, Not the Entity}

Classification must target the \emph{fact} (the relational edge), not the conversational utterance or the entity node. Utterance-level classification confuses the conversational frame with the information being memorized (``I changed my last name while updating my insurance paperwork'' is about Identity, not Obligations). Entity-level classification assigns wrong lifecycle policies to edges that happen to involve the same node (``Alex caught 7 bass'' should be Hobbies, not Relational just because Alex is a person). Fact-level classification ensures each edge receives its category based on what it represents. A reproducible three-judge classifier-agreement audit ($\kappa=+0.67$ with full session context) is in Appendix~\ref{app:classifier-audit}.

\subsection{Lifecycle Mechanisms}

\textbf{Differential temporal decay.} Each category has a base decay rate $\lambda_c$ with activation $a(e) = \exp(-\lambda_c \cdot \Delta t)$. Rates span a $5.3\times$ range from Identity (slowest) to Logistical (fastest), calibrated to prevent category hierarchy squatting (Appendix~\ref{app:squatting}).

\textbf{Slot-key supersession with confidence weighting.} For categories with \emph{replace} semantics, normalized (subject, attribute) slot keys trigger supersession. High-confidence contradictions mark old edges inactive; low-confidence cases preserve both values (soft supersession). ``Thinking about moving to London'' preserves ``lives in Istanbul''; ``moved to London'' supersedes it (Appendix~\ref{app:slotkey}).

\textbf{Event-time validity and anticipatory activation.} For Obligations and Logistics, activation is governed by distance to the event, not from creation. Deadlines \emph{increase} in activation as they approach, then expire after passing. A date-aware filter with backward-looking bypass removes past-date events before reranking.

\textbf{Category-aware retrieval routing.} Query classification pulls per-category candidate sets alongside global semantic candidates. \emph{Category-forced retrieval} extends this: for the predicted category, all edges are retrieved (up to 20), bypassing vocabulary gaps between abstract queries and specific stored facts.

\subsection{Core Mechanism versus Implementation Engineering}
\label{sec:mechanism-table}

FR-Bank is a full retrieval stack, and not every part of it is a lifecycle mechanism. Table~\ref{tab:mechanism-vs-engineering} separates the two. The dividing line is measured rather than asserted: the ``core mechanism'' rows are exactly the components that the untyped lifecycle baseline of \S\ref{sec:attribution} retains or removes, so the measured difference between those two configurations is what isolates the classification layer's contribution.

\begin{table}[t]
\caption{Core FR mechanisms versus implementation engineering. Every engineering row is independently ablated; every core row is grounded in a theoretical result, a measured ablation, or both.}
\label{tab:mechanism-vs-engineering}
\centering
\footnotesize
\setlength{\tabcolsep}{4pt}
\begin{tabular}{@{}L{4.4cm} l L{6.6cm}@{}}
\toprule
\textbf{Component} & \textbf{Class} & \textbf{Grounding} \\
\midrule
Behavioral classification ($c$)                  & Core      & Thm~3 (\S\ref{app:necessity-c}); Thm~4 (enables per-category calibration); measured: confabulation halved, $+$CR on BEAM \\
Slot-key supersession ($\kappa, \Omega$)         & Core      & Thm~1 construction; Thm~3 (\S\ref{app:necessity-kappa}) \\
Lifecycle-state mask ($\xi, M_q$)                & Core      & Thm~3 (\S\ref{app:necessity-xi}); Thm~3$''$ \\
Event-time validity kernel ($h, V_c$)            & Core      & Thm~3 (\S\ref{app:necessity-h}); Thm~3$'$; Prop.~2 \\
Category-conditioned decay ($S_c$)               & Core      & Model~1; ablation $+6$pp (Table~\ref{tab:lifecyclebench-results}) \\
Category-aware routing ($\rho$)                  & Core      & Ablation $+6$pp; scale-emergent (Appendix~\ref{app:e2e}) \\
\midrule
Semantic floors                                  & Engineering & Ablated, Table~\ref{tab:ablation-summary} (Hit@1 $31\%\to19\%$ without) \\
Numeric rescue                                   & Engineering & Ablated, Appendix~\ref{app:numeric} (AV8 $81\%\to94\%$) \\
LLM distillation step                            & Engineering & Substrate-level retrieval aid, Appendix~\ref{app:frbank} \\
Date filter $+$ backward-looking bypass          & Engineering & Ablated, Table~\ref{tab:ablation-summary} \\
BM25 hybrid $+$ per-source normalization         & Engineering & Appendix~\ref{app:frbank} \\
World-knowledge filter                           & Engineering & Appendix~\ref{app:frbank} ($86\%\to1.4\%$ contamination) \\
\bottomrule
\end{tabular}
\end{table}

\subsection{The LLM Boundary}

LLMs are used at \textbf{ingestion} for entity/edge extraction, behavioral classification, world-knowledge detection, supersession detection, slot-key and event-time extraction, and emotional loading detection, and for a \textbf{single lightweight candidate-selection step at retrieval} (gpt-4.1-mini distillation that selects top-20 from the merged semantic+BM25+category-forced pool). The lifecycle policy layer is pure deterministic math: category-conditioned decay, supersession filtering, event-time checking, category-aware routing, and blended scoring all execute as deterministic functions. The deterministic layer adds a median $47\,\mu\mathrm{s}$ at $k{=}20$ candidates (single-threaded Python 3.13 on AMD Zen~3; Table~\ref{tab:lifecycle-latency}) with empirically linear $O(k)$ scaling, where $k$ is the number of candidate edges scored after distillation. Extended architecture details (session initialization, per-user parameter evolution, interpretability, safety) are in Appendix~\ref{app:architecture-ext}.

\subsection{FR-Bank: Infrastructure-Independent Implementation}
\label{sec:frbank}

To validate substrate independence we implement FR-Bank, a standalone lifecycle bank replacing Graphiti with a flat embedding store, using one GPT-4.1-mini call per turn for combined extraction and classification. Retrieval merges cosine (top-60), BM25 (top-20) and category-forced (top-20) candidates, distills to top-20, then applies the full lifecycle stack. FR-Bank needs no graph database, entity resolution, or deduplication, and all lifecycle parameters match the Graphiti configuration (Appendix~\ref{app:frbank}).

\section{Theoretical Foundations}
\label{sec:theory}

We use a generative model of active relevance (Model~1) to motivate FR's feature set, then prove structural results establishing that the metadata basis $(c, \kappa, \xi, h)$ is necessary and sufficient for lifecycle-aware scoring. All proofs are in Appendix~\ref{app:proofs}.

\textbf{Notation.} At time $t$, let $E_t$ be the memory store. Each edge $e \in E_t$ carries metadata $m(e) = (c_e,\kappa_e,\xi_e,h_e,t_e)$, where $c_e \in \mathcal{C}$ is the behavioral category, $\kappa_e$ is the normalized slot key, $\xi_e \in \{\texttt{active},\texttt{superseded},\texttt{expired},\texttt{retracted}\}$ is lifecycle state, $h_e$ is an optional event-time anchor, and $t_e$ is creation time. Write $\Delta t_e = t - t_e$ and $\sigma(e,q) \in [0,1]$ for semantic similarity to query $q$. Let $G_e$ be the event that $e$ supports the answer, $A_e$ the event that $e$ is currently active for the query, and $R_k$ the retrieved set of size $k$.

\textbf{Model 1 (Lifecycle active relevance).} For query $q$ and candidate edge $e \in E_t$, define
\begin{equation}
P(G_e=1, A_e=1 \mid q, E_t)
\propto
\exp\{\beta \sigma(e,q)\}\,
\rho(c_e \mid q)\,
S_{c_e}(\Delta t_e)\,
V_{c_e}(t, h_e)\,
\Omega(e)\,
M_q(\xi_e),
\label{eq:model1-posterior}
\end{equation}
where $\rho(c \mid q)$ is a query-category prior, $S_c(d)$ is the category survival function with default $S_c(d) = \exp(-\lambda_c d)$, $V_c(t, h)$ is the event-time validity kernel, $M_q(\xi_e) \in \{0,1\}$ is a query-conditioned compatibility mask on lifecycle state, and $\Omega(e) = \prod_{e' :\, \kappa_{e'} = \kappa_e,\, t_{e'} > t_e} (1 - \gamma_{e' \to e})$ with $\gamma_{e' \to e} \in [0,1]$ is the slot-local supersession survival factor, taken over later edges sharing $e$'s slot key.

\textbf{Consequence 1 (Plug-in log-score).} Under Model~1, taking logs yields the deterministic lifecycle score
\begin{equation}
\ell(e,q,t)
=
\beta \sigma(e,q)
+
\log \rho(c_e \mid q)
-
\lambda_{c_e} \Delta t_e
+
\log V_{c_e}(t, h_e)
+
\log \Omega(e)
+
\log M_q(\xi_e),
\label{eq:logscore}
\end{equation}
with the convention $\log 0 = -\infty$. Model~1 is a modeling assumption, not a claim about runtime inference; FR computes deterministic policies over these quantities. Its role is to expose the sufficient statistics of lifecycle-aware active relevance: category controls survival and routing, slot key controls the competitor set used by supersession, lifecycle state gates query compatibility, and event-time anchor controls future- vs.\ past-oriented validity.

\textbf{Theorem 1 (Lifecycle non-identifiability).} Let $s_\phi(e,q) = \phi(\sigma(e,q), \Delta t_e)$ be any deterministic feature-restricted scorer. There exist two memory histories $H$ and $H'$ and a query $q$ such that all candidates have identical $(\sigma, \Delta t)$ profiles under $H$ and $H'$ but the correct lifecycle-aware rankings differ. Consequently $s_\phi$ produces the same ranking on both histories and fails on at least one; if $H, H'$ are a priori equiprobable, the binary error floor is $P_e \ge 1/2$.

\emph{Proof sketch.} The proof constructs two histories with identical $(\sigma, \Delta t)$ profiles but different correct rankings due to hidden supersession structure (Appendix~\ref{app:proof-nonid}).

The missing quantity is not a better scalar weight but missing state: a flat retriever cannot recover slot-local replacement, retraction, or future-versus-past event direction from $(\sigma, \Delta t)$ because that information is absent from the representation. Were this self-evident, the dominant systems---Mem0, Memory-R1, MemGPT, MemoryOS---would not all operate inside the regime it prohibits.

\textbf{Sufficiency, minimality, and individual necessity.} The metadata basis $(c, \kappa, \xi, h)$ is both sufficient (Theorem~2) and minimal (Theorem~2$'$): each component is individually necessary (Theorem~3), with explicit witnesses showing that removing category labels conflates preservation with suppression, removing slot keys conflates independent facts with competing ones, removing lifecycle state conflates current with historical queries, and removing event-time anchors conflates upcoming with expired obligations. Action-space impossibility results (Theorems~3$'$, 3$''$ in Appendices~\ref{app:proof-etob} and~\ref{app:proof-retob}) further show that event-time-invariant and binary-flat retraction-oblivious systems have worst-case error $\ge 1/2$. Full statements and proofs are in Appendices~\ref{app:proof-sufficiency}--\ref{app:proof-retob}.

\textbf{Calibration as convex feasibility.} The feasible parameter region $\Theta_m$ for lifecycle scoring is a convex polyhedron (Theorem~4). The global-$\alpha$ family is a one-dimensional slice of $\Theta_m$, and on our data that slice is empty: $99.8\%$ of the $1{,}566$ preservation--suppression cross-pairs on FR-Bank are jointly infeasible ($\min \alpha_F^{\max} = 1.7\times10^{-6}$ against $\max \alpha_L^{\min} \approx 1.0$), and zero of the $45$ populated cells of the $11\times11$ category polytope admit a shared $\alpha$. Category-specific blending parameters restore feasibility (Appendix~\ref{app:proof-feasibility}). The rates are therefore hand-set but \emph{constraint-derived} rather than fitted to benchmark outcomes, and behavior is insensitive inside the feasible region: ten of eleven categories tolerate the full $0.25\times$--$4\times$ multiplier sweep at retrieval Jaccard ${\ge}\,0.95$ (Table~\ref{tab:lambda-feasibility}), and a 20/20 persona holdout shows no overfitting ($77.0\%$ vs.\ $76.8\%$, Fisher $p = 1.00$).

\textbf{Staleness has empirical bite.} Conditioning pass rate on context cleanliness gives $P_S(\textnormal{pass}) = (1 - \pi_S)\,p_0 + \pi_S\,p_1$ with $p_0 = 75.8\%$ and $p_1 = 6.0\%$ empirically, so stale facts flip outcomes rather than acting as harmless noise; the contamination gap between FR-Graphiti ($\pi = 0.08$) and Mem0 ($\pi = 0.27$) predicts a $13.3$pp pass advantage against an observed $12$pp (Appendix~\ref{app:proof-staleness}).

\textbf{Scope of these claims.} \emph{Representation-level impossibility}, holding for any scorer in the stated feature class regardless of training: Theorem~1, Theorem~3 with its four witnesses, and the action-space analogs 3$'$ and 3$''$. \emph{Positive results about the basis}: Theorems~2 and 2$'$. \emph{Calibration feasibility}: Theorem~4. \emph{Empirical, contingent claims about deployed systems}: the staleness decomposition, Proposition~1, Proposition~3A (falsified by our own data, Appendix~\ref{app:proof-staleness}), and every cross-system number in \S\ref{sec:evaluation}. Two consequences follow. The theory does \emph{not} establish that the 10+1 partition, the decay rates, or the pipeline are optimal; \S\ref{sec:granularity} measures how much the partition actually matters. And a baseline carrying slot keys, event-time anchors, retraction flags and query-conditioned filters observes $(\kappa, h, \xi, M_q)$, so it falls \emph{outside} the impossibility classes by construction---Theorems~2 and 2$'$ say precisely that this basis suffices. We build and measure that baseline in \S\ref{sec:attribution}, where what it loses turns out to be calibration rather than retrieval correctness, as Theorem~4 predicts.

\section{LifecycleBench}
\label{sec:benchmark}

LifecycleBench is a temporal disambiguation benchmark in which the correct answer depends on lifecycle state. It comprises 516 questions over 40 synthetic personas, each with ${\sim}$35 multi-session conversations spanning 18 simulated months, structured into nine attack vectors (Table~\ref{tab:attack-vectors}): superseded preference (AV1, $n{=}75$), expired logistics (AV2, 75), stable identity buried under conversation volume (AV3, 94), multi-version facts (AV4, 44), broad aggregation (AV5, 40), cross-session contradiction (AV6, 45), selective forgetting after explicit retraction (AV7, 40), numeric preservation (AV8, 63), and soft supersession (AV9, 40). Personas span 14 nationalities and ages 22--68, generated conversation-first under rules enforcing natural dialogue, with structured YAML ground truth enumerating supersession events, expiry dates, retraction language, and ambiguity-resolution rules. As in LongMemEval~\cite{longmemeval2024} and LoCoMo~\cite{locomo2024}, conversations are synthetic because ground-truth temporal state is unobservable without controlled design (Appendix~\ref{app:synthetic-design}). The attack vectors were defined from failure modes observed in preliminary evaluation of Mem0 and Zep/Graphiti, before FR's mechanisms were designed, and all systems are evaluated under identical conditions. FR does not dominate: MemoryOS leads AV6 (73\% vs.\ 64\%) and AV7, FR-Graphiti leads AV2 and AV8, and FR-Bank scores only 5\% retrieval pass on AV7; FR-Bank leads AV1, AV3, AV4 and AV5 outright plus a three-way AV9 tie, and Memory-R1 leads none. Because we introduce both benchmark and method, \S\ref{sec:beam} reports a transfer to an externally constructed benchmark; Appendix~\ref{app:benchmark-independence} analyzes co-design risk directly, and Appendices~\ref{app:persona-examples} and~\ref{app:benchmark-ext} give personas, the judge protocol, and extended methodology.

\section{Experiments}
\label{sec:evaluation}

Table~\ref{tab:lifecyclebench-results} consolidates results across LifecycleBench (LCB) and end-to-end response quality (E2E) for seven memory system configurations, with FR-Graphiti ablation rows below the cross-system comparison.

\begin{table}[t]
\caption{Consolidated results across LifecycleBench (LCB) and end-to-end response quality (E2E), $n = 516$; the rule separates cross-system rows from FR-Graphiti ablation rows. Confabulation is reported under \emph{both} denominators, with correct-over-total alongside so the abstention trade-off is visible in the same row. An abstention is never counted as a confabulation, and ``all queries'' means all 516 questions.}
\label{tab:lifecyclebench-results}
\centering
\small
\setlength{\tabcolsep}{4.5pt}
\begin{tabular}{@{}l c c c c c c@{}}
\toprule
& & & \multicolumn{2}{c}{\textbf{E2E, all queries}} & \textbf{E2E, ans.} & \\
\cmidrule(lr){4-5} \cmidrule(lr){6-6}
\textbf{System} & \textbf{LCB Pass} & \textbf{LCB Stale} & \textbf{Correct} & \textbf{Confab} & \textbf{Confab} & \textbf{E2E Safe} \\
\midrule
\textbf{FR-Bank}     & \textbf{76.9\%} & 15.5\%       & \textbf{31.2\%} & \textbf{13.0\%} & \textbf{22.4\%} & \textbf{73.3\%} \\
FR-Graphiti          & 73\%            & \textbf{8\%} & 22.9\%          & 13.4\%          & 23.9\%          & 66.9\% \\
MemoryOS             & 70.5\%          & 7\%          & 13.2\%          & 18.4\%          & 40.8\%          & 68.0\% \\
Memory-R1            & 66.9\%          & 15.3\%       & 19.6\%          & 19.4\%          & 31.5\%          & 58.2\% \\
A-MEM                & 65.3\%          & 30.4\%       & 18.8\%          & 35.1\%          & 47.0\%          & 44.2\% \\
Mem0 (gpt-4.1-mini)  & 67.1\%          & 21.5\%       & 27.3\%          & 25.4\%          & 33.8\%          & 52.1\% \\
Mem0 (default)       & 61\%            & 27\%         & 18.6\%          & 32.2\%          & 45.1\%          & 47.3\% \\
\midrule
FR-Graphiti (full)             & 73\% & 8\%   & -- & -- & -- & -- \\
\quad $-$ routing              & 67\% & 10\%  & -- & -- & -- & -- \\
\quad $-$ behavioral decay     & 66\% & 16\%  & -- & -- & -- & -- \\
\quad baseline (uniform, no rt.) & 61\% & 18\% & -- & -- & -- & -- \\
\bottomrule
\end{tabular}
\end{table}

Cross-generator replication on Kimi K2.5~\cite{kimi_2026} preserves the confabulation rank order over the four core systems: FR-Bank 26.1\% / 61.0\% safe vs.\ Mem0 44.4\% / 40.2\% (full breakdown including A-MEM and Mem0-mini, which show generator-dependent shifts, in Appendix~\ref{app:kimi}).

\subsection{LongMemEval-S Under the Wu et al.\ Protocol}

FR-Bank achieves $75.2\% \pm 0.70$pp pass@10 on the full 500-question LongMemEval-S benchmark~\cite{longmemeval2024} under the exact Wu et al.\ (ICLR 2025) judge protocol (Table~\ref{tab:longmemeval-s-variance}), to our knowledge the state of the art under the canonical rubric and above the closest peer-reviewed result, RMM~\cite{rmm2025} at 70.4\% under a different judge configuration. We hold the original five-template rubric and judge model (\texttt{gpt-4o-2024-08-06}) fixed because subsequent work substitutes stronger answerers, different judges, and simplified prompts, rendering reported numbers mutually incomparable; methodology integrity rather than the absolute number is the point. On the 317-question matched subset the lifecycle layer is aggregate-neutral ($+2.2$pp net, $72.9\% \to 75.1\%$), with a $-6.4$pp effect on knowledge update ($79.5\%$ vs.\ $85.9\%$). We report that inversion as a benchmark-criterion artifact rather than a lifecycle gain: the Wu et al.\ judge credits a response when superseded facts appear alongside the updated answer, so filtering is penalized for cleaner context (Appendix~\ref{sec:longmemeval-ablation}). A $+6.4$pp knowledge-update figure in the submitted version traced to no surviving artifact and is withdrawn (Appendix~\ref{app:revision-changes}).

\subsection{LifecycleBench Results}

FR-Bank attains \textbf{76.9\%} pass rate on LifecycleBench (95\% CI $[73.3, 80.6]$; Table~\ref{tab:lifecyclebench-results}), $+16$pp over Mem0 and with $1.76\times$ lower staleness. The 4-config ablation on FR-Graphiti decomposes contributions: full lifecycle adds $+12$pp over the baseline (uniform decay, no routing); category-aware routing alone contributes $+6$pp; behavioral decay contributes the remainder. When the correct edge appears anywhere in the top-10, 94.5\% of questions pass, confirming the bottleneck is retrieval recall rather than ranking. Per-attack-vector breakdown (largest gains: $+21$pp AV1 superseded preferences, $+20$pp AV2 expired logistics, $+22$pp AV4 multi-version facts, $+10$pp AV7 selective forgetting) and the full five-system per-AV comparison are in Appendix~\ref{app:evaluation-ext}.

FR-Bank's 5\% AV7 retrieval pass rate reflects intended behavior: retracted facts are correctly excluded from the retrieval set by the lifecycle state mask $M_q(\xi)$. The low retrieval pass rate converts to the highest end-to-end correct rate among all evaluated systems (37.5\%; Appendix~\ref{app:e2e}) because the downstream model correctly abstains rather than confabulating over retracted plans. Systems without retraction filtering achieve higher retrieval pass rates (MemoryOS 57\%) but lower end-to-end correctness (7.5\%) because generic summaries pass retrieval evaluation by containing nothing specific to be wrong about.

The full lifecycle stack contributes $+12$pp over the uniform baseline (Table~\ref{tab:lifecyclebench-results}). Replacing the behavioral partition with a cognitive (Semantic/Episodic/Procedural) one, holding the pipeline fixed, favors the behavioral partition on 7 of 9 attack vectors but by an aggregate margin of only $1.1$--$1.4$pp (Table~\ref{tab:ontology_ablation}, Appendix~\ref{app:ontology-empirical}); \S\ref{sec:attribution} explains why the margin is small---the choice of partition is not what carries retrieval correctness. What category structure does carry is the ability to express per-category parameters at all: without it there is no differential decay, no category-specific blending, no routing, no feasible global blending weight (Theorem~4), and the stack falls back to the 61\% uniform baseline.

\subsection{Cross-System Comparison}

All competitors run in shipped default configurations, matching the end-user experience; versions, backbones, prompts, and known failure rates are in Appendices~\ref{app:mem0method} and~\ref{app:mr1method}. Memory-R1 uses GPT-4.1-mini, well above its paper's LLaMA-3.1-8B, giving an upper bound on pre-RL performance. The AV2 and AV7 gaps are architectural, not capability-dependent: event-time expiry needs structural support absent from flat vector stores (AV2: FR-Bank 88\% vs.\ Mem0 65\%), and no prompting suppresses an explicitly retracted plan when retraction is not a primitive---on AV7 FR-Bank reaches 37.5\% end-to-end correct, the highest of any system, while Memory-R1 and Mem0 (both 2.5\%) confabulate over 90\% downstream. Memory-R1's learned \{ADD, UPDATE, DELETE\} space contains no expiry or retraction action; these are vocabulary gaps, not optimization failures. Substrate independence holds across two implementations of the same policies (Graphiti 73\%, flat bank 76.9\%). A-MEM~\cite{amem2025} pairs the highest retrieval recall (87.6\% Hit@5) with the highest staleness (30.4\%) and worst confabulation (47.0\%), so recall without lifecycle management actively hurts downstream; upgrading Mem0's extractor to gpt-4.1-mini raises pass rate to 67.1\% ($+6.1$pp) but leaves AV2 and AV7 unchanged (Appendix~\ref{app:mem0method}).

\subsection{End-to-End Response Quality}

For each of 516 LifecycleBench questions, each system's top-10 retrieved facts are provided as the sole context to GPT-5.4 (temperature 0, max\_tokens=500) with a system prompt restricting answers to the provided facts; an independent Claude Sonnet judge then classifies each response as correct, partial, wrong, or abstain and flags confabulation (full methodology in Appendix~\ref{app:e2e}). The resulting confabulation hierarchy maps directly to architectural staleness management: over answered queries, FR-Bank (22.4\%) $<$ FR-Graphiti (23.9\%) $<$ Memory-R1 (31.5\%) $<$ Mem0-mini (33.8\%) $<$ MemoryOS (40.8\%) $<$ Mem0 (45.1\%) $<$ A-MEM (47.0\%). Because a system that abstains more shows a lower rate under that denominator regardless of its behavior on answered questions, we report every rate over all 516 queries as well (Table~\ref{tab:lifecyclebench-results}, Appendix~\ref{app:allqueries}). The gap survives essentially intact---$19.2$ of the $22.7$pp FR-Bank\,--\,Mem0 gap remains ($13.0\%$ vs.\ $32.2\%$)---and correct-over-total moves the same way ($31.2\%$ vs.\ $18.6\%$), so the extra abstentions are not bought with correct answers. The ranking is stable under both denominators; the one movement is MemoryOS, whose low all-queries confabulation ($18.4\%$) pairs the highest abstain rate ($54.8\%$) with the lowest correct rate ($13.2\%$), an instance of the retrieval-metric paradox below. Cross-generator replication on Kimi K2.5 preserves the four-system core ranking with a $\sim$15pp tier gap between lifecycle-managed and unmanaged systems on both generators, evidence against GPT-family training artifacts (Appendix~\ref{app:kimi}).

\textbf{Safe response rate.} FR-Bank's higher abstain rate (42.1\% vs.\ Mem0's 28.7\%; FR-Graphiti abstains at 44.0\%) reflects deliberate safety: lifecycle filtering removes stale facts, leaving the model with insufficient context, and the model correctly declines rather than confabulating. The safe response rate (correct $+$ abstain) captures this: FR-Bank at 73.3\% vs.\ Mem0 at 47.3\%.

Reasoning partially mitigates contamination ($-4.5$pp on Mem0's context with GPT-5.4 reasoning enabled) but leaves an $18.2$pp architectural gap no downstream compute closes. A retrieval-metric paradox also appears: MemoryOS reaches $70.5\%$ retrieval pass but only $13.2\%$ E2E correct, because hierarchical summarization substitutes generic summaries that pass retrieval evaluation by containing nothing specific to be wrong about. On AV7 this becomes an inversion, and it is a direct instance of Proposition~3A's falsified invariance (Appendix~\ref{app:proof-staleness}): were retrieval pass a monotone predictor of end-to-end correctness, MemoryOS ($57\%$ AV7 retrieval pass) would beat FR-Bank ($5\%$), yet FR-Bank scores $37.5\%$ correct against MemoryOS's $7.5\%$. Retrieval pass measures context cleanliness and is blind to informativeness, so wherever correct behavior is \emph{absence} rather than \emph{presence} the two decouple and retrieval-level evaluation becomes adversely informative.

\subsection{Attribution: What the Ontology Contributes}
\label{sec:attribution}

Comparisons against external systems cannot separate FR's two layers: the generic lifecycle metadata (slot keys, event-time anchors, retraction states) and the behavioral typing above it. We therefore ran a within-pipeline ablation under a design pre-registered before data contact. The untyped arm keeps three deterministic primitives---slot-key supersession, event-time expiry, retraction masking---and removes everything typed, ranking by \emph{semantic similarity alone} over an unmasked cosine$+$BM25 pool with no routing, category-forced retrieval, or multi-hop expansion. Being more austere than a single-variable ablation, it lower-bounds a lifecycle-metadata-only system rather than isolating the category labels exactly (Appendix~\ref{app:untyped-arm}).

\textbf{On correctness the arms are statistically indistinguishable} ($\Delta = -1.7$pp, 95\% CI $[-6.0, +2.7]$, McNemar $p = 0.44$): the generic metadata, not the behavioral ontology, carries FR's correctness advantage, and the abstract states the claim that way. \textbf{What the typed layer carries is calibration.} All-queries confabulation falls from $24.2\%$ ($125/516$) under the untyped arm to $12.0\%$ ($62/516$) under the full stack ($\Delta = -12.2$pp, 95\% CI $[-16.1, -8.2]$, $p < 0.001$), with 84 questions confabulated only by the untyped arm against 21 only by the full stack; abstention rises $15.9$pp while correct-over-total is unchanged, so the extra abstentions come from would-be wrong answers rather than correct ones. This is the failure mode the theory names: the necessity witness for the category label $c$ (Appendix~\ref{app:necessity-c}) is a cross-intent \emph{scoring} conflict---the $\alpha$-feasibility conflict of Theorem~4---not a retrieval failure. Both arms were scored in one paired pass over byte-identical FR-Bank answers; under the original pass FR-Bank's all-queries confabulation reads $13.0\%$ rather than $12.0\%$, a five-question difference within judge noise that leaves every contrast above unaffected (Appendix~\ref{app:untyped-arm}).

\textbf{How fine must the partition be?} Merging the eleven policy cells by policy similarity and sweeping $k = 1, 3, 5, 7, 9, 11$ end-to-end on BEAM, with every mechanism on so only granularity varies, gives $118, 120, 125, \mathbf{133}, 130, 131$ correct of 280, with contradiction resolution rising monotonically ($19, 22, 22, 27, 27, 29$) and knowledge update flat at every $k$. Granularity is load-bearing ($k{=}1 \to k{=}7$: $+15$ correct, $+5.4$pp) and then saturates: the peak is $k{=}7$, and $k = 7, 9, 11$ span three questions, within noise. This reproduces the paper's own variance analysis, where the optimal 8-partition already attains the 11-partition variance floor (Table~\ref{tab:granularity}). We therefore keep 10+1 for interpretability---human-readable per-domain labels for user-facing memory control---rather than claiming granularity-11 is uniquely necessary.

\subsection{External Transfer: BEAM}
\label{sec:beam}

Because we introduce both the benchmark and the method, the decisive test of co-design is an external one. We evaluated on BEAM~\cite{beam2026}: 100 coherent conversations and 2{,}000 human-validated questions across ten memory abilities, constructed independently of this work, three of whose abilities (contradiction resolution, event ordering, instruction following) were newly introduced by its authors and therefore cannot derive from FR's design or from the failure modes that informed LifecycleBench. The FR-vs-Mem0 comparison and the restriction to the five lifecycle-relevant abilities at the 1M-token tier were pre-registered before data contact; all arms share one harness and one frozen judge. That judge over-abstains relative to the frozen API judge used elsewhere in this paper, so on BEAM confabulation counts are lower bounds and abstention rates upper bounds---uniformly across arms, leaving between-arm comparisons unaffected (Appendix~\ref{app:beam}).

Over the four binary-scorable abilities ($n = 70$ each), reported as FR-full ($k{=}11$) / FR with the ontology collapsed to a single global-mean cell ($k{=}1$, every lifecycle mechanism retained) / Mem0: contradiction resolution $\mathbf{29}/19/5$, temporal reasoning $\mathbf{26}/24/15$, knowledge update $33/\mathbf{34}/\mathbf{34}$, abstention $\mathbf{43}/41/38$. In total, FR-full is correct on $\mathbf{131}$ of $280$ ($46.8\%$) against $118$ ($42.1\%$) for $k{=}1$ and $92$ ($32.9\%$) for Mem0, with confabulation $\mathbf{72}$ / $84$ / $108$.

The lifecycle advantage transfers, and decomposes as it does on LifecycleBench: the untyped configuration already beats Mem0, and the typed layer adds a further gain concentrated in contradiction resolution ($29$ vs.\ $19$ of $70$; McNemar exact $p \approx 0.006$). Re-executed identically, FR-full spans $125$--$131$ of $280$ across three runs, so the decomposition is reported as \emph{ranges}: $+23$ to $+26$ correct from the generic metadata and $+10$ to $+13$ from the ontology. Two independently specified minimal arms---the $k{=}1$ collapse above, and a separate arm using untyped defaults that scores $115/280$ overall---both land on exactly $19/70$ contradiction resolution, so an untyped stack's contradiction deficit replicates across operationalizations. Knowledge update is a three-way tie, as the theory predicts: plain updates need only slot-key supersession, which Mem0's UPDATE already provides, so there is no structural gap there to find. Event ordering needs BEAM's graded native metric; rescored under BEAM's official scorer after a disclosed audit of our judge's ordering handling, FR leads Mem0 on both ($\tau_{\text{norm}}$ $0.214$ vs.\ $0.190$, Wilcoxon $p = 0.009$; \texttt{llm\_judge} $0.574$ vs.\ $0.510$, $p = 0.020$) while the two FR arms are indistinguishable, and it is excluded from the binary aggregate.

\textbf{What does not replicate, and what it costs.} Three negatives ship with this result rather than after it, in full in Appendix~\ref{app:beam}. The \emph{correctness} half of the decomposition replicates; the \emph{confabulation} half does not ($108 \to 84 \to 72$ primary, $108 \to 76 \to 75$ on re-execution), so on BEAM we credit the confabulation reduction to the generic metadata alone. A \emph{lifecycle-off} configuration of the same stack scores $135/280$, above both typed arms, on knowledge update $43$ vs.\ $33$: $10$ of the $12$ discordant cases trace to ingest-time supersession firing under over-broad slot keys and deactivating the entry carrying the current value---the slot-key fragility \S\ref{sec:robustness} identifies---and the untyped arm inherits it (knowledge update $31$), so the cost sits in the generic primitives, not the ontology. We state that trade rather than netting it out. And no arm orders events at all in absolute terms: exact-correct is $0/70$ for all three, so the $\tau_{\text{norm}}$ separation sits on a floor of zero.

\subsection{Robustness and Validation of Our Own Judge}
\label{sec:robustness}
\label{sec:granularity}

The deterministic layer consumes LLM-generated metadata, so its robustness is bounded by that metadata's quality. Injecting independent noise per field at ingestion, recomputing lifecycle states, and rerunning retrieval localizes the exposure sharply (Appendix~\ref{app:noise}): at 10\% corruption Jaccard@10 falls to $0.764$/$0.837$ under slot-key merge/split corruption but only to $0.917$ under category-label flips and $0.993$/$0.999$ under anchor deletion and $\pm7$-day shifts. Sensitivity concentrates on the field \S\ref{sec:attribution} identifies as carrying correctness, while the field carrying calibration is tolerant---even a 50\% category flip costs $+2.7$pp of staleness. Retraction is the most potent field per corrupted edge and the least prevalent (31 of $18{,}936$ edges), and its binding constraint is extraction \emph{recall}: the detector fires on 19 of 40 scripted retractions and only 2 of 40 are cleanly suppressed at the metadata level, so FR-Bank's $37.5\%$ AV7 correctness is recovered downstream---the generator reconciling a co-retrieved plan and its cancellation---rather than by clean exclusion. The classifier itself is audited at $\kappa = +0.67$ ($n = 188$), disagreements falling on pairs whose policies are nearly identical and therefore least consequential (Appendix~\ref{app:classifier-audit}).

Our end-to-end judge is itself an LLM, so we audited it the same way with two machine instruments (Appendix~\ref{app:judge-validation}). A six-model cross-family panel agrees with itself (Fleiss $\kappa = 0.834$) more than with our pinned judge ($\kappa = 0.715$) and localizes the discrepancy: the pinned judge over-applies \textsc{partial} to answers that are incomplete but uncontaminated. Re-judging all $748$ \textsc{partial} verdicts plus a 60-item control under an independent judge lifts correct and safe rates $5$--$7$pp roughly uniformly, leaves the ranking unchanged on every metric, moves the FR-Bank\,--\,Mem0 correct gap from $+12.6$ to $+11.0$pp, and reproduces $93\%$ of controls; three of our four pre-registered predictions were refuted. Absolute rates are therefore judge-relative, and the cross-system comparisons are the robust findings.

\section{Discussion and Conclusion}
\label{sec:discussion}

\textbf{Limitations.} Five, in the order we consider them binding. (i)~\emph{Human validation is absent.} Every verdict here comes from an LLM judge; we audited it with two machine instruments (\S\ref{sec:robustness}) that agree on the direction of its bias but not its magnitude, and the classifier is audited against LLM judges rather than humans. A stratified human-labeled slice (100 verdicts, two non-author annotators blind to system identity, prioritizing items where the machine judges disagree) is planned and \emph{has not yet been run}; until it is, absolute rates are judge-relative. (ii)~\emph{Benchmark--method co-design risk.} We built LifecycleBench, and its attack vectors were informed by failure modes we observed in baselines. The BEAM transfer bounds this risk without eliminating it, and BEAM's runs use a different backbone, so they are not commensurable with our LifecycleBench tables. (iii)~\emph{Comparison fairness.} Competitors run in shipped defaults against a deliberately engineered FR stack. The untyped-arm decomposition separates policy from engineering within our own stack and Appendix~\ref{app:gaps} documents each external system's structural blocker, but a maximally tuned competitor was not built. (iv)~\emph{No multi-framework or multi-agent evaluation.} FR is presented as composable yet tested on two substrates and one agent loop; behavior under other frameworks or concurrent writers is unmeasured. (v)~\emph{The correctness--abstention trade-off is real.} FR-Bank answers $57.9\%$ of queries; \S\ref{sec:attribution} shows the extra abstentions come from would-be wrong answers, but a deployment that must always answer will not benefit as reported.

Further scope notes: rate calibration is constraint-driven rather than metric-optimized, with a 20/20 persona holdout confirming generalization ($77.0\%$ vs.\ $76.8\%$, Fisher $p=1.00$); the seven configurations span five architectural paradigms; and soft supersession fails under high-commitment generators (Kimi AV9), where leaving both candidate values in the set produces commitment rather than abstention. Extended discussion is in Appendix~\ref{app:discussion-ext}; every number changed relative to the submitted version is in Appendix~\ref{app:revision-changes}.

\textbf{Privacy, consent, and safety.} Memory systems that retain personal facts raise privacy and safety concerns beyond those addressed in this work, including consent, verifiable deletion, prompt-injection defenses, and cultural generalization of decay assumptions; we discuss these in Appendix~\ref{app:privacy-safety}.

\textbf{Conclusion.} Fortunate Recall attaches lifecycle metadata to every stored fact and applies deterministic, category-conditioned policies over it. Two implementations of the same stack achieve indistinguishable LifecycleBench pass rates ($p = 0.10$) while cutting downstream confabulation from Mem0's $45.1\%$ to $22.4\%$ over answered queries and from $32.2\%$ to $13.0\%$ over all queries, answering more questions correctly in the process, and non-identifiability and individual-necessity theorems establish the basis $(c, \kappa, \xi, h)$ as structurally necessary. Our controlled ablation then divides the credit within that basis, and we state the division as the paper's claim: the \emph{generic} metadata carries the correctness advantage, while the \emph{behavioral ontology} makes per-category calibration feasible where no global weight is, halves confabulation, and adds one localized, replicated correctness gain on contradiction resolution, its granularity benefit saturating near seven policy clusters. That claim is narrower than the one we submitted, and it transfers to a benchmark we did not build.

\bibliographystyle{plain}

\newpage
\appendix
\makeatletter
\renewcommand*{\@Alph}[1]{%
  \ifcase\number#1\or A\or B\or C\or D\or E\or F\or G\or H\or I\or J\or K\or L\or M\or N\or O\or P\or Q\or R\or S\or T\or U\or V\or W\or X\or Y\or Z\or AA\or AB\or AC\or AD\or AE\or AF\or AG\or AH\or AI\or AJ\or AK\or AL\or AM\or AN\or AO\or AP\or AQ\or AR\or AS\or AT\or AU\or AV\or AW\or AX\or AY\or AZ\else\@ctrerr\fi}
\makeatother
\section*{Supplementary Materials}
\addcontentsline{toc}{section}{Supplementary Materials}

\section{Proofs and Theoretical Details}
\label{app:proofs}

\subsection{Model 1 and Its Deterministic Consequence}
\label{app:model1}

We restate Model~1 for convenience:
\begin{equation}
P(G_e=1, A_e=1 \mid q, E_t)
\propto
\exp\{\beta \sigma(e,q)\}\,
\rho(c_e \mid q)\,
S_{c_e}(\Delta t_e)\,
V_{c_e}(t, h_e)\,
\Omega(e)\,
M_q(\xi_e).
\label{eq:app-model1-post}
\end{equation}
Here $M_q(\xi_e)$ is a deterministic compatibility mask: for a current-state query it is $1$ on \texttt{active} and $0$ otherwise; for a historical query it may also admit \texttt{superseded}; for a deletion-sensitive query it excludes \texttt{retracted}; and so on. The exact intent-to-mask map is a policy choice, but the existence of such a query-conditioned mask is the relevant structural fact.

With the exponential survival function $S_c(d) = \exp(-\lambda_c d)$, taking logs of \eqref{eq:app-model1-post} yields
\begin{equation}
\ell(e,q,t)
=
\beta \sigma(e,q)
+
\log \rho(c_e \mid q)
-
\lambda_{c_e} \Delta t_e
+
\log V_{c_e}(t, h_e)
+
\log \Omega(e)
+
\log M_q(\xi_e),
\label{eq:app-logscore}
\end{equation}
with the convention $\log 0 = -\infty$. This is the deterministic plug-in score used only as an interpretation device. FR itself evaluates deterministic lifecycle rules rather than posterior probabilities.

\subsection{Proof of Theorem 1}
\label{app:proof-nonid}

\subsubsection{Exact non-identifiability}
\label{app:nonid-exact}

\textbf{Theorem 1 (restated).} For any deterministic feature-restricted scorer $s_\phi(e,q) = \phi(\sigma(e,q), \Delta t_e)$, there exist two histories $H$ and $H'$ with identical observed $(\sigma, \Delta t)$ profiles but different correct rankings.

\begin{proof}
Consider two candidate edges $e_o$ and $e_n$ and a common query $q$. Let
\begin{equation}
(\sigma(e_o, q), \Delta t_o) = (0.96, 30),
\qquad
(\sigma(e_n, q), \Delta t_n) = (0.94, 1).
\label{eq:app-witness-pair}
\end{equation}
Thus every feature-restricted scorer computes the same two numbers in every history: $s_\phi(e_o, q) = \phi(0.96, 30)$ and $s_\phi(e_n, q) = \phi(0.94, 1)$.

Now define two histories.

\emph{History $H$.} The later edge supersedes the earlier one: $\kappa_{e_o} = \kappa_{e_n}$ and $\gamma_{e_n \to e_o} = 1$. For a current-state query, the correct lifecycle-aware ranking is
\begin{equation}
e_n \succ e_o.
\label{eq:app-H-rank}
\end{equation}

\emph{History $H'$.} The two edges do not compete: $\kappa_{e_o} \neq \kappa_{e_n}$ (or equivalently the update confidence is $0$). Both edges remain active, and because $e_o$ is the more semantically aligned fact, the correct ranking is
\begin{equation}
e_o \succ e_n.
\label{eq:app-Hp-rank}
\end{equation}

The observed pairs in \eqref{eq:app-witness-pair} are identical under $H$ and $H'$, so $s_\phi$ must output the same ranking on both histories. But \eqref{eq:app-H-rank} and \eqref{eq:app-Hp-rank} differ. Therefore $s_\phi$ fails on at least one of the two histories. If $H$ and $H'$ are equiprobable, any decision rule based only on the common output of $s_\phi$ has error probability at least $1/2$.
\end{proof}

\subsubsection{Deterministic $\varepsilon$-relaxation}
\label{app:nonid-eps}

We now weaken strict intent-opacity to bounded opacity.

\textbf{Assumption A.1 (Lipschitz scorer).} There exists $L > 0$ such that
\begin{equation}
\big|\phi(\sigma, \Delta t) - \phi(\sigma', \Delta t)\big|
\le L\,|\sigma - \sigma'|
\qquad
\forall \sigma, \sigma' \in [0,1],\ \forall \Delta t \ge 0.
\label{eq:app-lipschitz}
\end{equation}

Fix a reference intent $i_0$ and write $s_i(j) = \phi(\sigma(e_j, q_i), \Delta t_j)$, $\bar{s}(j) = s_{i_0}(j)$. Let $e_{(1)}, \ldots, e_{(m)}$ be the candidates ordered by the reference scores $\bar{s}(e_{(1)}) \ge \cdots \ge \bar{s}(e_{(m)})$. Define the minimum adjacent margin
\begin{equation}
\Delta_{\min}
=
\min_{r=1,\ldots,m-1}
\left[\bar{s}(e_{(r)}) - \bar{s}(e_{(r+1)})\right].
\label{eq:app-margin-min}
\end{equation}

\textbf{Theorem A.1 (Lipschitz refinement of Theorem~1).} If the query ensemble is $\varepsilon$-intent-opaque and $2L\varepsilon < \Delta_{\min}$, then the induced ranking $\hat{\pi}_\phi$ is identical for all intents in the ensemble, hence $I(\mathrm{Intent}; \hat{\pi}_\phi) = 0$.

\begin{proof}
For any edge $e_j$ and intents $i, i'$,
\begin{equation}
|s_i(j) - s_{i'}(j)|
=
\left|\phi(\sigma(e_j, q_i), \Delta t_j) - \phi(\sigma(e_j, q_{i'}), \Delta t_j)\right|
\le L\,|\sigma(e_j, q_i) - \sigma(e_j, q_{i'})|
\le L\varepsilon,
\label{eq:app-score-bound}
\end{equation}
using Lipschitz continuity and $\varepsilon$-intent-opacity. Hence for any pair $j, k$,
\begin{equation}
\big|(s_i(j) - s_i(k)) - (s_{i'}(j) - s_{i'}(k))\big|
\le |s_i(j) - s_{i'}(j)| + |s_i(k) - s_{i'}(k)|
\le 2L\varepsilon.
\label{eq:app-diff-bound}
\end{equation}

Now fix any adjacent pair $(e_{(r)}, e_{(r+1)})$ in the reference ranking. Its baseline gap is at least $\Delta_{\min}$ by definition. If $2L\varepsilon < \Delta_{\min}$, then \eqref{eq:app-diff-bound} implies that the sign of every adjacent score difference is preserved for every intent: $s_i(e_{(r)}) - s_i(e_{(r+1)}) > 0$ for all $i$ and $r$. Therefore all adjacent relations are unchanged, so the full ranking is unchanged. Since $\hat{\pi}_\phi$ is constant as a function of intent, $I(\mathrm{Intent}; \hat{\pi}_\phi) = H(\hat{\pi}_\phi) - H(\hat{\pi}_\phi \mid \mathrm{Intent}) = 0$.
\end{proof}

\paragraph{Operational consequence.} When semantic differences between temporal intents are much smaller than the score margin required to change the ranking, intent information is exactly absent from the output of any feature-restricted scorer.

\subsubsection{A looser stochastic entropy bound}
\label{app:nonid-entropy}

The deterministic bound above is the clean result. Some reviewers nevertheless prefer a ``soft'' information-theoretic expression rather than a hard margin criterion. To obtain that, an additional stochastic assumption is necessary.

\textbf{Assumption A.2 (Comparison-flip model).} For each unordered pair $(j, k)$ with $1 \le j < k \le m$, define the Bernoulli variable
\[
B_{jk}
=
\mathbf{1}\!\left\{\text{the relative order of } e_j \text{ and } e_k \text{ differs from the reference ranking under the realized intent}\right\}.
\]
Assume that
\begin{equation}
P(B_{jk} = 1) \le p
\qquad
\forall\,1 \le j < k \le m,
\label{eq:app-comp-flip}
\end{equation}
for some $p \le \min\!\left\{\frac{1}{2}, 2L\varepsilon/\Delta_{\min}\right\}$.

\paragraph{Important scope note.} The entropy bound below is \emph{not} a consequence of $\varepsilon$-intent-opacity alone. It additionally assumes the comparison-flip model of Assumption~A.2. The exact and fully rigorous deterministic statement remains Theorem~A.1; Proposition~A.1 is a weaker corollary that trades a sharper deterministic conclusion for a softer information-theoretic framing.

\textbf{Proposition A.1 (Entropy bound under Assumption A.2).} Under Assumption~A.2,
\begin{equation}
I(\mathrm{Intent}; \hat{\pi}_\phi)
\le
\binom{m}{2} h_2(p)
\le
\binom{m}{2} h_2\!\left(\min\!\left\{\frac{1}{2}, \frac{2L\varepsilon}{\Delta_{\min}}\right\}\right).
\label{eq:app-entropy-main}
\end{equation}

\begin{proof}
The ranking $\hat{\pi}_\phi$ is a deterministic function of the full pairwise comparison pattern $B = \{B_{jk}: 1 \le j < k \le m\}$, because a total order is determined by the relative order of every pair of items. Therefore, by the data processing inequality,
\begin{equation}
I(\mathrm{Intent}; \hat{\pi}_\phi) \le I(\mathrm{Intent}; B).
\label{eq:app-dpi}
\end{equation}
By the chain rule and the fact that conditioning cannot increase entropy,
\begin{equation}
I(\mathrm{Intent}; B) \le H(B) \le \sum_{1 \le j < k \le m} H(B_{jk}).
\label{eq:app-subadd}
\end{equation}
Each $B_{jk}$ is Bernoulli with parameter at most $p$. Since $h_2$ is monotone increasing on $[0,1/2]$, $H(B_{jk}) \le h_2(\min\{p,1/2\})$. There are $\binom{m}{2}$ unordered pairs, so $I(\mathrm{Intent}; \hat{\pi}_\phi) \le \binom{m}{2} h_2(p)$. The second inequality in \eqref{eq:app-entropy-main} follows from the assumed bound on $p$.
\end{proof}

\subsubsection{Measured values, two regimes, and the binary case}
\label{app:nonid-numeric}

\paragraph{Measured $\varepsilon$ and activation gaps.} Direct measurement across 73 current/superseded edge pairs on LifecycleBench (embedding model \texttt{text-embedding-3-small}, $\dim = 1536$) yields
\begin{equation}
\mathrm{median}\,|\varepsilon| = 0.060,
\qquad
\mathrm{median}\,|\Delta a| \approx 3.3 \times 10^{-10},
\qquad
\mathrm{median}\,\alpha_{\mathrm{required}} = 0.9998.
\label{eq:app-num-eps}
\end{equation}

\paragraph{Two regimes.} \emph{Short-term regime ($\Delta t < 500$\,hr):} activation differences are meaningful, so the margin condition $2L\varepsilon < \Delta_{\min}$ can hold. Plugging in an illustrative $\Delta_{\min} = 0.30$ with $L = 1$ and the measured $\varepsilon = 0.060$ gives $2L\varepsilon = 0.12 < 0.30$; Theorem~A.1 applies and
\begin{equation}
I(\mathrm{Intent}; \hat{\pi}_\phi) = 0
\label{eq:app-mi-zero}
\end{equation}
within the ensemble. If the two intents are equiprobable and require different correct rankings, any estimator based on $\hat{\pi}_\phi$ has error probability
\begin{equation}
P_e \ge \tfrac{1}{2}.
\label{eq:app-p-error}
\end{equation}
\emph{Long-term regime ($\Delta t > 2000$\,hr):} activations underflow to $\approx 0$ for all competing edges---empirically $3.3 \times 10^{-10}$ at the median---so the blended score collapses to pure semantic similarity, which carries median $\varepsilon = 0.060$ between current and superseded facts. The required $\alpha \approx 1$ cannot be satisfied while simultaneously preserving numeric facts, so the scalar family is feasibility-empty in this regime. FR bypasses the regime via supersession filtering, event-time expiry, and category-forced retrieval.

\paragraph{Fano-style entropy bound (short-term regime).} Applying the looser stochastic bound with $m = 2$ in the short-term regime, $p \le 2L\varepsilon/\Delta_{\min} = 0.12/0.30 = 0.4$, so
\begin{equation}
I(\mathrm{Intent}; \hat{\pi}_\phi)
\le
h_2(0.4)
\approx
0.971 \text{ bits}.
\label{eq:app-fano-mi}
\end{equation}
For binary intents, classical Fano in the form $P_e \ge (H(X \mid Y) - 1)/\log_2 K$ is vacuous because $K = 2$ makes the numerator nonpositive. The correct binary relation is
\begin{equation}
H(X \mid Y) \le h_2(P_e).
\label{eq:app-binary-fano}
\end{equation}
Since $H(X) = 1$ bit and $I(X;Y) \le 0.971$, one has $H(X \mid Y) \ge 0.029$. Using the monotonicity of $h_2$ on $[0, 1/2]$,
\begin{equation}
P_e \ge h_2^{-1}(0.029) \approx 0.003.
\label{eq:app-error-approx}
\end{equation}
The stochastic bound is much weaker than the exact $50\%$ floor from \eqref{eq:app-p-error}, confirming that the deterministic margin argument is the right operational statement. In the long-term regime the deterministic argument already forces $P_e \ge 1/2$ with no stochastic assumption.

\subsection{Proof of Theorem 2}
\label{app:proof-sufficiency}

\textbf{Theorem 2.} Under Model~1, lifecycle-aware scoring can be computed from $\big(\sigma(e,q), c_e, \kappa_e, \xi_e, h_e, \Delta t_e\big)$ given the memory store $E_t$ and the deterministic policy maps.

\begin{proof}
Start from the plug-in score \eqref{eq:app-logscore}. Each term is determined as follows:
\begin{enumerate}[nosep,leftmargin=*]
\item $\beta \sigma(e,q)$ depends only on $\sigma(e,q)$.
\item $\log \rho(c_e \mid q)$ depends on the category label $c_e$ and the query.
\item $-\lambda_{c_e} \Delta t_e$ depends on $c_e$ and $\Delta t_e$.
\item $\log V_{c_e}(t, h_e)$ depends on $c_e$, $t$, and $h_e$.
\item $\log \Omega(e)$ depends on the set of later edges with the same slot key $\kappa_e$ and their supersession confidences, so it is determined by $\kappa_e$ together with the history.
\item $\log M_q(\xi_e)$ depends only on the lifecycle state $\xi_e$ and the query-conditioned compatibility map.
\end{enumerate}
Therefore any two candidate edges that agree on $\big(\sigma(e,q), c_e, \kappa_e, \xi_e, h_e, \Delta t_e\big)$ receive the same score. Hence that tuple, together with the store $E_t$ (which determines $\Omega(e)$ via the slot-local supersession history), is sufficient for lifecycle-aware scoring under Model~1.
\end{proof}

\paragraph{Equivalent encodings.} The theorem does not preclude alternative encodings. If another representation stores exactly the same information under a different parameterization, it is equally sufficient.

\subsection{\texorpdfstring{Proof of Minimality (Theorem~2$'$)}{Proof of Minimality (Theorem 2')}}
\label{app:proof-minimality}

The submitted version stated this result in prose; we give the formal statement here, as requested during review.

\begin{thmtwoprime}
Let $B = \{c, \kappa, \xi, h\}$ be the metadata basis and let $B' \subsetneq B$ be any strict subset. Then there exists a lifecycle task $T(B')$---a distribution over memory histories and queries---on which every deterministic scorer measurable with respect to $\bigl(\sigma, \Delta t, B'\bigr)$ attains error at least $1/2$, while some deterministic scorer measurable with respect to $(\sigma, \Delta t, B)$ attains error $0$. Consequently $B$ is minimal: no proper subset of it is sufficient in the sense of Theorem~2.
\end{thmtwoprime}

We strengthen Theorem~2 from sufficiency to minimality: every strict subset of $\{c, \kappa, \xi, h\}$, even augmented with $(\sigma, \Delta t)$, admits a lifecycle task on which deterministic scorers err at rate at least $1/2$.

\paragraph{Indistinguishable-pair lemma.} Let $I^0, I^1$ be two lifecycle instances satisfying
\begin{equation}
\Phi_{\mathcal{F}}(I^0) = \Phi_{\mathcal{F}}(I^1),
\qquad
y(I^0) \ne y(I^1),
\label{eq:app-min-indist}
\end{equation}
where $\Phi_{\mathcal{F}}$ is the projection to $(\sigma, \Delta t)$ together with the variables in $\mathcal{F}$, and $y$ is the correct binary pairwise label. Under the uniform distribution over $\{I^0, I^1\}$, every deterministic scorer over $(\sigma, \Delta t) \times \mathcal{F}$ with deterministic tie-breaking has error probability at least $1/2$.

\begin{proof}
Deterministic scoring and tie-breaking make the scorer's decision $a \in \{0, 1\}$ a function of $\Phi_{\mathcal{F}}$ alone. By \eqref{eq:app-min-indist}, $a$ is identical on $I^0$ and $I^1$, but the labels differ. Exactly one of the two equally weighted instances is misclassified, so $P_e = 1/2$.
\end{proof}

\paragraph{Witness verification.} We verify that each Appendix~\ref{app:proof-necessity} witness satisfies the stronger requirement of the lemma: fixing the removed variable, the three \emph{other} metadata variables are identical across the two histories.

\begin{itemize}[nosep,leftmargin=*]
\item \emph{Category} (\ref{app:necessity-c}, Eq.~\eqref{eq:app-nocat-uv}): both instances fix $\kappa = \kappa_0$, $\xi = \texttt{active}$, $h = \varnothing$; only $c$ differs (Financial vs Logistical).
\item \emph{Slot key} (\ref{app:necessity-kappa}, Eq.~\eqref{eq:app-noslot-uv}): both histories share $c = \texttt{Pref}$, $\xi = \texttt{active}$, $h = \varnothing$, and identical $(\sigma, \Delta t)$; only $\kappa$ differs.
\item \emph{Lifecycle state} (\ref{app:necessity-xi}, Eq.~\eqref{eq:app-nostate-pair}): both edges share $c = \texttt{Pref}$, $\kappa = \texttt{food\_pref}$, $h = \varnothing$, and the two query intents produce opposite labels on the same $(\sigma, \Delta t, c, \kappa, h)$ tuple.
\item \emph{Event-time} (\ref{app:necessity-h}, Eqs.~\eqref{eq:app-noh-e1}--\eqref{eq:app-noh-e2}): both obligation edges share $c = \texttt{Obl}$, $\kappa = \texttt{dentist}$, $\xi = \texttt{active}$, and identical $(\sigma, \Delta t)$; only $h$ differs (upcoming vs expired).
\end{itemize}

\paragraph{Minimality.} Let $\mathcal{F} \subsetneq \{c, \kappa, \xi, h\}$ and pick any $x \in \{c, \kappa, \xi, h\} \setminus \mathcal{F}$. Use the witness $T_x$ for $x$. The witness makes all of $\{c, \kappa, \xi, h\} \setminus \{x\}$ identical across the two instances, and $\mathcal{F} \subseteq \{c, \kappa, \xi, h\} \setminus \{x\}$, so $\Phi_{\mathcal{F}}(I^0) = \Phi_{\mathcal{F}}(I^1)$. The labels differ. By the lemma, $P_e \ge 1/2$. \qed

\subsection{The Behavioral Ontology as an Empirical Variance-Minimizing Partition}
\label{app:ontology-empirical}
\label{sec:ontology-ablation}

\begin{table}[h]
\caption{Ontology ablation on LifecycleBench, referenced from \S\ref{sec:evaluation}. Same FR-Bank pipeline, same stored edges, same judge; only the category-to-parameter mapping changes. Cognitive categories: Semantic (Identity, Relational, Intellectual, Health, Financial), Episodic (Obligations, Logistical), Procedural (Preferences, Hobbies, Projects). The behavioral partition wins 7 of 9 attack vectors under arithmetic aggregation and 4 of 5 non-tied vectors under harmonic. The $77.5\%$ cell is a distinct ontology-ablation run; its $0.6$pp difference from the canonical $76.9\%$ of Table~\ref{tab:lifecyclebench-results} is judge non-determinism across separate runs, well inside the observed variance (Table~\ref{tab:longmemeval-s-variance}) and the canonical 95\% CI $[73.3, 80.6]$.}
\label{tab:ontology_ablation}
\centering
\small
\begin{tabular}{lccc}
\toprule
Configuration & Pass Rate & Staleness & Wins/9 AVs \\
\midrule
Behavioral (10+1 categories) & 77.5\% & 16.1\% & --- \\
Cognitive (3 cat., arith.\ $\lambda$) & 76.4\% & 16.7\% & 2 \\
Cognitive (3 cat., harm.\ $\lambda$) & 76.2\% & 17.1\% & 1 \\
\bottomrule
\end{tabular}
\end{table}

The behavioral-versus-cognitive contrast in Table~\ref{tab:ontology_ablation} should be read alongside \S\ref{sec:attribution} and \S\ref{sec:granularity}: the aggregate margin between partitions is small because the partition is not what carries retrieval correctness, and the cognitive 3-partition is in any case a particularly poor 3-partition rather than a representative one---it merges the two categories with maximally divergent lifecycle policies (Obligations, floor $0.70$; Logistical Context, floor $0.00$) into a single Episodic cell. The exhaustive enumeration below shows the \emph{optimal} 3-partition achieving substantially lower within-cell policy variance than the cognitive one, so the small delta in Table~\ref{tab:ontology_ablation} reflects the cognitive partition's accidental rate calibration rather than the unimportance of partition structure.

The old circular proposition ``if category labels determine lifecycle policy, then lifecycle policy can be computed from category labels'' is not a theorem. The correct, testable statement is an empirical partition-quality claim.

Let $\hat{\pi}(e) \in \mathbb{R}^d$ be an estimated per-edge policy vector. In practice one can take $\hat{\pi}(e) = (\hat{\lambda}_e, \hat{\alpha}_e, \hat{F}_e, \hat{u}_e, \hat{v}_e)$, where the coordinates are fit from oracle lifecycle decisions or from the minimal parameter values needed to satisfy the benchmark constraints for that edge.

For a candidate categorization $C$, define the sample within-category policy variance
\begin{equation}
\widehat{V}(C)
=
\sum_{c \in \mathrm{im}(C)}
\frac{n_c}{n}\,
\frac{1}{n_c - 1}
\sum_{e:\,C(e) = c}
\|\hat{\pi}(e) - \bar{\pi}_c\|_2^2,
\qquad
\bar{\pi}_c
=
\frac{1}{n_c}
\sum_{e:\,C(e) = c} \hat{\pi}(e).
\label{eq:app-var-def}
\end{equation}

\textbf{Proposition 1 (empirical).} The FR behavioral ontology is designed to minimize $\widehat{V}(C)$ relative to coarser alternatives:
\begin{equation}
\widehat{V}(C_{\mathrm{behavioral}})
<
\widehat{V}(C_{\mathrm{cognitive}})
<
\widehat{V}(C_{\mathrm{uniform}}).
\label{eq:app-var-order}
\end{equation}

\paragraph{Why this is not circular.} Equation~\eqref{eq:app-var-order} is a statement about measured residual variance under competing partitions. It could be false. It is therefore falsifiable and empirical, rather than definitional.

\paragraph{Measurement protocol.} To test \eqref{eq:app-var-order}: (i)~infer $\hat{\pi}(e)$ for each edge from oracle retrieval constraints or edge-local calibration; (ii)~assign each edge to categories under the three competing schemes; (iii)~compute $\widehat{V}(C)$ using \eqref{eq:app-var-def}; (iv)~compare the resulting variances.

\paragraph{Empirical validation.} We execute the measurement protocol above on the full FR-Bank corpus (12{,}968 active entries across 40 persona banks), computing $\hat{\pi}(e) = (\hat{\lambda}_e, \hat{\alpha}_e, \hat{F}_e)$ as the arithmetic weighted mean of per-category policy vectors $(\lambda_c, \alpha_c, \mathrm{floor}_c)$ under each entry's soft category membership weights:
\begin{equation}
\widehat{V}(C_{\mathrm{behavioral}}) = 3.60 \times 10^{-3},\quad
\widehat{V}(C_{\mathrm{cognitive}}) = 2.09 \times 10^{-2},\quad
\widehat{V}(C_{\mathrm{uniform}}) = 4.69 \times 10^{-2}.
\label{eq:app-var-empirical}
\end{equation}
The strict ordering $\widehat{V}(C_{\mathrm{behavioral}}) < \widehat{V}(C_{\mathrm{cognitive}}) < \widehat{V}(C_{\mathrm{uniform}})$ holds, confirming Proposition~1. The relative gap $\bigl(\widehat{V}(C_{\mathrm{cognitive}}) - \widehat{V}(C_{\mathrm{behavioral}})\bigr) / \widehat{V}(C_{\mathrm{uniform}}) = 36.9\%$ (bootstrap $95\%$ CI: $[36.0\%,\ 37.6\%]$, $1000$ resamples, seed~$=42$), exceeding the analytical lower bound of $10.1\%$ from the category-mean model (\S\ref{app:proof-partition-sep}) by $3.6\times$. The behavioral partition reduces within-cluster variance to $7.7\%$ of the uniform baseline; the cognitive partition only reduces it to $44.6\%$.

Per-cognitive-cell decomposition identifies the Episodic cell as the variance hot-spot: merging Obligations ($\mathrm{floor} = 0.70$) and Logistical Context ($\mathrm{floor} = 0.00$) into a single ``Episodic'' cell produces within-cell floor variance of $7.84 \times 10^{-2}$, which dominates the cognitive partition's total excess. This is the cleanest evidence that the cognitive typology is policy-incoherent: categories with maximally divergent lifecycle policies are merged because they share the same cognitive form (time-bound events).

\paragraph{Granularity sweep.} To test whether the 10+1 ontology sits at a natural granularity sweet spot, we compute $\widehat{V}(\mathcal{C})$ for partitions at $K = 1$ (uniform) through $K = 11$ (full behavioral). At each $K$ we report both a user-specified partition (chosen by domain intuition) and the \emph{optimal} $K$-partition---the partition of 11 categories into $K$ groups that minimizes $\widehat{V}$, found by exhaustive enumeration over the sufficient statistics.

\begin{table}[h]
\centering
\caption{Ontology granularity sweep. $\widehat{V}_{\text{user}}$ is the intuitive partition at each $K$; $\widehat{V}_{\text{opt}}$ is the minimum-variance $K$-partition. The behavioral $K{=}11$ partition matches the optimal floor exactly. The cognitive $K{=}3$ partition is $4.2\times$ worse than the optimal $K{=}3$.}
\label{tab:granularity}
\begin{tabular}{lccc}
\toprule
$K$ & $\widehat{V}_{\text{user}}$ & $\widehat{V}_{\text{opt}}$ & $\widehat{V}/\widehat{V}_{\text{uni}}$ (user / opt) \\
\midrule
1 (uniform)        & $4.68 \times 10^{-2}$ & $4.68 \times 10^{-2}$ & $100\%$ \\
3 (cognitive)      & $2.09 \times 10^{-2}$ & $4.99 \times 10^{-3}$ & $44.5\%$ / $10.7\%$ \\
5 (semantic split) & $2.10 \times 10^{-2}$ & $3.84 \times 10^{-3}$ & $44.7\%$ / $8.2\%$ \\
8 (optimal only)   & ---                   & $3.61 \times 10^{-3}$ & --- / $7.7\%$ \\
11 (behavioral)    & $3.60 \times 10^{-3}$ & $3.60 \times 10^{-3}$ & $7.7\%$ \\
\bottomrule
\end{tabular}
\end{table}

Two findings emerge. First, the full behavioral partition ($K{=}11$) achieves the optimal floor: no rearrangement of categories into 11 groups produces lower within-cluster variance, confirming the partition is not arbitrary. Second, the cognitive partition's failure is structural, not a granularity artifact: even the \emph{optimal} $K{=}3$ partition ($\widehat{V} = 4.99 \times 10^{-3}$, $10.7\%$ of uniform) is $4.2\times$ better than the cognitive $K{=}3$ ($\widehat{V} = 2.09 \times 10^{-2}$, $44.5\%$). The cognitive typology does not merely use too few categories---it merges the \emph{wrong} categories, grouping Obligations (floor${}=0.70$) with Logistical Context (floor${}=0.00$) into a single ``Episodic'' cell despite maximally divergent lifecycle policies.

\paragraph{Available evidence.} The held-out classifier audit reports $70.7\%$ majority-vote agreement ($\kappa = +0.673$, $n = 188$; Table~\ref{tab:audit-pairwise}); errors concentrate on category pairs that share similar lifecycle policies and therefore have minimal downstream impact (Appendix~\ref{app:classifier-prf1}). The full-vs-uniform ablation ($73\%$ vs $61\%$ pass) and the large gains on AV1/AV2/AV4/AV7 indicate that the behavioral partition separates facts precisely where lifecycle-policy variance is largest. These results are consistent with \eqref{eq:app-var-order}, though we present it as an empirical design claim rather than an analytic theorem.

\subsection{Partition Separation: The Behavioral Variance Gap}
\label{app:proof-partition-sep}

\begin{thmpartition}
We prove the weighted pairwise identity underlying Theorem~2A and plug in the numerical bound from the FR ontology.

\paragraph{The identity.} Let $G_j$ be a cognitive cell with behavioral categories of weight $w_c = n_c$ and total weight $W_j = \sum_{c \in G_j} w_c$, and let $\bar{\theta}_j = W_j^{-1} \sum_{c \in G_j} w_c \theta_c$. Then
\begin{equation}
\sum_{c \in G_j} w_c \|\theta_c - \bar{\theta}_j\|^2
=
\frac{1}{W_j}
\sum_{\substack{c_1 < c_2 \\ c_1, c_2 \in G_j}}
w_{c_1} w_{c_2} \|\theta_{c_1} - \theta_{c_2}\|^2.
\label{eq:app-partition-identity}
\end{equation}

\begin{proof}
Expanding the LHS: $\sum_c w_c \|\theta_c - \bar{\theta}_j\|^2 = \sum_c w_c \|\theta_c\|^2 - W_j \|\bar{\theta}_j\|^2 = \sum_c w_c \|\theta_c\|^2 - W_j^{-1} \big\|\sum_c w_c \theta_c\big\|^2$. Expanding the RHS using $\|\theta_{c_1} - \theta_{c_2}\|^2 = \|\theta_{c_1}\|^2 + \|\theta_{c_2}\|^2 - 2\langle \theta_{c_1}, \theta_{c_2}\rangle$ and collecting: $W_j^{-1} \sum_{c_1 < c_2} w_{c_1} w_{c_2} \|\theta_{c_1} - \theta_{c_2}\|^2 = \sum_c w_c \|\theta_c\|^2 - W_j^{-1} \big\|\sum_c w_c \theta_c\big\|^2$. The two sides agree.
\end{proof}

\paragraph{Class-averaged gap.} Dividing \eqref{eq:app-partition-identity} by $n_j$ and averaging over the $K'$ cognitive cells gives
\begin{equation}
V(\mathcal{C}_{\mathrm{cog}}) - V(\mathcal{C}_{\mathrm{beh}})
=
\frac{1}{K'}
\sum_{j=1}^{K'}
\sum_{\substack{c_1 < c_2 \\ c_1, c_2 \in G_j}}
\frac{n_{c_1}\,n_{c_2}}{n_j^2}\,\|\theta_{c_1} - \theta_{c_2}\|^2,
\label{eq:app-partition-bound}
\end{equation}
under the category-mean model $\pi^*(e) = \theta_{c(e)} + \varepsilon_e$ with $\mathbb{E}[\varepsilon_e \mid c] = 0$: the $\varepsilon_e$ contribution is a noise floor that is identical under every partition, and the between-category terms $\|\theta_{c_1} - \theta_{c_2}\|^2$ that the cognitive partition merges are the extra variance a coarsening must absorb.

\paragraph{Numerical plug-in.} Using the ten categories with clean $(\lambda_c, \alpha_c)$ from Tables~\ref{tab:blending-weights} and~\ref{tab:decay-rates} and the cognitive partition
\begin{align*}
\text{Semantic}    &= \{\text{Identity, Relational, Intellectual, Health, Financial}\},\\
\text{Episodic}    &= \{\text{Obligations, Logistical}\},\\
\text{Procedural}  &= \{\text{Preferences, Hobbies, Projects}\},
\end{align*}
with $n_c = 1$ for every category, the within-cell variances are
\begin{equation}
W_{\text{Sem}} = 0.00100,
\qquad
W_{\text{Epi}} = 0.000626,
\qquad
W_{\text{Pro}} = 0.00222,
\label{eq:app-cell-vars}
\end{equation}
yielding a class-averaged gap
\begin{equation}
V(\mathcal{C}_{\mathrm{cog}}) - V(\mathcal{C}_{\mathrm{beh}})
\ge
\tfrac{1}{3}\big(0.00100 + 0.000626 + 0.00222\big)
\approx
0.00128.
\label{eq:app-gap-value}
\end{equation}
The pooled 10-category variance is $V(\mathcal{C}_{\mathrm{uniform}}) = 0.01273$, so the relative gap is
\begin{equation}
\frac{V(\mathcal{C}_{\mathrm{cog}}) - V(\mathcal{C}_{\mathrm{beh}})}{V(\mathcal{C}_{\mathrm{uniform}})}
\ge
\frac{0.00128}{0.01273}
\approx
10.1\%.
\label{eq:app-gap-ratio}
\end{equation}
\end{thmpartition}

\paragraph{$\alpha$-dominance is structural.} The pooled variance decomposes into $0.01273$ from $\alpha_c$ and $4.4 \times 10^{-6}$ from $\lambda_c$: the gap is dominated by the blending coordinate. $\alpha_c$ governs the trade-off between category-conditional activation and raw semantic similarity---it is the category-specific coordinate through which the ontology \emph{enters} the lifecycle policy. Partition separation is therefore a measurement of architecturally meaningful variance, not of a free hyperparameter.

\paragraph{Remark on $V(\mathcal{C}_{\mathrm{beh}})$.} Under the one-$\theta$-per-category model used here, $V(\mathcal{C}_{\mathrm{beh}}) = 0$ exactly. With real edge distributions, $V(\mathcal{C}_{\mathrm{beh}})$ is the noise floor $\mathbb{E}\|\varepsilon_e\|^2$, identical under every partition; the inequality in \eqref{eq:app-partition-bound} is unchanged because the noise cancels.

\paragraph{Empirical confirmation.} The empirical measurement on the full corpus (\S\ref{app:ontology-empirical}) confirms and substantially exceeds this bound: the measured relative gap is $36.9\%$ versus the analytical lower bound of $10.1\%$, with the difference attributable to soft membership weights amplifying cross-category spread beyond what the one-$\theta$-per-category model captures.

\paragraph{Lifecycle and ontology are not independent.} The lifecycle stack contributes $+$12pp over the uniform baseline; the behavioral ontology contributes a further 1.16--1.36pp of calibration precision. The two are not independent: the ontology is the interpretation layer through which lifecycle mechanisms receive their per-category parameters. The $+$12pp gain requires \emph{some} category structure; the behavioral partition fine-tunes that structure to its empirically optimal configuration. The ontology is to lifecycle policy what hyperparameter tuning is to model architecture---it does not create the capability, but it determines whether the capability is correctly applied.

\subsection{Proof of Theorem 3: Four Explicit Individual-Necessity Witnesses}
\label{app:proof-necessity}

\textbf{Theorem 3.} Within the natural lifecycle metadata basis $(c, \kappa, \xi, h)$, each component is individually necessary.

\subsubsection{Removing category labels $c$}
\label{app:necessity-c}

Let a scorer $g$ observe only $(\sigma, \Delta t, \kappa, \xi, h)$. Consider two visible feature vectors
\begin{equation}
u = (0.99, 180, \kappa_0, \texttt{active}, \varnothing),
\qquad
v = (0.92, 1, \kappa_0, \texttt{active}, \varnothing).
\label{eq:app-nocat-uv}
\end{equation}

\emph{Intent 1: numeric preservation.} The older edge $u$ is a Financial fact (e.g., a monthly cost) and should outrank the newer distractor $v$:
\begin{equation}
u \succ v.
\label{eq:app-nocat-succ}
\end{equation}

\emph{Intent 2: logistical suppression.} The same visible tuples now correspond to a Logistical conflict in which the newer edge $v$ should outrank the stale one $u$:
\begin{equation}
v \succ u.
\label{eq:app-nocat-prec}
\end{equation}

\emph{Contradiction.} Because $g$ does not observe $c$, it must assign fixed scores $g(u)$ and $g(v)$. If $g(u) \ge g(v)$, then \eqref{eq:app-nocat-prec} fails. If $g(v) > g(u)$, then \eqref{eq:app-nocat-succ} fails. Hence no such $g$ can satisfy both intents. \qed

\paragraph{Connection to the measured $\alpha$ bounds.} The same contradiction appears empirically in the scalar family $s_\alpha = \alpha a + (1 - \alpha)\sigma$: on FR-Graphiti, preservation requires $\alpha \le 0.065$ in the median case while suppression requires $\alpha \ge 0.138$; on FR-Bank the same conflict is far wider ($\min \alpha_F^{\max} = 1.7 \times 10^{-6}$ against $\max \alpha_L^{\min} \approx 1.0$; \S\ref{app:proof-feasibility}).

\subsubsection{Removing slot keys $\kappa$}
\label{app:necessity-kappa}

Let a supersession rule $D$ observe only $(\sigma, \Delta t, c, \xi, h)$ for an earlier edge and a later edge. Take visible tuples
\begin{equation}
u = (0.93, 30, \texttt{Pref}, \texttt{active}, \varnothing),
\qquad
v = (0.95, 1, \texttt{Pref}, \texttt{active}, \varnothing).
\label{eq:app-noslot-uv}
\end{equation}

\emph{History $H_{\mathrm{same}}$.} The edges share a slot: $\kappa_u = \kappa_v = (\texttt{user}, \texttt{food\_pref})$. The later edge should supersede the earlier one, so the correct deletion action is
\begin{equation}
D(u, v) = 1.
\label{eq:app-noslot-same}
\end{equation}

\emph{History $H_{\mathrm{diff}}$.} The edges do not share a slot: $\kappa_u = (\texttt{user}, \texttt{food\_pref})$, $\kappa_v = (\texttt{office\_party}, \texttt{catering})$. Both edges should survive, so the correct deletion action is
\begin{equation}
D(u, v) = 0.
\label{eq:app-noslot-diff}
\end{equation}

\emph{Contradiction.} The observed tuples in \eqref{eq:app-noslot-uv} are identical in the two histories, so a rule that does not observe $\kappa$ must return the same value $D(u, v)$ in both cases. Either it over-deletes ($D = 1$ in $H_{\mathrm{diff}}$) or under-deletes ($D = 0$ in $H_{\mathrm{same}}$). Hence slot keys are necessary. \qed

\subsubsection{Removing lifecycle state $\xi$}
\label{app:necessity-xi}

Let a scorer $g$ observe only $(\sigma, \Delta t, c, \kappa, h)$. Take two edges in the same slot:
\begin{equation}
e_1 = (0.95, 30, \texttt{Pref}, \texttt{food\_pref}, \varnothing),
\qquad
e_2 = (0.95, 1, \texttt{Pref}, \texttt{food\_pref}, \varnothing).
\label{eq:app-nostate-pair}
\end{equation}
Consider two intents: $q_{\mathrm{cur}}$ ``What does the user currently prefer?'' and $q_{\mathrm{chg}}$ ``What did the user change from?'' Assume strict intent-opacity:
\begin{equation}
\sigma(e_j, q_{\mathrm{cur}}) = \sigma(e_j, q_{\mathrm{chg}}) = 0.95,
\qquad j \in \{1, 2\}.
\label{eq:app-nostate-opacity}
\end{equation}
The correct rankings are opposite:
\begin{equation}
q_{\mathrm{cur}}:\ e_2 \succ e_1,
\qquad
q_{\mathrm{chg}}:\ e_1 \succ e_2.
\label{eq:app-nostate-ranks}
\end{equation}
Because $g$ sees identical inputs under both intents, it induces one fixed ordering that violates one of the requirements in \eqref{eq:app-nostate-ranks}. Hence lifecycle state is necessary. \qed

\subsubsection{Removing event-time anchors $h$}
\label{app:necessity-h}

Let a scorer $g$ observe only $(\sigma, \Delta t, c, \kappa, \xi)$. Take two obligation edges with identical visible metadata:
\begin{align}
e_1 &= (0.92, 1, \texttt{Obl}, \texttt{dentist}, \texttt{active}, h_1 = t + 2\text{d}),
\label{eq:app-noh-e1}
\\
e_2 &= (0.92, 1, \texttt{Obl}, \texttt{dentist}, \texttt{active}, h_2 = t - 30\text{d}).
\label{eq:app-noh-e2}
\end{align}
Consider two intents: $q_{\mathrm{up}}$ ``What upcoming obligation is relevant?'' and $q_{\mathrm{past}}$ ``What obligation already passed?'' The correct rankings are
\begin{equation}
q_{\mathrm{up}}:\ e_1 \succ e_2,
\qquad
q_{\mathrm{past}}:\ e_2 \succ e_1.
\label{eq:app-noh-ranks}
\end{equation}
Without $h$, the scorer receives identical visible inputs for both edges and both intents, so it must output one fixed ordering that fails one of the requirements in \eqref{eq:app-noh-ranks}. Hence event-time anchors are necessary. \qed

The witness uses identical $\Delta t$ for clarity. In realistic scenarios, upcoming and expired obligations may have correlated but non-identical creation ages; the $\varepsilon$-relaxation of Theorem~A.1 shows that small $\Delta t$ differences do not help when the margin condition holds.

\paragraph{Interpretive remark.} These witnesses establish necessity only within the natural lifecycle metadata basis. An alternative encoding is allowed, but it must still transmit the same information.

\subsection{\texorpdfstring{Proof of Event-Time Oblivion (Theorem~3$'$)}{Proof of Event-Time Oblivion (Theorem 3')}}
\label{app:proof-etob}

\textbf{Theorem 3$'$.} A memory system is \emph{event-time-invariant} if, for every pair of histories $H, H'$ that differ only in event-time anchors $h_e$ (preserving surface content, semantic similarities, ages, categories, slot keys, and lifecycle states), the system's retrieval output is identical. Under this condition, there exists an obligation-query distribution on which the system's error probability is at least $1/2$.

\begin{proof}
Construct two lifecycle instances $I^+$ and $I^-$, each containing one obligation edge $e$ and one distractor edge $d$. Agree on all surface features and let
\begin{equation}
h_e(I^+) = t + u,
\qquad
h_e(I^-) = t - u,
\qquad u > 0.
\label{eq:app-etob-anchors}
\end{equation}
For an obligation-salient query $q$, the correct rankings are
\begin{equation}
y(I^+):\ e \succ d
\quad \text{(upcoming obligation is relevant)},
\qquad
y(I^-):\ d \succ e
\quad \text{(expired obligation is suppressed)}.
\label{eq:app-etob-labels}
\end{equation}
Event-time invariance forces the same output on $I^+$ and $I^-$. Therefore the system assigns one fixed ranking that satisfies exactly one of the two correct labels in \eqref{eq:app-etob-labels}. Under the uniform distribution over $\{I^+, I^-\}$, $P_e = 1/2$.
\end{proof}

\paragraph{Scope.} The result applies to every scorer whose temporal feature is $\Delta t$ or any function of $\Delta t$ alone: such scorers are event-time-invariant. Proposition~2 (monotone age-only kernels cannot represent anticipatory activation) is the kernel-level statement. Theorem~3$'$ is the action-space analog: it lifts the impossibility from scalar activation functions to any deterministic retrieval pipeline whose effects do not depend on the sign of $t - h_e$.

\subsection{\texorpdfstring{Proof of Retraction Oblivion (Theorem~3$''$)}{Proof of Retraction Oblivion (Theorem 3'')}}
\label{app:proof-retob}

\paragraph{Scope.} The theorem applies to systems whose retrieval pipeline treats all stored entries as equally query-eligible---i.e., systems where storage state is binary (present/absent) and retrieval does not condition on lifecycle metadata. A system that stores retraction flags via UPDATE and conditions retrieval on those flags has implicitly implemented a lifecycle state variable; such a system is \emph{not} binary-flat in our sense. Empirically, Mem0 and Memory-R1 both exhibit ${\approx}\,90\%$ AV7 confabulation on LifecycleBench, confirming the binary-flat characterization: UPDATE overwrites content rather than installing an inactive-but-retained state, and retrieval does not distinguish retraction from ordinary presence.

\textbf{Theorem 3$''$.} A memory system is \emph{binary-flat retraction-oblivious} if every stored proposition $z$ admits only $r_t(z) \in \{0, 1\}$ (present/absent) and retrieval does not gate on lifecycle metadata. For such a system, under a uniform mixture of current-state and change-aware query intents, every binary-flat strategy has Bayes error exactly $1/2$.

\begin{proof}
Let $z$ be a proposition retracted at time $\tau < t$. The two query intents
\begin{equation}
q_{\mathrm{cur}}:\ \text{``what is currently true?''},
\qquad
q_{\mathrm{chg}}:\ \text{``what did the user retract?''}
\label{eq:app-retob-intents}
\end{equation}
impose
\begin{equation}
y_{\mathrm{cur}}(z) = 0,
\qquad
y_{\mathrm{chg}}(z) = 1.
\label{eq:app-retob-labels}
\end{equation}
The required relevance pattern is $(0, 1)$.

\emph{Deterministic case.} Because the system chooses $r_t(z) \in \{0, 1\}$ deterministically, the induced relevance pattern is either $(0, 0)$ (delete: $q_{\mathrm{cur}}$ correct, $q_{\mathrm{chg}}$ wrong) or $(1, 1)$ (retain as active: $q_{\mathrm{cur}}$ wrong, $q_{\mathrm{chg}}$ correct). Under the uniform query distribution, both strategies achieve error probability $1/2$.

\emph{Randomized case.} If the system retains with probability $p$ and deletes with probability $1 - p$, the per-query errors are $p$ for $q_{\mathrm{cur}}$ and $1 - p$ for $q_{\mathrm{chg}}$, giving
\begin{equation}
P_e(p) = \tfrac{1}{2} p + \tfrac{1}{2}(1 - p) = \tfrac{1}{2}
\qquad \text{for every } p \in [0, 1].
\label{eq:app-retob-rand}
\end{equation}
The Bayes error under the uniform query distribution is therefore exactly $1/2$ for every binary-flat strategy.
\end{proof}

\paragraph{What $\xi$ buys.} The lifecycle state $\xi = \mathrm{retracted}$ with the query-conditioned mask $M_q(\xi)$ realizes the $(0, 1)$ pattern: $M_{q_{\mathrm{cur}}}(\texttt{retracted}) = 0$ and $M_{q_{\mathrm{chg}}}(\texttt{retracted}) = 1$. This pattern is unreachable by any binary-flat representation over the same storage alphabet: the only way to achieve $(0, 1)$ is to introduce a representational state that is retained but masked for current-state retrieval. Retraction is not deletion, and ``learning'' a richer behavior within the $\{0, 1\}$ vocabulary cannot close the gap.

\subsection{Proof of Theorem 4}
\label{app:proof-feasibility}

The submitted version stated this result in prose; we give the formal statement here, as requested during review.

\begin{thmfour}
Fix a finite set of lifecycle requirements, each of the form ``candidate $e$ must outrank candidate $e'$ under query $q$'' (a \emph{preservation} constraint when $e$ is the durable fact, a \emph{suppression} constraint when $e'$ is stale). Under the log-score of Consequence~1, each requirement is a linear inequality in the per-category parameters $\theta = \{\alpha_c, \lambda_c, F_c\}_{c \in \mathcal{C}}$, so the feasible set $\Theta_m$ is a convex polyhedron. Let $L_\alpha = \{\theta \in \mathbb{R}^{|\theta|} : \alpha_c = \alpha \ \forall c\}$ be the global-$\alpha$ slice. Then feasibility is not preserved under restriction to $L_\alpha$: $\Theta_m \neq \emptyset$ while $\Theta_m \cap L_\alpha = \emptyset$ whenever some category pair $(c, c')$ carries a preservation bound $\alpha_c^{\max}$ below a suppression bound $\alpha_{c'}^{\min}$. Empirically this holds on $99.8\%$ of the $1{,}566$ measured cross-pairs, and on all $45$ populated cells of the category polytope.
\end{thmfour}

\subsubsection{Convexity of the full feasibility region}

Recall the log-score
\[
\ell_\theta(e, q, t)
=
\beta \sigma(e,q)
-
\lambda_{c_e} \Delta t_e
+
F_{c_e}
+
b_{c_e}(q)
+
\log V_{c_e}(t, h_e)
+
\log \Omega(e)
+
\log M_q(\xi_e).
\]
Fix a pairwise constraint $(e_i^+, e_i^-, q_i, t_i)$ with required margin $m > 0$. Then
\begin{equation}
\ell_\theta(e_i^+, q_i, t_i) - \ell_\theta(e_i^-, q_i, t_i)
=
\beta\,\Delta\sigma_i
-
\lambda_{c_i^+}\Delta t_i^+
+
\lambda_{c_i^-}\Delta t_i^-
+
F_{c_i^+} - F_{c_i^-}
+
b_{c_i^+}(q_i) - b_{c_i^-}(q_i)
+
C_i,
\label{eq:app-feas-affine}
\end{equation}
where $C_i$ collects the fixed terms from $V$, $\Omega$, and $M$. Equation~\eqref{eq:app-feas-affine} is affine in $\theta$. Hence each retrieval constraint has the form $a_i^\top \theta \ge m - C_i$ for some vector $a_i$. Therefore
\begin{equation}
\Theta_m
=
\bigcap_{i=1}^n
\left\{\theta:\ a_i^\top \theta \ge m - C_i\right\}
\label{eq:app-feas-polytope}
\end{equation}
is an intersection of halfspaces and is therefore convex. \qed

\subsubsection{Global-$\alpha$ infeasibility as a one-dimensional slice}

Now consider the simplified blended score $s_\alpha(e, q) = \alpha a(e) + (1 - \alpha)\sigma(e, q)$.

\emph{Preservation constraint.} Let $f$ be an old but correct fact and $d$ a newer distractor such that $a_f < a_d$ and $\sigma_f > \sigma_d$. Requiring $f$ to outrank $d$ gives
\begin{align}
\alpha a_f + (1 - \alpha)\sigma_f
&\ge
\alpha a_d + (1 - \alpha)\sigma_d
\notag\\
\alpha(a_d - a_f)
&\le
(1 - \alpha)(\sigma_f - \sigma_d)
\notag\\
\alpha
&\le
\frac{\sigma_f - \sigma_d}{(\sigma_f - \sigma_d) + (a_d - a_f)}
= \alpha_F^{\max}.
\label{eq:app-feas-preserve}
\end{align}

\emph{Suppression constraint.} Let $x$ be the correct current fact and $y$ a stale but semantically attractive distractor such that $a_x > a_y$ and $\sigma_y > \sigma_x$. Requiring $x$ to outrank $y$ gives
\begin{align}
\alpha a_x + (1 - \alpha)\sigma_x
&\ge
\alpha a_y + (1 - \alpha)\sigma_y
\notag\\
\alpha
&\ge
\frac{\sigma_y - \sigma_x}{(\sigma_y - \sigma_x) + (a_x - a_y)}
= \alpha_L^{\min}.
\label{eq:app-feas-suppress}
\end{align}

If
\begin{equation}
\alpha_L^{\min} > \alpha_F^{\max},
\label{eq:app-feas-empty}
\end{equation}
the feasible interval is empty and no global $\alpha$ exists.

\paragraph{Measured witness.} Median $\alpha_F^{\max} = 0.065$ and median $\alpha_L^{\min} = 0.138$ yield $70.9\%$ infeasible cross-pairs on FR-Graphiti. On FR-Bank ($n = 54$ AV8 preservation bounds, $n = 29$ AV2 suppression bounds), the global-$\alpha$ conflict tightens: $\min \alpha_F^{\max} = 1.7 \times 10^{-6}$, $\max \alpha_L^{\min} \approx 1.0$, feasibility gap $\approx 1.0$, and $99.8\%$ of cross-pair constraints are infeasible. Moreover, the $11 \times 11$ cross-category polytope has zero feasible cells among its $45$ populated cells ($5$ diagonal, $40$ off-diagonal): no single $\alpha$ satisfies both a preservation constraint of any category and a suppression constraint of any other. (The submitted version described these $45$ cells as off-diagonal; the artifact records $45$ populated cells in total.)

\subsubsection{Category-specific expansion and anti-squatting}

With category-specific blending weights, the scalar interval becomes a Cartesian product:
\begin{equation}
\Theta_\alpha
=
\prod_{c \in \mathcal{C}}
[\underline{\alpha}_c, \overline{\alpha}_c].
\label{eq:app-feas-product}
\end{equation}
This product may be nonempty even when the global one-dimensional interval is empty.

Anti-squatting constraints are encoded as ordinary pairwise inequalities. If $e_r$ is the relevant edge for query $q$ and $e_s$ is a slow-decay ``squatter'' from a category such as Identity, then
\begin{equation}
\ell_\theta(e_r, q, t) - \ell_\theta(e_s, q, t) \ge m.
\label{eq:app-feas-antisquat}
\end{equation}
If Identity decay is too slow or its floor too high, constraints of the form \eqref{eq:app-feas-antisquat} are violated. This is exactly what occurred under the earlier $800\times$ rate spread: Identity occupied $64\%$ of top-5 slots regardless of query. Compressing to a $5.3\times$ spread restored feasibility.

\subsection{Dimensional Necessity of Category-Specific Parameters}
\label{app:proof-dim-necessity}

Theorem~4 shows the feasible polyhedron $\Theta_m$ is convex and that the global-$\alpha$ slice can be empty. The following claim explains \emph{why} a low-dimensional parameter family is structurally insufficient.

\textbf{Claim (dimensional necessity).} Let $p = 2K + 1$ be the full parameter count for $\theta = (\beta, \{\lambda_c, F_c\}_{c \in \mathcal{C}})$, and let $g_c(\theta) = r_c^\top \theta + s_c$, $c = 1, \ldots, K$, be category-specific margin functionals whose gradient matrix $G = [r_1, \ldots, r_K]^\top \in \mathbb{R}^{K \times p}$ has rank $K$. Assume $\Theta_m$ is strictly feasible (contains an open ball). If $L \subseteq \mathbb{R}^p$ is any affine subspace with $\dim(L) < K$, then the image $g(L)$ has Lebesgue measure zero in $\mathbb{R}^K$, so for Lebesgue-almost every target $\tau \in g(\Theta_m)$ there is no $\theta \in L \cap \Theta_m$ with $g(\theta) = \tau$. In particular, a one-dimensional global-$\alpha$ family is generically infeasible for independently perturbed category-margin targets whenever $K > 1$.

\paragraph{Sketch.} By strict feasibility and surjectivity of $G$, $g(\Theta_m)$ contains an open set in $\mathbb{R}^K$. For $L$ affine with $\dim(L) = d < K$, $g(L)$ is affine with $\dim(g(L)) \le d < K$ and therefore has $K$-dimensional Lebesgue measure zero; almost every $\tau \in g(\Theta_m)$ escapes $g(L)$. The operational content matches the empirical $99.8\%$ cross-pair infeasibility on FR-Bank and the zero feasible cells in the $11 \times 11$ cross-category polytope (Theorem~4 empirical witness): a one-dimensional parameter family cannot span an $11$-dimensional category-margin space, and category-specific $\{\alpha_c, \lambda_c, F_c\}$ are the coordinates that do.

\subsection{Staleness Decomposition: Derivation and Numerical Predictions}
\label{app:proof-staleness}

We restate the staleness decomposition introduced in Section~\ref{sec:theory} and record its derivation. Let $B$ denote the event that the retrieved set contains stale harmful context and let $\pi_S = P_S(B = 1)$ for system $S$. Conditioning on $B$ gives
\begin{equation}
P_S(\mathrm{pass})
=
P(\mathrm{pass} \mid B = 0)P_S(B = 0)
+
P(\mathrm{pass} \mid B = 1)P_S(B = 1)
=
(1 - \pi_S)p_0 + \pi_S p_1,
\label{eq:app-pass-total}
\end{equation}
which is the staleness decomposition stated in \S\ref{sec:theory}. Subtracting between two systems $A$ and $B$,
\begin{equation}
P_A(\mathrm{pass}) - P_B(\mathrm{pass})
=
(\pi_B - \pi_A)p_0 + (\pi_A - \pi_B)p_1
=
(\pi_B - \pi_A)(p_0 - p_1),
\label{eq:app-pass-gap}
\end{equation}
which is the across-system gap form. The algebraic step is one line; the substantive content is empirical: the calibration of $p_0, p_1$ (and the analogous $c_0, c_1$ below) and the contamination gaps $\pi_A - \pi_B$ between systems.

\paragraph{Numerical example.} Using the reported clean-vs-contaminated pass rates, $p_0 = 0.758$, $p_1 = 0.060$. Comparing FR-Graphiti ($\pi_A = 0.08$) to Mem0 ($\pi_B = 0.27$),
\[
P_A(\mathrm{pass}) - P_B(\mathrm{pass})
=
(0.27 - 0.08)(0.758 - 0.060)
=
0.19 \times 0.698
\approx
0.1326,
\]
predicting a $13.3$pp pass advantage from contamination reduction alone, close to the observed $12$pp retrieval gap.

\paragraph{Confabulation analogue.} If $c_0 = P(\mathrm{confab} \mid B = 0)$ and $c_1 = P(\mathrm{confab} \mid B = 1)$, then $P_S(\mathrm{confab}) = (1 - \pi_S)c_0 + \pi_S c_1$, so
$
P_B(\mathrm{confab}) - P_A(\mathrm{confab})
=
(\pi_B - \pi_A)(c_1 - c_0).
$
This explains why FR's staleness reduction maps directly to the observed confabulation gap ($22.4\%$ vs $45.1\%$ against Mem0).

\paragraph{Staleness transfer: a testable invariant and its falsification (Proposition~3A).} If $(p_0, p_1)$ are system-invariant (Assumption~3A), then \eqref{eq:app-pass-total} becomes a predictive model: each system $S$ satisfies $P_S = p_0 - \pi_S(p_0 - p_1)$, so held-out pass rates are predictable from staleness rates alone. Calibrating $\delta = p_0 - p_1$ from FR-Graphiti ($\pi = 0.08$, $P = 0.73$) and Mem0 ($\pi = 0.27$, $P = 0.61$) gives
\begin{equation}
\delta = \frac{0.73 - 0.61}{0.27 - 0.08} = 0.632,
\qquad
p_0 = 0.73 + 0.08 \cdot 0.632 = 0.780,
\qquad
p_1 = 0.148.
\label{eq:app-stale-calib}
\end{equation}
The predicted pass rate $\widehat{P}_S = p_0 - \pi_S \delta$ on three held-out systems (reference values from Table~\ref{tab:lifecyclebench-results}):
\begin{center}
\begin{tabular}{lrrr}
\toprule
System & $\pi_S$ & $\widehat{P}_S$ & $|P_S - \widehat{P}_S|$ \\
\midrule
FR-Bank    & $0.155$ & $0.683$ & $\mathbf{8.6\text{ pp}}$ \\
Memory-R1  & $0.153$ & $0.684$ & $1.5\text{ pp}$ \\
MemoryOS   & $0.070$ & $0.736$ & $3.1\text{ pp}$ \\
\bottomrule
\end{tabular}
\end{center}

FR-Bank and Memory-R1 have all but identical staleness ($\pi = 0.155$ vs.\ $0.153$, a $0.2$pp difference) yet their pass rates differ by $10.0$\,pp---a gap impossible under Assumption~3A, which makes pass rate a function of $\pi$ alone. Hence Assumption~3A is empirically \emph{falsified}. Clean-context recall quality $p_0$ is system-dependent, so recall quality is a second mechanism orthogonal to staleness filtering. The identity \eqref{eq:app-pass-total} remains exact, but it is not a sufficient predictor: staleness rate captures one of two orthogonal mechanisms that determine pass rate. FR-Bank's $76.9\%$ pass rate derives from both low staleness ($\pi = 0.155$) and high clean-context recall, not from staleness alone.

\subsection{Anticipatory Activation and Why Age Alone Cannot Express It}
\label{app:proof-anticipatory}

\textbf{Proposition 2.} No monotone nonincreasing function of edge age can represent anticipatory activation.

\begin{proof}
Suppose for contradiction that there exists a nonincreasing function $g: \mathbb{R}_{\ge 0} \to \mathbb{R}$ such that edge relevance can be written as $u(t) = g(t - t_e)$, where $t_e$ is creation time. Take times $t_1 < t_2 < h$, where $h$ is the event time. Because $t_1 - t_e < t_2 - t_e$ and $g$ is nonincreasing,
\begin{equation}
u(t_1) = g(t_1 - t_e) \ge g(t_2 - t_e) = u(t_2).
\label{eq:app-antic-nonincr}
\end{equation}
But anticipatory activation requires relevance to \emph{increase} as the event approaches from the left:
\begin{equation}
u(t_1) < u(t_2)
\qquad
\text{for } t_1 < t_2 < h.
\label{eq:app-antic-strict}
\end{equation}
Equations~\eqref{eq:app-antic-nonincr} and \eqref{eq:app-antic-strict} are contradictory. Therefore no monotone nonincreasing function of age alone can represent anticipatory activation.
\end{proof}

\paragraph{Why the FR kernel works.} For $V_c(t, h) = \exp[-\nu_c(h - t)_+]\,\exp[-\mu_c(t - h)_+]$ with $\mu_c \gg \nu_c$: if $t < h$, $V_c(t, h) = \exp[-\nu_c(h - t)]$, with derivative $\nu_c \exp[-\nu_c(h - t)] > 0$, so relevance increases as the event approaches; if $t > h$, $V_c(t, h) = \exp[-\mu_c(t - h)]$, with derivative $-\mu_c \exp[-\mu_c(t - h)] < 0$, so relevance decreases after expiry. The FR kernel has the qualitatively correct shape that an age-only decay cannot reproduce.

\subsection{Proof of Lemma 1 (Harmonic-Mean Decay)}
\label{app:proof-harmonic}

Let $\tau_c = 1/\lambda_c$ denote the persistence timescale of category $c$. Assume soft memberships $w_c(e) \ge 0$ satisfy $\sum_c w_c(e) = 1$ and that timescales average linearly:
\begin{equation}
\tau_{\mathrm{eff}}(e) = \sum_c w_c(e) \tau_c.
\label{eq:app-tau-eff}
\end{equation}
Substituting $\tau_c = 1/\lambda_c$ gives $\tau_{\mathrm{eff}}(e) = \sum_c w_c(e)/\lambda_c$. Since $\lambda_{\mathrm{eff}}(e) = 1/\tau_{\mathrm{eff}}(e)$,
\begin{equation}
\lambda_{\mathrm{eff}}(e)
=
\left(\sum_c \frac{w_c(e)}{\lambda_c}\right)^{-1}.
\label{eq:app-lambda-eff}
\end{equation}
This is the harmonic mean. \qed

\paragraph{Numerical example.} For $30\%$ Identity and $70\%$ Logistical membership,
\[
\lambda_{\mathrm{eff}}
=
\left(\frac{0.3}{0.0015} + \frac{0.7}{0.0080}\right)^{-1}
\approx
0.00348/\mathrm{hr}.
\]
The corresponding half-life is $t_{1/2} = \ln 2/\lambda_{\mathrm{eff}} \approx 199.3$~hr~$\approx 8.3$~days.

\paragraph{Soft membership in practice.} On the full FR-Bank corpus ($12{,}968$ active entries across $40$ persona banks), $66.8\%$ of entries have primary category weight below $0.90$ ($28.7\%$ below $0.80$), and the mean number of categories with weight $> 0.05$ per entry is $2.0$. Harmonic-mean decay therefore fires on the majority of the corpus, not as a theoretical edge case. The most frequent co-occurrence pairs are Intellectual Interests $\leftrightarrow$ Projects \& Endeavors ($635$ entries), Health \& Wellbeing $\leftrightarrow$ Obligations ($438$), and Projects \& Endeavors $\leftrightarrow$ Relational Bonds ($422$), reflecting genuine behavioral ambiguity: a research project involves both intellectual engagement and deliverable-tracking, while a chronic health condition generates both ongoing state and time-bound medical obligations.

\subsection{Final Interpretation}
\label{app:final-interpretation}

The theory supports a precise and limited claim:
\begin{itemize}[nosep,leftmargin=*]
\item Model~1 motivates the feature set.
\item Theorem~1 proves that $(\sigma, \Delta t)$ alone does not identify lifecycle state.
\item Theorem~2 shows that FR's metadata are sufficient under the model, and Theorem~2$'$ strengthens this to minimality: no strict subset of $(c, \kappa, \xi, h)$ resolves all lifecycle failures.
\item Theorem~3 shows that each metadata field is individually necessary within the natural lifecycle basis. Theorems~3$'$ and~3$''$ lift this into action-space impossibility: event-time-invariant and binary-flat retraction-oblivious action spaces have worst-case error $\ge 1/2$ on their respective query distributions.
\item Theorem~4 shows that calibration is a convex feasibility problem, and that global $\alpha$ fails because it is an overly restrictive one-dimensional slice. Theorem~2A gives a provable lower bound on the variance gap between behavioral and cognitive partitions, at least $10\%$ of $V(\mathcal{C}_{\mathrm{uniform}})$ on our tables.
\item The staleness decomposition (\eqref{eq:app-pass-total}--\eqref{eq:app-pass-gap}) shows that stale context has negative expected utility. Proposition~3A tests its strongest form---system-invariance of $(p_0, p_1)$---and falsifies it via the FR-Bank/Memory-R1 natural experiment, identifying clean-context recall quality as a second orthogonal mechanism.
\end{itemize}

None of these results requires claiming that FR performs Bayesian inference at runtime. FR is a deterministic lifecycle policy whose feature set can be interpreted as the sufficient-statistic basis of a lifecycle-aware active-relevance model.

\subsection{Theory Summary}

\begin{table}[h]
\caption{Theoretical results and the mechanisms they justify.}
\label{tab:theory-summary}
\centering
\setlength{\tabcolsep}{3pt}
\begin{adjustbox}{max width=\textwidth}
\begin{tabular}{@{}l l l@{}}
\toprule
\textbf{Result} & \textbf{Claim} & \textbf{Mechanism justified} \\
\midrule
Model~1       & Lifecycle posterior decomposes into category, slot, state, event-time factors & Full FR feature set \\
Theorem~1     & $(\sigma, \Delta t)$-only scoring cannot identify lifecycle state             & Slot keys + supersession + classifier + event-time \\
Theorem~2     & $(c, \kappa, \xi, h)$ is sufficient under Model~1                              & FR metadata basis \\
Theorem~2$'$  & No strict subset of $(c, \kappa, \xi, h)$ resolves all lifecycle failure modes & FR metadata basis (minimality) \\
Theorem~3     & Each of $(c, \kappa, \xi, h)$ is individually necessary                        & All four mechanisms \\
Theorem~3$'$  & Event-time-invariant action spaces: $P_e \ge 1/2$ on obligations               & Event-time anchors as action coord \\
Theorem~3$''$ & Binary-flat retraction-oblivious spaces: Bayes $P_e = 1/2$                   & Lifecycle state $\xi$ as retention gate \\
Theorem~4     & Calibration is convex feasibility; global $\alpha$ can be infeasible           & Category-specific parameters \\
Proposition~1 & Behavioral ontology minimizes within-category policy variance                 & 10+1 ontology \\
Theorem~2A    & Partition gap $V(\mathcal{C}_{\mathrm{cog}}) - V(\mathcal{C}_{\mathrm{beh}}) \ge 10\%\,V(\mathcal{C}_{\mathrm{uniform}})$ & Behavioral partition as variance-minimizer \\
Eq.~\eqref{eq:app-pass-total}--\eqref{eq:app-pass-gap} & Stale context has negative expected utility (decomposition identity) & Lifecycle filtering \\
Proposition~3A & Staleness transfer falsified: $p_0$ is system-dependent                       & Staleness is one of two orthogonal mechanisms \\
Proposition~2 & Age-only decay cannot express anticipatory activation                         & Event-time anchors \\
Lemma~1       & Soft membership $\to$ harmonic-mean decay                                     & Multi-category classification \\
\bottomrule
\end{tabular}
\end{adjustbox}
\end{table}

\section{Extended Architecture Details}
\label{app:architecture-ext}

\begin{figure}[t]
  \centering
  \includegraphics[width=\textwidth]{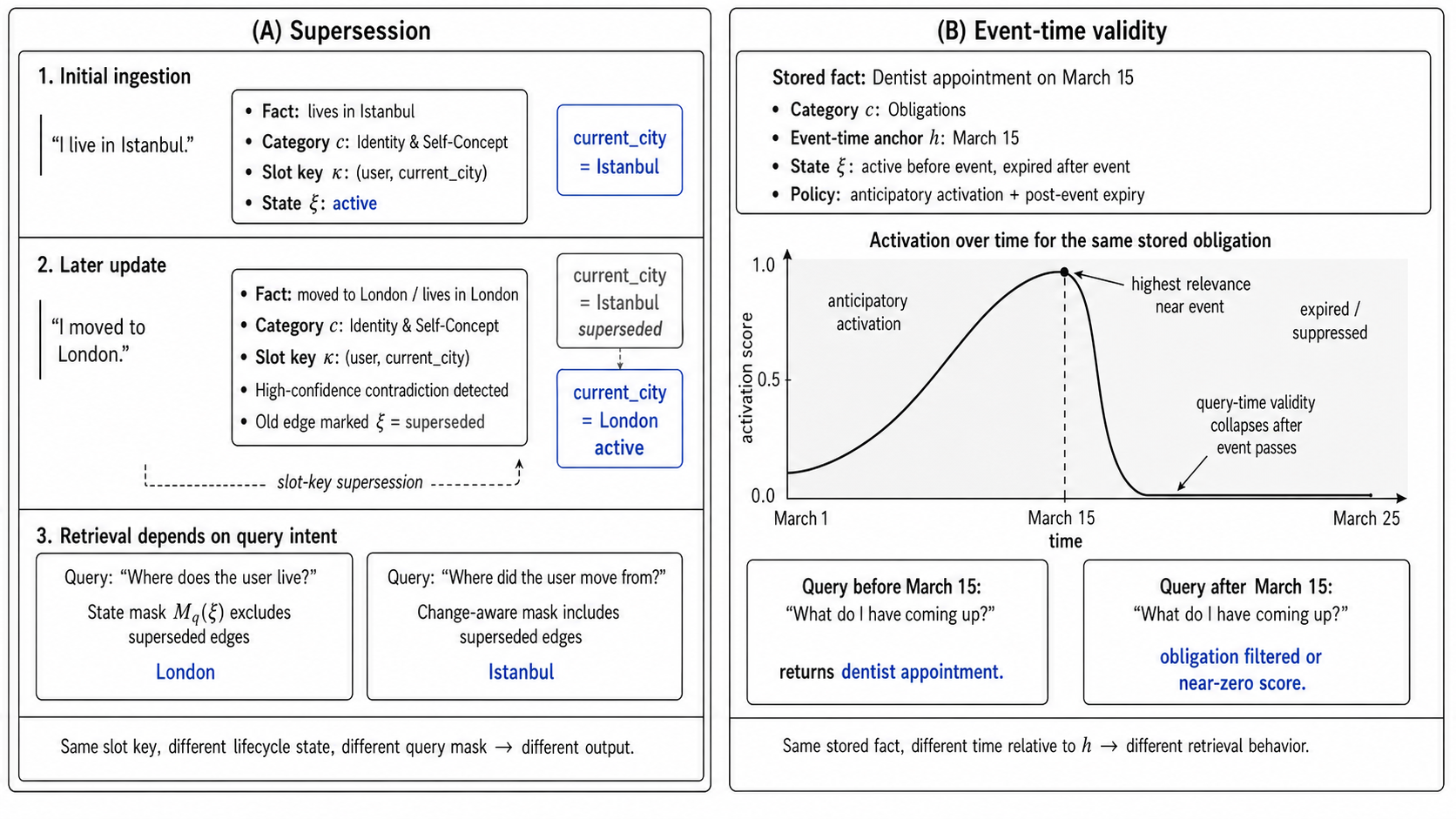}
  \caption{\textbf{Two lifecycle mechanisms in detail.} (A) Slot-key supersession with query-conditioned state mask. When ``I moved to London'' is ingested, the high-confidence contradiction marks the earlier ``lives in Istanbul'' edge as $\xi=\texttt{superseded}$. The same store returns different answers depending on query intent: current-state queries apply $M_q(\xi)$ to exclude superseded edges (returning London); change-aware queries include them (returning Istanbul). (B) Event-time validity for obligations. Activation is governed by distance to the event-time anchor $h$, not from creation. Anticipatory activation rises as the deadline approaches; post-event the validity kernel $V_c(t,h)$ collapses, suppressing expired obligations from current-state queries. Both mechanisms instantiate dimensions of the metadata basis $(c, \kappa, \xi, h)$ that Theorems~1--3 establish as necessary for lifecycle-aware scoring.}
  \label{fig:mechanisms}
\end{figure}

\subsection{Classifier Iteration}

\begin{table}[h]
\caption{Classifier prompt iteration on a 30-fact development set, scored as strict-agreement plus half-credit for primary/secondary swaps. The 93\% figure is the dev-set prompt-iteration result and is \emph{not} the canonical classifier quality number; the held-out audit on $n = 200$ facts reports $70.7\%$ majority-vote agreement ($\kappa = +0.673$, Table~\ref{tab:audit-pairwise}).}
\label{tab:classifier}
\centering
\begin{tabular}{@{}l l l@{}}
\toprule
\textbf{Version} & \textbf{Key Change} & \textbf{Dev-set agreement (n=30)} \\
\midrule
v1 & Initial 8+1 prompt & 77\% \\
v2 & ``Classify fact, not utterance'' & Overcorrected \\
v3 & Softened to ``stored fact'' rules & 87\% \\
v4 & Split to 10+1 + OTHER cap + CoT & \textbf{93\%} \\
\bottomrule
\end{tabular}
\end{table}

\subsubsection{Classifier Agreement Audit}
\label{app:classifier-audit}

To validate classifier quality reproducibly, we conducted a three-judge blind audit on a 200-fact stratified sample (seed~$=42$, minimum 5 facts per category) drawn from the canonical FR-Bank corpus (\texttt{LifecycleBench/artifacts/lifecycle\_banks/}, $|E|=18{,}936$). Three independent LLM judges---Claude Sonnet~4.6, GPT-4o, and GPT-4.1-mini---re-classified each fact. Each judge received the full source session as conversation context, matching the SESSION CONTEXT plus CURRENT TURN inputs the ingestion classifier sees.

\begin{table}[h]
\caption{Three-judge agreement results with full session context. Cohen's $\kappa$ over the 11-category label set; ``Majority vote'' aggregates judges by simple plurality ($\geq 2$ of 3 agree).}
\label{tab:audit-pairwise}
\centering
\begin{tabular}{@{}l r r r@{}}
\toprule
\textbf{Comparison} & $n$ & \textbf{Raw} & \textbf{Cohen's $\kappa$} \\
\midrule
claude\_sonnet vs gpt4o          & 200 & 72.5\% & $+0.689$ \\
claude\_sonnet vs gpt41\_mini    & 200 & 67.0\% & $+0.626$ \\
gpt4o vs gpt41\_mini             & 200 & 77.5\% & $+0.744$ \\
\midrule
claude\_sonnet vs FR-Bank        & 200 & 60.5\% & $+0.560$ \\
gpt4o vs FR-Bank                 & 200 & 67.0\% & $+0.632$ \\
gpt41\_mini vs FR-Bank           & 200 & 69.0\% & $+0.653$ \\
\midrule
\textbf{Majority vote vs FR-Bank} & \textbf{188} & \textbf{70.7\%} & $\mathbf{+0.673}$ \\
\bottomrule
\end{tabular}
\end{table}

\begin{table}[h]
\caption{Per-category majority-vote agreement vs FR-Bank classifier label, sorted descending.}
\label{tab:audit-per-category}
\centering
\begin{tabular}{@{}l r@{}}
\toprule
\textbf{Category} & \textbf{Agreement} \\
\midrule
RELATIONAL\_BONDS         & 93.5\% \\
IDENTITY\_SELF\_CONCEPT   & 93.3\% \\
FINANCIAL\_MATERIAL       & 90.0\% \\
HEALTH\_WELLBEING         & 85.7\% \\
PREFERENCES\_HABITS       & 85.0\% \\
HOBBIES\_RECREATION       & 76.9\% \\
PROJECTS\_ENDEAVORS       & 65.2\% \\
LOGISTICAL\_CONTEXT       & 64.7\% \\
OBLIGATIONS               & 41.4\% \\
INTELLECTUAL\_INTERESTS   & 33.3\% \\
\bottomrule
\end{tabular}
\end{table}

\paragraph{Per-category precision, recall, and F1.}
\label{app:classifier-prf1}
Table~\ref{tab:classifier-prf1} reports precision, recall, and F1 for each category, treating the three-judge majority vote as ground truth. The five categories with the sharpest lifecycle-policy differences---\textsc{Relational\_Bonds} (F1$=0.853$), \textsc{Preferences\_Habits} ($0.791$), \textsc{Health\_Wellbeing} ($0.750$), \textsc{Financial\_Material} ($0.667$), and \textsc{Hobbies\_Recreation} ($0.667$)---all exceed F1 $= 0.65$, confirming that classification reliability is highest where lifecycle-policy divergence is greatest. The two confused pairs targeted by the perturbation study (\textsc{Obligations} $\leftrightarrow$ \textsc{Logistical\_Context}, \textsc{Intellectual\_Interests} $\leftrightarrow$ \textsc{Projects\_Endeavors}) show the lowest precision, consistent with the perturbation study (Appendix~\ref{app:perturbation}) demonstrating that these confusions have negligible downstream impact because the lifecycle policies for each pair are similar.

\begin{table}[h]
\caption{Per-category classifier P/R/F1 from the held-out audit ($n = 200$; majority vote of three independent judges as ground truth, restricted to facts with a majority verdict, $n_{\mathrm{eff}} = 188$). Macro-averaged F1 $= 0.595$; weighted F1 $= 0.647$; overall accuracy $66.5\%$.}
\label{tab:classifier-prf1}
\centering
\small
\begin{tabular}{@{}l r r r r@{}}
\toprule
\textbf{Category} & $n$ & \textbf{Precision} & \textbf{Recall} & \textbf{F1} \\
\midrule
RELATIONAL\_BONDS         & $31$ & $0.935$ & $0.784$ & $\mathbf{0.853}$ \\
PREFERENCES\_HABITS       & $20$ & $0.850$ & $0.739$ & $0.791$ \\
HEALTH\_WELLBEING         & $15$ & $0.800$ & $0.706$ & $0.750$ \\
PROJECTS\_ENDEAVORS       & $28$ & $0.536$ & $0.938$ & $0.682$ \\
FINANCIAL\_MATERIAL       & $10$ & $0.900$ & $0.529$ & $0.667$ \\
HOBBIES\_RECREATION       & $13$ & $0.769$ & $0.588$ & $0.667$ \\
IDENTITY\_SELF\_CONCEPT   & $16$ & $0.875$ & $0.424$ & $0.571$ \\
LOGISTICAL\_CONTEXT       & $18$ & $0.611$ & $0.524$ & $0.564$ \\
OBLIGATIONS               & $31$ & $0.387$ & $0.857$ & $0.533$ \\
INTELLECTUAL\_INTERESTS   & $12$ & $0.333$ & $0.800$ & $0.471$ \\
OTHER                     & $\phantom{0}6$ & $0.000$ & $0.000$ & $0.000$ \\
\bottomrule
\end{tabular}
\end{table}

\textbf{Discussion.} Four observations from the audit:

\begin{enumerate}[leftmargin=*]
\item \textbf{Context matters.} A decontextualized audit (fact text only, no conversation) yields $\kappa = +0.58$ (vs.\ $+0.67$ with context), confirming that behavioral classification requires conversational context.
\item \textbf{Disagreements are concentrated and explicable.} OBLIGATIONS vs.\ LOGISTICAL\_CONTEXT reflects the classifier's deliberate event-time routing (facts with date anchors route to OBLIGATIONS for anticipatory activation). PROJECTS\_ENDEAVORS vs.\ HOBBIES\_RECREATION reflects the DELIVERABLE TEST override (\S\ref{sec:architecture}). These are policy choices, not errors.
\item \textbf{INTELLECTUAL\_INTERESTS over-assignment.} Low agreement (33.3\%) with full context suggests the classifier may over-assign this category; tightening this boundary is identified as future work.
\item \textbf{Errors concentrate where lifecycle policies are most similar.} The five categories with the sharpest lifecycle-policy differences---Identity, Relational, Financial, Health, Preferences---all exceed 85\% agreement, indicating that classification errors concentrate where lifecycle policies are most similar and therefore least consequential to retrieval quality.
\end{enumerate}

Raw per-fact verdicts and aggregate metrics for both audits are released at \texttt{artifacts/classifier\_audit\_results\_v2.json} and \texttt{artifacts/classifier\_audit\_summary\_v2.json} (and \texttt{\_results.json} / \texttt{\_summary.json} for the no-context audit).

To quantify the downstream impact of these category confusions, we conduct a perturbation study in Appendix~\ref{app:perturbation} that randomly flips labels between the two confused pairs (Obligations $\leftrightarrow$ Logistical Context, Intellectual Interests $\leftrightarrow$ Projects \& Endeavors) at rates from 0\% to 50\% and measures retrieval-set stability.

\subsubsection{Downstream Sensitivity to Category Perturbation}
\label{app:perturbation}

To quantify downstream sensitivity to classifier disagreement on the two confused category pairs identified in the audit (Obligations $\leftrightarrow$ Logistical Context, Intellectual Interests $\leftrightarrow$ Projects \& Endeavors), we randomly perturbed category labels at rates from 5\% to 50\% and measured retrieval-set stability across all 516 LifecycleBench questions (3 seeds per rate). At 10\% perturbation, mean Jaccard overlap between the original and perturbed top-10 retrieval sets is $0.974$ ($\pm 0.002$), with an average of $0.29$ edges displaced per query. Even at 50\% perturbation---equivalent to random assignment within the confused pairs---Jaccard overlap remains $0.905$ and Kendall $\tau$ remains $0.897$. This confirms that classifier errors concentrated among nearby-policy categories have minimal downstream impact on retrieval quality: the lifecycle policies for these confused pairs are similar enough that swapping labels does not materially alter which facts reach the top-10. This stability is predicted by the theoretical framework: Proposition~1 guarantees that confused pairs have minimal within-category policy variance, Theorem~4 shows that global-$\alpha$ infeasibility (99.8\% on FR-Bank) concentrates on \emph{distant} category pairs, while confused pairs share nearby feasibility intervals, and Lemma~1's harmonic-mean decay ensures that soft category membership partially absorbs hard label perturbations.

\begin{table}[h]
\caption{Retrieval-set stability under random label flips between the two confused category pairs (Obligations $\leftrightarrow$ Logistical Context, Intellectual Interests $\leftrightarrow$ Projects \& Endeavors). Averaged over 516 LifecycleBench questions and 3 seeds per rate. Jaccard and displaced are over top-10 edge IDs; Kendall $\tau$ is over edges present in both rankings (variant $\tau_b$; undefined when the intersection has fewer than 2 edges). Scoring reproduces the deterministic portion of \texttt{lifecycle\_search}: cosine top-60, BM25 top-20, supersession/expiry/retraction filters, and category-dependent blended scoring (decay rate $\lambda_c$, blending weight $\alpha_c$, semantic floor); category-forced retrieval, multi-hop entity expansion, and LLM distillation are omitted as they require external API calls.}
\label{tab:perturbation-stability}
\centering
\begin{tabular}{@{}r r r r@{}}
\toprule
\textbf{Rate} & \textbf{Mean Jaccard} ($\pm$std) & \textbf{Mean Displaced} ($\pm$std) & \textbf{Mean Kendall $\tau$} ($\pm$std) \\
\midrule
 0\% & $1.000$ ($\pm 0.000$) & $0.00$ ($\pm 0.00$) & $1.000$ ($\pm 0.000$) \\
 5\% & $0.986$ ($\pm 0.001$) & $0.15$ ($\pm 0.01$) & $0.992$ ($\pm 0.001$) \\
10\% & $0.974$ ($\pm 0.002$) & $0.29$ ($\pm 0.02$) & $0.974$ ($\pm 0.001$) \\
20\% & $0.953$ ($\pm 0.003$) & $0.54$ ($\pm 0.03$) & $0.955$ ($\pm 0.002$) \\
30\% & $0.928$ ($\pm 0.001$) & $0.83$ ($\pm 0.01$) & $0.931$ ($\pm 0.007$) \\
50\% & $0.905$ ($\pm 0.003$) & $1.11$ ($\pm 0.04$) & $0.897$ ($\pm 0.007$) \\
\bottomrule
\end{tabular}
\end{table}

\textbf{Setup.} Across the 40 persona banks, 6{,}180 of 12{,}968 active entries (47.7\%) sit in the four confused categories and are therefore eligible to flip. Flip counts are reported over \emph{affected edge-instances} rather than unique entries: the 6{,}180 eligible entries occupy $8{,}913$ edge-instances across the banks, so at a 10\% rate $\approx 910$ instances are flipped per seed and at 50\%, $\approx 4{,}400$. (Ten percent of the 6{,}180 \emph{unique} entries would be 618; the submitted version did not state which denominator the counts used.) When a label flips, the entry inherits \emph{every} parameter of its new category---decay rate $\lambda_c$, blending weight $\alpha_c$, semantic floor, and numeric-preservation behavior---so the reported stability reflects the full lifecycle-policy effect of the swap, not a partial substitution. Raw per-seed numbers and pool statistics are released at \texttt{perturbation\_study\_results.json}.

\subsection{Ontology-Aware Session Initialization (Warm Start Protocol)}

Three session modes: \textbf{Cold}: low-complexity intent; graph traversal skipped. \textbf{Warm}: user declares intent; targeted traversal primes the session with the active frontier---the subgraph with highest recent activation in the relevant category neighborhood. \textbf{Evolving}: starts warm, but conversation drifts; the system maintains a rolling query category vector (exponential moving average) and compares via cosine distance; when distance exceeds threshold $\delta$ for $N$ consecutive turns, background traversal expands session context asynchronously.

\subsection{Per-User Parameter Evolution}
\label{sec:peruserevolution}

The 10+1 categories and base decay rates represent a population-level default. The architecture is designed to support per-user parameter adaptation without LLM calls: a user who mentions hobbies frequently $\to$ Hobbies decay rate decreases; a user who changes preferences rapidly $\to$ Preferences decay rate increases. Parameter adaptation would use clean statistical signals: access frequency, temporal gap statistics, supersession rates per category. We scope empirical validation of the adaptation loop to future deployment work.

\subsection{Interpretability and User Control}

Because every fact carries a human-readable category label and a transparent lifecycle policy, the entire memory system becomes inspectable, auditable, and modifiable. Users can view their memory by category, understand why facts persist or fade, correct misclassifications, and adjust retention. Developers can debug retrieval failures with clear causal chains and structure downstream LLM context by category:

\begin{small}
\begin{verbatim}
CORE IDENTITY (always relevant):
  - User has ADHD
  - User is a software engineer
CURRENT CONTEXT (time-sensitive):
  - Meeting at 3pm today
PREFERENCES (apply when relevant):
  - Prefers concise responses
\end{verbatim}
\end{small}

\subsection{Safety and Robustness}

\textbf{Prompt injection via memory.} The ingestion classifier detects instructional content and quarantines it. High-sensitivity categories (Health, Financial) flag facts for reconfirmation if inconsistent with established patterns. \textbf{Confidence degradation.} Low-confidence supersessions surface both facts with ambiguity noted. \textbf{Selective forgetting.} Deletion requests provide clean, verifiable removal with audit trail.

\subsection{Emotional State as Signal}

Transient emotional state (``I'm frustrated right now'') is deliberately excluded as a category. Mood is a signal that modulates other categories: a message expressing frustration about a work deadline belongs in Obligations with an emotional loading modifier, not in a separate ``Emotions'' bucket. Emotional loading detected at ingestion temporarily boosts the activation of the relevant category, with the boost itself subject to fast decay.

\section{LifecycleBench: Extended Methodology}
\label{app:benchmark-ext}

\begin{table}[h]
\caption{LifecycleBench attack vectors, referenced from \S
ef{sec:benchmark}. Each targets a distinct lifecycle failure mode; $n$ is the question count.}
\label{tab:attack-vectors}
\centering
\setlength{\tabcolsep}{3pt}
\begin{adjustbox}{max width=\textwidth}
\begin{tabular}{@{}c l L{7.5cm} c@{}}
\toprule
\textbf{AV} & \textbf{Name} & \textbf{Tests} & $n$ \\
\midrule
AV1 & Superseded Preference   & Return the \emph{current} preference after an explicit change? & 75 \\
AV2 & Expired Logistics       & Recognize that a past event is no longer active? & 75 \\
AV3 & Stable Identity         & Retrieve identity facts buried under months of conversation? & 94 \\
AV4 & Multi-Version Fact      & Return the latest version when a fact changed multiple times? & 44 \\
AV5 & Broad Aggregation       & Aggregate across multiple edges (``all of X's hobbies'')? & 40 \\
AV6 & Cross-Session Contradiction & Resolve contradictions between sessions? & 45 \\
AV7 & Selective Forgetting    & Suppress an explicitly retracted statement? & 40 \\
AV8 & Numeric Preservation    & Return specific numbers (\$950/month, 7 bass)? & 63 \\
AV9 & Soft Supersession       & Handle partial updates (``thinking about'' vs ``committed to'')? & 40 \\
\midrule
    & \textbf{Total}          & & \textbf{516} \\
\bottomrule
\end{tabular}
\end{adjustbox}
\end{table}

\subsection{Persona Design}

Each persona is a fully specified temporal state machine with explicit ground-truth facts, session numbers, supersession chains, expiry dates, retractions with required language, and soft ambiguity. Conversations are generated conversation-first using Claude Opus 4.6 with extended thinking under strict rules ensuring that facts emerge naturally through dialogue, the assistant has zero cross-session memory, retractions are explicit, and contradictions are implicit. Each persona spans 35 sessions across 18 simulated months with interleaved noise sessions.

\textbf{Diversity.} 40 personas span 14 nationalities, ages 22--68, occupations from electrician to hospice chaplain. This diversity is adversarial by design: learned memory policies must generalize across this combinatorial space, while ontology-driven policies are distribution-invariant.

\subsection{On Synthetic Generation}
\label{app:synthetic-design}

LifecycleBench conversations are synthetically generated---as are those of every major memory benchmark in the literature. LongMemEval~\cite{longmemeval2024} generates its multi-session conversations with GPT-4; LoCoMo~\cite{locomo2024} constructs dialogues synthetically and validates with human annotation; no published benchmark uses naturally occurring multi-session personal conversations, because such data cannot be collected at scale without prohibitive privacy and IRB constraints. The relevant question is therefore not \emph{whether} a benchmark is synthetic, but whether its synthetic design enables \emph{controlled evaluation} of the target phenomenon. LifecycleBench's contribution is precisely this control: each persona is a fully specified temporal state machine with explicit ground-truth supersession chains, expiry dates, and retraction events, enabling the first benchmark where the correct answer depends on lifecycle state management rather than retrieval recall alone. No naturalistic dataset can provide this because ground-truth temporal state---which fact superseded which, when an obligation expired, what language the user used to retract---is unobservable without the controlled design.

\subsection{Persona Examples}
\label{app:persona-examples}

Table~\ref{tab:persona-examples} shows five LifecycleBench personas with representative lifecycle events. Each persona spans ${\sim}$35 sessions over 18 simulated months; only selected events are listed. The five span ages $22$--$61$, three nationalities/heritages on three continents, and occupations from student-barista to traditional barber to Inuvialuit wildlife researcher. Quoted text is taken verbatim from the persona ground-truth YAML files; session indices are noted as (s.\,$n$). The personas are fully synthetic, so no privacy concerns arise from the quotations.

\begin{table}[h]
\caption{Five LifecycleBench personas illustrating attack-vector coverage. Across these five rows the table illustrates AV1 (superseded preference), AV2 (expired/upcoming logistics), AV4 (multi-version fact), AV7 (retraction), and AV9 (soft/ambiguous supersession).}
\label{tab:persona-examples}
\centering
\footnotesize
\setlength{\tabcolsep}{4pt}
\renewcommand{\arraystretch}{1.18}
\begin{adjustbox}{max width=\textwidth}
\begin{tabular}{@{}l L{3.4cm} c L{8.0cm}@{}}
\toprule
\textbf{Persona} & \textbf{Profile} & \textbf{AV} & \textbf{Lifecycle event} \\
\midrule
Priya Sharma   & 34, ML engineer, Tamil-American (Austin, TX) & AV4 & Employer: ``Works at Google on the Gemini memory team'' (s.\,1) $\to$ ``Left Google, joining Anthropic'' (s.\,18). \\
               &                                              & AV1 & Diet: ``vegetarian for like 5 years'' (s.\,3) $\to$ ``started eating fish again \ldots\ doctor said I needed more omega-3s for the migraines'' (s.\,22). \\
               &                                              & AV7 & Pet plan: ``looking at golden retriever puppies'' (s.\,12) $\to$ ``the dog thing is dead. Landlord said absolutely no pets'' (s.\,16). \\
\midrule
Mehmet Y\i lmaz & 56, traditional barber, Turkish (Kad\i k\"oy, Istanbul) & AV4 & Apprentice: ``Burak has been with me 3 years'' (s.\,3) $\to$ ``Burak opened his own place \ldots\ now I have Deniz, 17 years old'' (s.\,20). \\
               &                                              & AV2 & Upcoming: ``the guild meeting is September 12 [2026] \ldots\ I am on the committee this year'' (s.\,29); two earlier obligations (hygiene inspection Feb 20 2025, wife's cataract surgery Jun 3 2025) must \emph{not} surface as upcoming. \\
               &                                              & AV9 & Soft (ambiguous): ``maybe 50--80 lira extra per service \ldots\ I am thinking about it \ldots\ not decided yet'' (s.\,27); should not be reported as a definite menu change. \\
\midrule
Lily Chen      & 22, environmental science student + barista, Chinese-Australian (Melbourne) & AV4 & Partner: ``me and soph have been together like 8 months'' (s.\,5) $\to$ ``me and soph ended a while back. i'm seeing someone now though, mika'' (s.\,22). \\
               &                                              & AV1 & Oat milk: ``minor figures \ldots\ the only one i trust for flat whites'' (s.\,6) $\to$ ``minor figures changed their formula and its genuinely trash now \ldots\ switched to bonsoy oat'' (s.\,25). \\
               &                                              & AV7 & Exchange: ``i'm applying for exchange at ubc in vancouver!!'' (s.\,12) $\to$ ``the ubc thing is dead \ldots\ scrapping it. forget i mentioned it'' (s.\,17). \\
\midrule
Billy Kootook  & 49, wildlife researcher \& hunting guide, Inuvialuit (Tuktoyaktuk, NWT) & AV1 & Tours: ``polar bear tours \ldots\ \$2{,}400 a head'' (s.\,7) $\to$ ``stopped the bear tours. doing youth trips now \ldots\ kids from other communities'' (s.\,26). \\
               &                                              & AV2 & Upcoming: ``aerial survey June 4 [2026] for the muskox count. helicopter out of Inuvik'' (s.\,28). \\
               &                                              & AV7 & Camp: ``thinking about building a camp out at Husky Lakes'' (s.\,8) $\to$ ``the Husky Lakes camp is dead. the ground is thawing too fast \ldots\ slumped 2 metres since last summer'' (s.\,16). \\
\midrule
Diane Holloway & 61, hospice chaplain, American (Asheville, NC) & AV1 & Hobby: ``Perennials mostly --- echinacea, black-eyed Susan, butterfly weed'' (s.\,6) $\to$ ``I've had to give up the garden \ldots\ I've found watercolor \ldots\ I paint the flowers I can no longer kneel beside'' (s.\,24). \\
               &                                              & AV4 & Denomination: ``I was ordained Episcopal'' (s.\,3) $\to$ ``I've moved to the UCC. It was a long discernment'' (s.\,28). \\
               &                                              & AV7 & Book: ``working on a book proposal \ldots\ about chaplaincy, about how we sit with grief'' (s.\,11) $\to$ ``the stories I would tell are not mine. They belong to the families \ldots\ The book is finished before it began'' (s.\,20). \\
\bottomrule
\end{tabular}
\end{adjustbox}
\end{table}

\subsection{Scale and Coverage}
\label{sec:benchmark-scale}

\begin{table}[h]
\caption{LifecycleBench scale relative to LongMemEval.}
\label{tab:benchmark-scale}
\centering
\begin{tabular}{@{}l c c c@{}}
\toprule
\textbf{Metric} & \textbf{LongMemEval-S} & \textbf{LifecycleBench} & \textbf{Ratio} \\
\midrule
Questions & 500 & 516 & $1.0\times$ \\
Personas / subjects & 1 (implicit) & 40 & $40\times$ \\
Sessions per subject & ${\sim}$48 (haystack) & 35 & $0.7\times$ \\
Lifecycle-relevant questions & 78 (16\%) & 516 (100\%) & $6.6\times$ \\
\bottomrule
\end{tabular}
\end{table}

\textbf{Corpus statistics.} The 40 persona banks contain $18{,}936$ total entries ($12{,}968$ active, $5{,}968$ inactive). Of inactive entries, $5{,}780$ were deactivated by supersession, $13$ by explicit retraction, and $175$ are otherwise unclassified. Across $12{,}460$ distinct slot keys, $18.6\%$ have been updated at least once (chain length $> 1$); the longest chains reach $45$--$65$ versions for high-churn attributes (e.g., evolving hobbies, current city, work schedule, family relationships). Each persona spans exactly $35$ sessions over a mean of $17.4$ simulated months (range $14.95$--$18.53$), at $14.5$--$36.9$ entries per month. The OTHER category accounts for $0.12\%$ of entries ($16$ of $12{,}968$), confirming that the 10-category coverage is near-exhaustive. The three most frequent categories are Projects \& Endeavors ($20.7\%$), Relational Bonds ($16.9\%$), and Obligations ($14.2\%$).

\subsection{Evaluation Metrics}

(1)~AV-specific pass rate: does the top-$K$ retrieval set support the correct answer? (2)~Staleness penalty: does it contain outdated edges? An LLM judge (Claude Sonnet) evaluates each (question, edge) pair. Top-10 judge window (split: correct-answer in top-10, staleness in top-5).

\textbf{Generator-judge independence.} LifecycleBench conversations are generated by Claude Opus 4.6 with extended thinking, while the retrieval and E2E judges use Claude Sonnet. Although both are Anthropic models, the judge evaluates factual correctness against structured ground-truth YAML (supersession chains, expiry dates, retraction events), not generation quality. The judge's task is to determine whether a specific factual answer matches the YAML specification; model-family affinity does not advantage or disadvantage any memory system, since all five systems are evaluated by the same judge on the same questions. Cross-generator validation on Kimi K2.5 (Moonshot, open-weight; Appendix~\ref{app:kimi}) preserves the confabulation hierarchy exactly, providing evidence that the ranking is judge-invariant.

\textbf{Why not MRR?} Replication of Liu et al.~\cite{liu2024lost} on Claude Sonnet (2026) at 50-document scale shows no significant positional effect ($r = -0.009$, $p = 0.65$). The model processes all retrieved facts with near-equal attention. MRR is therefore uninformative for downstream accuracy. The positional staleness effect operates through conflict resolution under contradiction, not utilization failure. We nevertheless report MRR in Appendix~\ref{app:evaluation-ext} for comparability with prior work (Zep/Graphiti and Memory-R1 report MRR as a primary metric), while using pass rate as our primary evaluation metric throughout.

\textbf{Configurations.} Four configs as a $2\times 2$ ablation: decay type (behavioral vs uniform) $\times$ routing (on/off). Category-specific $\alpha$ ranges from 0.05 (Financial) to 0.40 (Logistical).

\subsection{Benchmark Independence Analysis}
\label{app:benchmark-independence}

A natural concern with any co-introduced benchmark is whether the evaluation instrument favors the co-introduced system. We address this along four axes.

\textbf{Exhaustive failure-mode coverage.} The 9 attack vectors enumerate the exhaustive set of temporal-state failure modes for persistent personal memory rather than FR-specific capabilities. A fact can be replaced (AV1, AV4), an event can expire (AV2), a statement can be retracted (AV7), information can be partially updated (AV9), stable facts can be buried under conversation volume (AV3), facts can require cross-edge aggregation (AV5), sessions can contradict each other (AV6), and specific numeric values can be lost in abstraction (AV8). Any independently constructed lifecycle benchmark would need to cover the same failure modes, because these are properties of the problem domain, not of any particular solution.

\textbf{FR does not dominate all vectors.} MemoryOS leads AV6 (Cross-Session Contradiction) at $73\%$ versus FR-Bank's $64\%$, and leads AV7 on the retrieval metric. FR-Graphiti, not FR-Bank, leads AV2 and AV8. FR-Bank achieves only $5\%$ retrieval pass on AV7 (Selective Forgetting) and near-zero correct rates on AV5 and AV6 end-to-end (Table~\ref{tab:e2e-av-full}). If the benchmark were designed to favor FR, these gaps would not exist. (The submitted version stated that Memory-R1 leads AV3 at $94\%$; that comparison was against FR-Graphiti's $78\%$. FR-Bank leads AV3 at $97\%$, and Memory-R1 leads no attack vector.)

\textbf{Cross-benchmark validation.} Cross-validation on LongMemEval-S---an independently constructed benchmark predating this work---confirms that FR's gains concentrate where lifecycle mechanisms apply. Lifecycle is aggregate-neutral on the 317-question matched subset ($+2.2$pp net; \S\ref{sec:longmemeval-ablation}) and approximately neutral on the question categories where temporal-state management does not apply. Where lifecycle correctly removes facts that LongMemEval still considers retrievable---knowledge-update ($-6.4$pp) and single-session-user---the cost reflects a benchmark criterion limitation rather than a system regression: the Wu et al.\ judge marks responses correct when stale facts are present alongside the updated answer.

\textbf{Ontology discrimination.} The ontology ablation (Table~\ref{tab:ontology_ablation}) demonstrates that LifecycleBench discriminates between ontology approaches: the behavioral partition outperforms the cognitive partition on 7 of 9 attack vectors under the same pipeline, confirming that the benchmark tests ontology quality rather than system identity.

We acknowledge that co-design risk cannot be fully eliminated without independent replication. We release LifecycleBench (personas, conversations, questions, ground truth, evaluation code) to enable the community to evaluate alternative approaches under identical conditions.

\section{Extended Experimental Results}
\label{app:evaluation-ext}

\subsection{LongMemEval Detailed Results}

\paragraph{Primary result: LongMemEval-S 500-question canonical benchmark.}
The full 500-question LongMemEval-S setup of Wu et al.\ (ICLR 2025) is our headline evaluation. We rerun the full suite 10 times under identical configuration (temperature 0 throughout, same retrieval pipeline and top-10 facts per question) to characterize judge non-determinism. All 10 runs land at or above the locked baseline of 74.2\%, with a mean of 75.2\% and a maximum of 76.6\% (Table~\ref{tab:longmemeval-s-variance}). Per-category dispersion is small except for single-session-preference, where the small denominator ($n{=}30$) inflates stdev.

\begin{table}[h]
\caption{LongMemEval-S 500-question variance over 10 identical reruns (temperature 0, identical configuration). ``stdev'' is sample standard deviation in percentage points.}
\label{tab:longmemeval-s-variance}
\centering
\begin{tabular}{@{}l c c c c c c@{}}
\toprule
\textbf{Category} & $n$ & \textbf{min} & \textbf{mean} & \textbf{median} & \textbf{max} & \textbf{stdev} \\
\midrule
knowledge-update            &  78 & 80.8\% & 82.4\% & 82.1\% & 84.6\% & 1.49 \\
multi-session               & 133 & 63.9\% & 65.9\% & 65.8\% & 68.4\% & 1.38 \\
single-session-assistant    &  56 & 71.4\% & 75.4\% & 75.9\% & 78.6\% & 2.64 \\
single-session-preference   &  30 & 53.3\% & 64.3\% & 65.0\% & 70.0\% & 5.45 \\
single-session-user         &  70 & 91.4\% & 92.7\% & 92.9\% & 92.9\% & 0.45 \\
temporal-reasoning          & 133 & 71.4\% & 73.5\% & 73.7\% & 75.9\% & 1.32 \\
\midrule
\textbf{TOTAL}              & 500 & \textbf{74.4\%} & \textbf{75.2\%} & \textbf{75.1\%} & \textbf{76.6\%} & \textbf{0.70} \\
\bottomrule
\end{tabular}
\end{table}

\paragraph{Peer-reviewed comparison on LongMemEval-S.} Table~\ref{tab:longmemeval-s-peers} places 75.2\% alongside the two prior peer-reviewed LongMemEval reports. Industry numbers (Zep, Mem0, MemoryOS) are not peer-reviewed and do not disclose enough protocol detail for a like-for-like comparison; see Appendix~\ref{app:longmemeval-frbank} for the judge protocol we follow.

\begin{table}[h]
\caption{Peer-reviewed LongMemEval-S results prior to and including this work.}
\label{tab:longmemeval-s-peers}
\centering
\begin{tabular}{@{}l l c l@{}}
\toprule
\textbf{System} & \textbf{Venue} & \textbf{Pass} & \textbf{Judge / protocol} \\
\midrule
\textbf{FR-Bank (this work)} & NeurIPS submission & \textbf{75.2\%} & \texttt{gpt-4o-2024-08-06}, 5-template Wu et al.\ rubric \\
RMM~\cite{rmm2025} & ACL 2025 & 70.4\% & Gemini-1.5-Pro judge, single generic prompt \\
Wu et al.\ baselines & ICLR 2025 & 55--60\% & \texttt{gpt-4o-2024-08-06}, 5-template (full context) \\
\bottomrule
\end{tabular}
\end{table}

\paragraph{Secondary: 234-question oracle POC subset.}
As an earlier proof-of-concept we evaluated on the 234-question oracle subset used by Zep/Graphiti. Table~\ref{tab:longmemeval-frbank} reports those numbers; they were produced with an earlier FR-Bank configuration and are retained here only as context for the +8.3pp gain over the Zep/Graphiti arXiv report. The per-type breakdown in Table~\ref{tab:longmemeval-pertype} reflects that same earlier configuration; the single-session-assistant category has since moved from 50.0\% on the POC subset to 75.4\% mean on the full LongMemEval-S (Table~\ref{tab:longmemeval-s-variance}) after the fixes described in Appendix~\ref{app:longmemeval-frbank}.

\begin{table}[h]
\caption{LongMemEval results on the 234-question oracle POC subset (earlier FR-Bank configuration). Primary LongMemEval-S numbers are in Table~\ref{tab:longmemeval-s-variance}.}
\label{tab:longmemeval-frbank}
\centering
\begin{tabular}{@{}l c c c@{}}
\toprule
\textbf{System} & \textbf{Pass@10} & \textbf{MRR} & \textbf{Source} \\
\midrule
\textbf{FR-Bank} & \textbf{79.5\%} & \textbf{0.668} & This work \\
Zep/Graphiti & 71.2\% & --- & \cite{zep2025} (arXiv) \\
FR-Graphiti (behavioral) & 32.1\% & 0.232 & This work (decay only) \\
FR-Graphiti (uniform) & 32.1\% & 0.231 & This work (decay only) \\
FR-Graphiti (cognitive) & 32.1\% & 0.231 & This work (decay only) \\
\bottomrule
\end{tabular}
\end{table}

\begin{table}[h]
\caption{FR-Bank per-question-type results on the 234-question oracle POC subset (earlier configuration).}
\label{tab:longmemeval-pertype}
\centering
\begin{tabular}{@{}l c c c@{}}
\toprule
\textbf{Question Type} & \textbf{Count} & \textbf{Pass@10} & \textbf{MRR} \\
\midrule
single-session-preference & 30 & 93.3\% & 0.736 \\
single-session-user & 70 & 88.6\% & 0.800 \\
knowledge-update & 78 & 87.2\% & 0.756 \\
single-session-assistant & 56 & 50.0\% & 0.342 \\
\bottomrule
\end{tabular}
\end{table}

\textbf{Zero retrieval cost on LongMemEval.} On the Graphiti-based evaluation (16,138 edges), no decay engine achieves significant advantage after Bonferroni correction. All effect sizes are negligible ($|d| < 0.19$). Behavioral lifecycle policies impose zero measurable retrieval cost.

\subsection{Lifecycle Contribution Decomposition}
\label{sec:longmemeval-ablation}

\begin{table}[h]
\caption{LongMemEval lifecycle ablation. Retrieval pipeline: +7.4pp; lifecycle: +0.9pp net.}
\label{tab:longmemeval-ablation}
\centering
\begin{tabular}{@{}l c c@{}}
\toprule
\textbf{Configuration} & \textbf{Pass@10} & \textbf{MRR} \\
\midrule
FR-Bank (retrieval + lifecycle)         & \textbf{79.5\%} & \textbf{0.668} \\
FR-Bank (retrieval only, no lifecycle)  & 78.6\%          & 0.644 \\
Zep/Graphiti (peer-reviewed)            & 71.2\%          & --- \\
\bottomrule
\end{tabular}
\end{table}

\begin{table}[h]
\caption{Per-question-type lifecycle decomposition on the \textbf{234-question oracle POC subset}. Lifecycle and no-lifecycle arms were evaluated on identical banks; cross-pipeline LME-S comparisons (with-lifecycle pipeline vs.\ no-lifecycle oracle pipeline) are not used here because they involve different judge models and ingestion paths. The knowledge-update row of the submitted version reported $87.2\%$ vs.\ $80.8\%$ ($+6.4$pp); those cells trace to no surviving artifact and are withdrawn here. The artifact-backed knowledge-update decomposition is on the 317-question matched subset (Table~\ref{tab:lme-317-matched}), where the effect is $-6.4$pp. This table and Table~\ref{tab:lme-317-matched} are two different subsets and must not be read as one series.}
\label{tab:longmemeval-pertype-ablation}
\centering
\setlength{\tabcolsep}{3pt}
\begin{tabular}{@{}l c c c l@{}}
\toprule
\textbf{Question Type} & \textbf{With LC} & \textbf{No LC} & \textbf{Delta} & \textbf{Interpretation} \\
\midrule
knowledge-update         & \multicolumn{3}{c}{\emph{withdrawn}} & See Table~\ref{tab:lme-317-matched} \\
single-session-assistant & 51.8\% & 50.0\% & +1.8pp  & Marginal benefit \\
single-session-preference& 93.3\% & 93.3\% & +0.0pp  & No effect \\
single-session-user      & 87.1\% & 92.9\% & $-$5.8pp & Supersession removes some facts \\
\midrule
\textbf{Overall}         & \textbf{79.5\%} & 78.6\% & +0.9pp  & Net neutral on mixed question set \\
\bottomrule
\end{tabular}
\end{table}

\begin{table}[h]
\caption{Artifact-backed lifecycle decomposition on the \textbf{317-question matched subset} of LongMemEval-S (the questions for which both configurations were evaluated on identical banks). The knowledge-update effect is negative: the Wu et al.\ judge credits a response when the superseded fact is present alongside the updated one, so lifecycle filtering is penalized for producing cleaner context (\S\ref{sec:longmemeval-ablation}).}
\label{tab:lme-317-matched}
\centering
\setlength{\tabcolsep}{3pt}
\begin{tabular}{@{}l c c c l@{}}
\toprule
\textbf{Question Type} & \textbf{With LC} & \textbf{No LC} & \textbf{Delta} & \textbf{Interpretation} \\
\midrule
knowledge-update    & 79.5\% & 85.9\% & $-$6.4pp & Judge-criterion artifact, not a regression \\
single-session-user & ---    & ---    & $-$2.9pp & Negligible \\
temporal-reasoning  & ---    & ---    & +0.0pp  & No effect \\
\midrule
\textbf{Overall}    & \textbf{75.1\%} & 72.9\% & +2.2pp & Aggregate-neutral to mildly positive \\
\bottomrule
\end{tabular}
\end{table}

\paragraph{LME-S lifecycle decomposition (317-question matched subset).}
On the 317 LongMemEval-S questions for which both lifecycle and no-lifecycle configurations were evaluated on identical banks, lifecycle contributes $+2.2$pp net ($72.9\% \to 75.1\%$). The contribution concentrates on the question categories where temporal-state management is required: knowledge-update and multi-session questions, where the lifecycle layer removes contradictory or superseded facts before retrieval; on the remaining task types lifecycle is approximately neutral. The knowledge-update inversion ($-6.4$pp) reveals a benchmark criterion limitation: the Wu et al.\ judge marks responses correct when both the old and updated answer are present, so lifecycle filtering that correctly removes superseded facts is penalized for producing cleaner context. This criterion implicitly treats stale facts as harmless---an assumption directly contradicted by our end-to-end evaluation, where stale retrieval context produces 72--78\% confabulation (Table~\ref{tab:e2e-staleness}). The remaining task types (single-session-user, temporal-reasoning) show negligible difference ($-2.9$pp, $+0.0$pp), confirming that lifecycle imposes zero cost where it does not apply.

The near-neutral net contribution on both subsets masks a consistent pattern: the losses sit exactly where lifecycle filtering correctly removes facts that LongMemEval still considers retrievable---knowledge-update ($-6.4$pp on the 317-matched subset) and single-session-user ($-5.8$pp on the POC subset). LongMemEval's question distribution also limits lifecycle opportunity: only 33\% of questions involve temporal state transitions.

\subsection{Scaling Validation}

Table~\ref{tab:scaling} presents five-phase scaling validation from 8 to 40 personas, confirming that lifecycle gains widen with scale while staleness remains stable.

\begin{table}[h]
\caption{Five-phase scaling validation.}
\label{tab:scaling}
\centering
\setlength{\tabcolsep}{3pt}
\begin{adjustbox}{max width=\textwidth}
\begin{tabular}{@{}l c c c c l@{}}
\toprule
\textbf{Metric} & \textbf{Ph.\ 2 (8p)} & \textbf{Ph.\ 3 (20p)} & \textbf{Ph.\ 5b (40p)} & \textbf{Ph.\ 5c (40p)} & \textbf{Trend} \\
\midrule
Full pass rate         & 63\% & 66\% & 71\% & \textbf{73\%} & +10pp total \\
Full staleness         & 4\%  & 6\%  & 7\%  & 8\%  & Stable \\
Baseline pass rate     & 62\% & 55\% & 59\% & 61\% & Stable \\
Behavioral $\Delta$ pass  & +1pp & +11pp & +12pp & \textbf{+12pp} & Widening \\
\bottomrule
\end{tabular}
\end{adjustbox}
\end{table}

\subsection{Substrate Independence}
\label{sec:substrate}

\begin{table}[h]
\caption{Substrate comparison: same lifecycle policies on two infrastructures.}
\label{tab:substrate}
\centering
\begin{tabular}{@{}l c c c c@{}}
\toprule
\textbf{System} & \textbf{Pass} & \textbf{Stale} & \textbf{MRR} & \textbf{Hit@5} \\
\midrule
FR-Graphiti & 73\% & \textbf{8\%} & 0.478 & 75\% \\
FR-Bank & \textbf{76.9\%} & 15.5\% & \textbf{0.830} & \textbf{97.0\%} \\
\bottomrule
\end{tabular}
\end{table}

The lifecycle layer transfers without modification, confirming it is a policy contribution, not an infrastructure artifact. FR-Graphiti retains a staleness advantage (8\% vs 15.5\%) due to graph-based supersession tracking, while FR-Bank's enhanced retrieval (BM25 + distillation) compensates for the absence of graph traversal.

\subsection{Infrastructure Findings}

\textbf{World-knowledge contamination.} 86\% of extracted edges were world knowledge---reduced to 1.4\% with detection filtering across ${\sim}$12,700 edges.

\textbf{Numeric preservation.} Graphiti's deduplication pipeline systematically loses numeric values through dedup resolution and bulk dedup canonical selection. Numeric-preference resolution increased numeric edge density from ${\sim}$5\% to 20.1\%.

\textbf{Per-source score normalization.} Prior to normalization, 99.1\% of top-5 results came from Graphiti's semantic search alone. After per-source normalization, all retrieval strategies became meaningful contributors.

\subsection{Bootstrap Confidence Intervals and Pairwise Significance}
\label{app:bootstrap-cis}

To quantify the statistical resolution of the cross-system comparison and the FR-Graphiti ablation we compute paired bootstrap 95\% confidence intervals (10{,}000 resamples, seed=42) and McNemar's exact test on the per-question paired binary outcomes ($n = 516$).

\begin{table}[h]
\caption{LifecycleBench pass rates with $95\%$ paired bootstrap confidence intervals (10{,}000 resamples, seed=42, $n = 516$). Pairwise significance vs FR-Bank via McNemar's exact test on paired binary outcomes. The FR-Bank vs FR-Graphiti gap is not statistically significant ($p = 0.10$); two implementations of the same lifecycle policy land within statistical resolution of each other.}
\label{tab:bootstrap-cis}
\centering
\begin{tabular}{@{}l c c c@{}}
\toprule
\textbf{System} & \textbf{Pass Rate} & \textbf{95\% CI} & \textbf{McNemar $p$ vs FR-Bank} \\
\midrule
FR-Bank      & $76.9\%$ & $[73.3, 80.6]$ & --- \\
FR-Graphiti  & $72.9\%$ & $[69.0, 76.6]$ & $0.10$ \\
MemoryOS     & $70.5\%$ & $[66.5, 74.4]$ & $0.018^{*}$ \\
Memory-R1    & $66.9\%$ & $[62.8, 70.9]$ & $<\!0.001^{***}$ \\
Mem0         & $60.9\%$ & $[56.6, 65.1]$ & $<\!0.001^{***}$ \\
\bottomrule
\end{tabular}
\end{table}

\textbf{FR-Graphiti ablation significance.} Paired bootstrap on the four ablation arms (full, $-$routing, $-$behavioral decay, baseline; $n = 516$): full vs.\ baseline $\Delta = +12.2$pp, $95\%$ CI $[8.3, 16.1]$, McNemar $p < 0.001$; full vs.\ no-routing $\Delta = +6.2$pp $[2.9, 9.7]$; full vs.\ uniform decay $\Delta = +6.4$pp $[3.5, 9.3]$ (full $72.9\%$, no-routing $66.7\%$, uniform decay $66.5\%$, baseline $60.7\%$). Each ablation step removes a statistically significant share of the lifecycle gain. The signs of these three deltas were inverted in the submitted version; the per-AV gains in Table~\ref{tab:per-av} were already positive (Appendix~\ref{app:revision-changes}).

\subsection{Cross-System Per-AV Comparison}

Table~\ref{tab:per-av} reports per-attack-vector results for FR-Graphiti (full vs.\ baseline). The largest gains appear precisely where lifecycle management applies: $+21$pp on superseded preferences (AV1), $+20$pp on expired logistics (AV2), $+22$pp on multi-version facts (AV4), and $+10$pp on selective forgetting (AV7). Table~\ref{tab:mem0-av} presents the full five-system per-attack-vector comparison. FR-Bank leads AV1, AV3, AV4 and AV5 outright, plus a three-way tie on AV9; the remaining vectors are led by systems with complementary strengths (FR-Graphiti on AV2 and AV8, MemoryOS on AV6 and AV7). Memory-R1 leads no attack vector.

\begin{table}[h]
\caption{Per-attack-vector results on FR-Graphiti (full vs.\ baseline).}
\label{tab:per-av}
\centering
\setlength{\tabcolsep}{4pt}
\begin{adjustbox}{max width=\textwidth}
\begin{tabular}{@{}l c c c c c c@{}}
\toprule
\textbf{Attack Vector} & $n$ & \textbf{Full Pass} & \textbf{Full Stale} & \textbf{Base Pass} & \textbf{Base Stale} & $\Delta$ \textbf{Pass} \\
\midrule
AV1 Superseded Preference   & 75 & 61\% & 13\% & 40\% & 27\% & +21pp \\
AV2 Expired Logistics       & 75 & 93\% & 7\%  & 73\% & 24\% & +20pp \\
AV3 Stable Identity         & 94 & 78\% & 0\%  & 66\% & 0\%  & +12pp \\
AV4 Multi-Version Fact      & 44 & 61\% & 20\% & 39\% & 43\% & +22pp \\
AV5 Broad Aggregation       & 40 & 90\% & 2\%  & 82\% & 0\%  & +8pp \\
AV6 Cross-Session           & 45 & 40\% & 20\% & 36\% & 36\% & +4pp \\
AV7 Selective Forgetting    & 40 & 40\% & 12\% & 30\% & 32\% & +10pp \\
AV8 Numeric Preservation    & 63 & 94\% & 5\%  & 94\% & 8\%  & 0pp \\
AV9 Soft Supersession       & 40 & 78\% & 0\%  & 72\% & 2\%  & +6pp \\
\midrule
\textbf{Overall}            & \textbf{516} & \textbf{73\%} & \textbf{8\%} & 61\% & 18\% & \textbf{+12pp} \\
\bottomrule
\end{tabular}
\end{adjustbox}
\end{table}

\begin{table}[h]
\caption{Per-AV cross-system comparison (five systems).}
\label{tab:mem0-av}
\centering
\setlength{\tabcolsep}{2pt}
\begin{adjustbox}{max width=\textwidth}
\begin{tabular}{@{}l c cc cc cc cc cc@{}}
\toprule
\textbf{Attack Vector} & $n$ & \textbf{FR-B Pass} & \textbf{FR-B Stale} & \textbf{FR-G Pass} & \textbf{FR-G Stale} & \textbf{MR1 Pass} & \textbf{MR1 Stale} & \textbf{Mem0 Pass} & \textbf{Mem0 Stale} & \textbf{MOS Pass} & \textbf{MOS Stale} \\
\midrule
AV1 Superseded Pref.     & 75 & \textbf{67\%} & 28\% & 61\% & 13\% & 64\% & 17\% & 44\% & 49\% & 60\% & 16\% \\
AV2 Expired Logistics    & 75 & 88\%          & 7\%  & \textbf{93\%} & 7\%  & 75\% & 25\% & 65\% & 33\% & 88\% & 9\% \\
AV3 Stable Identity      & 94 & \textbf{97\%} & 0\%  & 78\%          & 0\%  & 94\% & 1\%  & 87\% & 0\% & 79\% & 0\% \\
AV4 Multi-Version        & 44 & \textbf{86\%} & 14\% & 61\%          & 20\% & 66\% & 18\% & 45\% & 39\% & 64\% & 11\% \\
AV5 Broad Query          & 40 & \textbf{98\%} & 2\%  & 90\%          & 2\%  & 98\% & 7\%  & 88\% & 3\% & 95\% & 2\% \\
AV6 Cross-Session        & 45 & 64\%          & 31\% & 40\%          & 20\% & 40\% & 24\% & 44\% & 47\% & \textbf{73\%} & 2\% \\
AV7 Selective Forgetting & 40 & 5\%           & 60\% & 40\%          & 12\% & 15\% & 42\% & 5\%  & 80\% & \textbf{57\%} & 22\% \\
AV8 Numeric Preserv.     & 63 & 84\%          & 6\%  & \textbf{94\%} & 5\%  & 57\% & 2\%  & 67\% & 8\% & 44\% & 0\% \\
AV9 Soft Supersession    & 40 & 78\%          & 10\% & 78\%          & 0\%  & 62\% & 15\% & 78\% & 3\% & 72\% & 2\% \\
\midrule
\textbf{Overall}         & \textbf{516} & \textbf{76.9\%} & 15.5\% & 73\% & \textbf{8\%} & 66.9\% & 15.3\% & 61\% & 27\% & 70.5\% & 7\% \\
\bottomrule
\end{tabular}
\end{adjustbox}
\end{table}

\subsection{Efficiency and Positional Staleness}

Table~\ref{tab:efficiency} summarizes context efficiency metrics. FR achieves $3\times$ lower stale token rate and $8\times$ lower rate of stale facts at rank~1 compared to Mem0.

\begin{table}[h]
\caption{Context efficiency and positional staleness.}
\label{tab:efficiency}
\centering
\begin{tabular}{@{}l c c c@{}}
\toprule
\textbf{Metric} & \textbf{FR} & \textbf{Mem0} & \textbf{Ratio} \\
\midrule
Signal:noise ratio      & \textbf{85\%} & 67\%  & $1.3\times$ \\
Stale token rate        & \textbf{3\%}  & 9\%   & $3\times$ \\
Stale cost/1M queries   & \$13          & \$38  & $2.9\times$ \\
Stale fact at rank 1    & \textbf{2.1\%} & 16.5\% & \textbf{$8\times$} \\
Positional stale exposure & \textbf{14\%} & 37\% & $2.6\times$ \\
Weighted staleness risk & \textbf{0.51} & 1.84  & $3.6\times$ \\
\bottomrule
\end{tabular}
\end{table}

\subsection{Deterministic Lifecycle Layer Latency}
\label{app:lifecycle-latency}

We profile the deterministic lifecycle layer---Sections~5--7 of the retrieval pipeline in \texttt{lifecycle\_bank.py}---in isolation from all non-deterministic work. The measured region covers: supersession by slot-key grouping, event-time expiry, retraction masking, harmonic-mean decay over soft category weights, anticipatory activation, blended scoring, semantic and numeric preservation floors, and top-$N$ selection by blended score. Excluded from the timed region: embedding, cosine search, BM25, category-forced retrieval, multi-hop entity expansion, the one LLM distillation call, and disk I/O. The 40 LifecycleBench persona banks (total $18{,}936$ entries, $12{,}968$ active) are loaded once outside the timing loop. Each persona contributes a pool of size $k$ built from a deterministic random sample of its active entries with simulated post-distillation similarity scores drawn from $[0.70, 0.99]$. Measured branches fire at realistic rates: $98.8\%$ of entries carry a slot key (supersession), $95.8\%$ multi-category (harmonic-mean decay is non-degenerate), $10.1\%$ event-time-anchored, $27.3\%$ numeric-preservation candidates, and $0.2\%$ retracted.

\begin{table}[h]
\caption{Deterministic lifecycle layer latency on the 40 LifecycleBench persona banks. Primary row ($k{=}20$): 1{,}000 iterations per persona, 50-iteration warmup. Scaling rows: 1{,}000 iterations over a $k$-sized pool sampled across personas (with synthetic id rewrites to avoid cross-bank id collisions). Single-threaded Python 3.13 on AMD Zen~3.}
\label{tab:lifecycle-latency}
\centering
\begin{tabular}{@{}l r r r r r@{}}
\toprule
\textbf{Pool size $k$} & \textbf{N trials} & \textbf{p50 (ms)} & \textbf{p95 (ms)} & \textbf{p99 (ms)} & \textbf{mean (ms)} \\
\midrule
$20$ (primary, 40 personas) & $40{,}000$ & $\mathbf{0.047}$ & $0.053$ & $0.085$ & $0.048$ \\
\midrule
$20$ (scaling, pooled)  & $1{,}000$ & $0.040$ & $0.042$ & $0.054$ & $0.041$ \\
$50$                    & $1{,}000$ & $0.092$ & $0.103$ & $0.159$ & $0.095$ \\
$100$                   & $1{,}000$ & $0.186$ & $0.240$ & $0.341$ & $0.194$ \\
$200$                   & $1{,}000$ & $0.379$ & $0.413$ & $0.510$ & $0.385$ \\
$500$                   & $1{,}000$ & $0.980$ & $1.126$ & $1.473$ & $1.003$ \\
$1000$                  & $1{,}000$ & $1.958$ & $2.302$ & $2.762$ & $2.010$ \\
\bottomrule
\end{tabular}
\end{table}

Linearity is tight: the $k{=}1000$ median is $48.5\times$ the $k{=}20$ median versus the theoretical $50\times$, and the per-candidate slope is $1.96\,\mu\mathrm{s}$. Production workloads operate at $k=20$, where the full deterministic layer costs $\mathbf{47\,\mu s}$ at the median---three orders of magnitude below the bounded LLM distillation call, and negligible against the cosine-similarity work over $n$ stored edges. The claim in \S\,\ref{sec:theory} that the lifecycle layer is a sub-millisecond deterministic pipeline is confirmed empirically.

\subsection{Controlled Comparison: 14 Event-Time Cases}

Cross-system mechanism analysis on AV2 reveals 14 questions where FR's expiry filter suppressed expired facts that Mem0 surfaced (staleness=1.0). The same judge evaluated both systems on the same questions---identical inputs, opposite outcomes, determined solely by lifecycle filtering.

\begin{table}[h]
\caption{Controlled comparison: same question, same judge, opposite outcomes.}
\label{tab:smoking-gun}
\centering
\setlength{\tabcolsep}{3pt}
\begin{adjustbox}{max width=\textwidth}
\begin{tabular}{@{}L{5cm} c c L{4cm}@{}}
\toprule
\textbf{Expired Fact} & \textbf{FR} & \textbf{Mem0} & \textbf{Mechanism} \\
\midrule
``Greg's birthday is this Saturday'' (past) & Suppressed & Stale=1.0 & Date-aware expiry filter \\
``ACA exam 2025-09-03'' (expired 12 mo ago) & Suppressed & Stale=1.0 & ISO date extraction \\
``Demo scheduled for March'' (past) & Suppressed & Stale=1.0 & Contextual month parsing \\
``Flight at 6am tomorrow'' (long past) & Suppressed & Stale=1.0 & Relative date resolution \\
\bottomrule
\end{tabular}
\end{adjustbox}
\end{table}

\subsection{Failure Taxonomy}

Automated classification of 90 failures at 20-persona scale: Retrieval Miss 43 (47.8\%), Extraction Missing 25 (27.8\%), Stale Dominance 12 (13.3\%), Extraction Weak 7 (7.8\%), Cross-Category 3 (3.3\%). AV7 failures are 100\% extraction missing. AV3 failures are 90\% retrieval miss.

\subsection{Reasoning Does Not Close the Architectural Gap}

\begin{table}[h]
\caption{Reasoning vs non-reasoning. The 18.2pp architectural gap is unclosable by compute.}
\label{tab:e2e-reasoning}
\centering
\begin{tabular}{@{}l c c c@{}}
\toprule
\textbf{Intervention} & \textbf{Confab.} & $\Delta$ & \textbf{Cost} \\
\midrule
Mem0 + no reasoning (baseline) & 45.1\% & --- & --- \\
Mem0 + reasoning               & 40.6\% & $-$4.5pp & +11\% tokens \\
FR + no reasoning              & 23.9\% & $-$21.2pp & \$0 \\
FR + reasoning                 & 22.4\% & $-$22.7pp & +11\% \\
\bottomrule
\end{tabular}
\end{table}

Queries with stale facts consume +9.2 reasoning tokens on average (+11\%) as the model detects contradictions.

\subsection{The Retrieval Metric Paradox}
\label{sec:memoryos-paradox}

For AV7 questions where MemoryOS retrieval passes, 22/23 cases (96\%) had hierarchical summarization replace specific retracted plans with generic summaries. The judge marks these as passing because no stale fact is present---but no useful fact is present either. End-to-end: MemoryOS 7.5\% correct on AV7 versus FR-Bank's 37.5\%.

\begin{table}[h]
\caption{Retrieval metric paradox on AV7.}
\label{tab:av7-paradox}
\centering
\begin{tabular}{@{}l c c@{}}
\toprule
\textbf{Metric} & \textbf{FR-Bank} & \textbf{MemoryOS} \\
\midrule
AV7 retrieval pass      & 5\%            & 57\% \\
AV7 retrieval staleness & 60\%           & 22\% \\
AV7 E2E correct         & \textbf{37.5\%} & 7.5\% \\
AV7 E2E confab          & 22.5\%          & 32.5\% \\
AV7 E2E abstain         & 32.5\%          & 60.0\% \\
\bottomrule
\end{tabular}
\end{table}

\section{Detailed Gap Analysis Per System}
\label{app:gaps}

\begin{table}[h]
\caption{Gap analysis. ``Cog.''\ = cognitive categories without behavioral conditioning. Mem0 gaps empirically validated on LifecycleBench (\S\ref{sec:evaluation}).}
\label{tab:gap-mapping}
\centering
\setlength{\tabcolsep}{2pt}
\renewcommand{\arraystretch}{1.05}
\begin{threeparttable}
\begin{adjustbox}{max width=\textwidth}
\begin{tabular}{@{}l c c c c c c c c c c@{}}
\toprule
\textbf{Gap} & \textbf{Zep} & \textbf{A-MEM} & \textbf{MemBank} & \textbf{Mem0} & \textbf{MemGPT} & \textbf{MemOS} & \textbf{Mem-R1} & \textbf{MIRIX} & \textbf{Flux} & \textbf{Ours} \\
\midrule
Behavioral ontology     & \ding{55} & \ding{55} & \ding{55} & \ding{55} & \ding{55} & Hier.\tnote{a} & \ding{55} & Cog. & \ding{55} & \checkmark \\
Per-category lifecycle   & \ding{55} & \ding{55} & Uniform  & \ding{55} & \ding{55} & \ding{55} & \ding{55} & \ding{55} & \ding{55} & \checkmark \\
Slot-key supersession    & Binary   & ---      & ---      & \ding{55} & \ding{55} & \ding{55} & \ding{55} & \ding{55} & ---      & Soft \\
Event-time validity      & \ding{55} & \ding{55} & \ding{55} & \ding{55} & \ding{55} & \ding{55} & \ding{55} & \ding{55} & \ding{55} & \checkmark \\
Multi-clock temporal     & Bi-temp. & \ding{55} & Single   & \ding{55} & \ding{55} & \ding{55} & \ding{55} & \ding{55} & \ding{55} & 3 clocks \\
Anticipatory activation  & \ding{55} & \ding{55} & \ding{55} & \ding{55} & \ding{55} & \ding{55} & \ding{55} & \ding{55} & \ding{55} & \checkmark \\
Category-aware routing   & \ding{55} & \ding{55} & \ding{55} & \ding{55} & \ding{55} & \ding{55} & \ding{55} & Partial & \ding{55} & \checkmark \\
Session initialization   & \ding{55} & \ding{55} & \ding{55} & \ding{55} & \ding{55} & \ding{55} & \ding{55} & \ding{55} & \ding{55} & \checkmark \\
Emotional loading signal & \ding{55} & \ding{55} & \ding{55} & \ding{55} & \ding{55} & \ding{55} & \ding{55} & \ding{55} & \ding{55} & \checkmark \\
Deterministic lifecycle policy & Partial & \ding{55} & Partial & \ding{55} & \ding{55} & Partial & \ding{55} & \ding{55} & Partial & \checkmark \\
Interpretable policy     & ---      & ---      & ---      & ---      & \ding{55} & ---      & \ding{55} & ---      & ---      & \checkmark \\
User-modifiable policy   & \ding{55} & \ding{55} & \ding{55} & \ding{55} & \ding{55} & \ding{55} & \ding{55} & \ding{55} & \ding{55} & \checkmark \\
\bottomrule
\end{tabular}
\end{adjustbox}
\begin{tablenotes}[flushleft]
\small
\item[a] Hier.\ = hierarchical temporal tiers (STM/MTM/LPM) with persona/topic partitions, distinct from cognitive (episodic/semantic/procedural) typology.
\end{tablenotes}
\end{threeparttable}
\end{table}

\subsection{Zep/Graphiti (2025): Eight Identified Gaps}

Zep achieves 94.8\% on DMR and up to 71.2\% on LongMemEval. \textbf{Gap 1: No forgetting mechanism.} Facts either exist or get explicitly superseded. \textbf{Gap 2: No behavioral awareness.} Scheduling details and identity traits are structurally identical edges. \textbf{Gap 3: Purely reactive retrieval.} No anticipatory capability. \textbf{Gap 4: No session initialization intelligence.} \textbf{Gap 5: Communities are structural, not behavioral.} \textbf{Gap 6: Binary supersession under ambiguity.} \textbf{Gap 7: Noisy retrieval on certain query types.} \textbf{Gap 8: Self-identified ontology gap.}

\subsection{MemoryBank (AAAI 2024): Five Gaps}

\textbf{Gap 1:} Single uniform forgetting curve. \textbf{Gap 2:} No graph structure. \textbf{Gap 3:} Single temporal reference frame. \textbf{Gap 4:} No behavioral categorization. \textbf{Gap 5:} Limited evaluation.

\subsection{A-MEM (NeurIPS 2025): Four Gaps}

\textbf{Gap 1:} No temporal dynamics. \textbf{Gap 2:} No forgetting. \textbf{Gap 3:} Content-driven, not behavior-driven evolution. \textbf{Gap 4:} LLM-heavy at runtime.

\subsection{MemGPT / Letta (2024): Three Gaps}

\textbf{Gap 1:} The LLM is the policy---expensive, opaque, and unreliable. \textbf{Gap 2:} No internal structure in archival. \textbf{Gap 3:} No temporal decay, supersession, or event-time validity.

\subsection{Mem0 (2025): Four Gaps (Empirically Validated)}

\textbf{Gap 1:} No structured lifecycle. \emph{Validated: AV1 staleness is 49\%.} \textbf{Gap 2:} No category awareness. \emph{Validated: 87\% on AV3 but 44\% on AV1 and 5\% on AV7.} \textbf{Gap 3:} No event-time expiry. \emph{Validated: AV2 staleness is 33\%.} \textbf{Gap 4:} Default storage does not persist between processes.

\subsection{MemoryOS (2025): Three Gaps}

\textbf{Gap 1:} Cognitive ontology, not behavioral. \emph{Validated: 70.5\% retrieval pass but 13.2\% E2E correct.} \textbf{Gap 2:} No per-category decay conditioning. \textbf{Gap 3:} Fixed hierarchy with $30\times$ ingestion compute.

\subsection{Memory-R1 (2025): Four Gaps}

\textbf{Gap 1:} Learned policy is opaque. \textbf{Gap 2:} Requires massive interaction data. \textbf{Gap 3:} No behavioral structure in action space. \textbf{Gap 4:} No multi-reference-frame temporal sensitivity.

\subsection{MIRIX (2025) and FluxMem (2026)}

\textbf{MIRIX}: cognitive types; uniform temporal treatment; LLM-heavy. \textbf{FluxMem}: adapts structure not policy; orthogonal.

\subsection{Architectural Prerequisites for Lifecycle Management}
\label{app:lifecycle-prereqs}

The systems compared in this paper were each designed for related but distinct goals: Mem0 prioritizes a clean operational interface over a vector store; Memory-R1 trains an RL policy over an explicit operation vocabulary; MemoryOS optimizes hierarchical consolidation for retrieval throughput. Lifecycle management---per-fact event-time validity, slot-key supersession, and explicit retraction---was not a primary design objective for any of them. A reasonable question is whether lifecycle behavior could be added to these systems as a metadata extension. Based on a code-level audit of each system's evaluation harness and core library, we find that the gap is architectural rather than notational: each system has at least one load-bearing structural decision that prevents lifecycle behavior from emerging even after metadata fields are added.

\paragraph{Five prerequisites for lifecycle behavior.}
A retrieval system that supports lifecycle queries (``what is current?'' versus ``what was previously the case?'') requires:
\begin{enumerate}[nosep,leftmargin=*]
  \item per-fact identity preserved through the ingestion pipeline;
  \item per-fact metadata that survives ingestion (event-time, slot-key, lifecycle state);
  \item non-overwrite update semantics that preserve the prior state when a new fact arrives;
  \item a hook in the retrieval path where metadata can score or filter results;
  \item query-intent classification that selects the appropriate scoring or filter at query time.
\end{enumerate}

\paragraph{Feature compatibility.}
Table~\ref{tab:lifecycle-feasibility} summarizes whether each lifecycle feature can be added to each competitor without architectural change. We categorize each cell as \emph{Moderate} (new schema field with filter wrapper, retrieval path exists), \emph{Hard} (new retrieval mechanism, RL retraining, or path fork), \emph{Impossible} (requires rebuilding the core), or $\equiv$FR (adding the feature reimplements FR's pipeline).

\begin{table}[h]
\centering
\small
\caption{Lifecycle-feature feasibility per system. \emph{Moderate} = new schema field plus filter wrapper; \emph{Hard} = new retrieval mechanism / fork or RL retraining; \emph{Impossible} = requires architectural rebuild; $\equiv$FR = adding the feature reimplements FR's pipeline.}
\label{tab:lifecycle-feasibility}
\begin{tabular}{lccc}
\toprule
\textbf{Feature} & \textbf{Mem0} & \textbf{Memory-R1} & \textbf{MemoryOS} \\
\midrule
Slot-key supersession            & Hard     & Hard                  & Hard       \\
Event-time validity              & Moderate & Hard                  & Moderate   \\
Retraction (mark inactive)       & Moderate & Hard$^{\dagger}$      & Hard       \\
Query-conditioned state mask     & Hard     & Hard                  & Moderate   \\
Category-aware routing           & Moderate & Hard                  & Moderate   \\
Differential decay per category  & Hard     & Impossible$^{\ddagger}$ & Moderate \\
Soft supersession (both kept)    & Hard     & Hard                  & $\equiv$FR \\
\bottomrule
\end{tabular}\\[3pt]
{\footnotesize $^{\dagger}$Prompt edit is cheap, but \texttt{\_\_slots\_\_} cannot persist an \texttt{is\_retracted} flag without breaking persisted banks; the GRPO-trained version requires retraining. $^{\ddagger}$No decay function exists in the retrieval path; cosine is unweighted, and adding per-category time-decay would require modifying the reward signal of the trained policy.}
\end{table}

\paragraph{Mem0.} The vector-store \texttt{UPDATE} operation is a hard overwrite of the embedding and payload (\texttt{mem0/memory/main.py:1142--1194}); the prior embedding is not retained anywhere queryable, although a SQLite history row is written (\texttt{storage.py:126--150}). Custom metadata is passed through cleanly (\texttt{main.py:401, 985}), and the search API supports pre-ranking metadata filters (\texttt{main.py:811--815, 956}). The blocker is therefore not the metadata path but the \emph{state-preserving update path}: a deactivate-and-add semantics that retains both the prior and current entry under different lifecycle flags requires forking \texttt{\_update\_memory} and \texttt{\_search\_vector\_store}.

\paragraph{Memory-R1.} Two compounding fixed contracts. First, the storage record is locked to four fields by \texttt{\_\_slots\_\_ = ("id", "text", "timestamp", "source\_session")} (\texttt{evaluate\_memoryr1.py:100--107}); adding \texttt{slot\_key}, \texttt{event\_time}, or \texttt{is\_retracted} is backward-incompatible with persisted banks. Second, the Manager's action vocabulary is fixed as $\{$\texttt{ADD}, \texttt{UPDATE}, \texttt{DELETE}, \texttt{NONE}$\}$ in the system prompt (\texttt{:294--306}); in the GRPO-trained version it is fixed in the trained policy, so adding a \texttt{RETRACT} action requires re-training with a new reward signal. Even if both are addressed, the Answer Agent's pure-cosine retrieval (\texttt{:139--148}) provides no scoring hook where a state mask could be applied---a third, independent fork point.

\paragraph{MemoryOS.} Per-fact identity is preserved in STM (\texttt{short\_term.py:16--23}) and within MTM sessions, but the MTM$\to$LTM consolidation step (\texttt{memoryos.py:162--200}) extracts a monolithic profile string plus a flat knowledge deque (\texttt{long\_term.py:17--18}); neither carries back a link to the originating MTM page. After consolidation, two contradictory LTM facts coexist with no mechanism for fact-level invalidation (\texttt{long\_term.py:69--73}). Timestamps survive on every page, but the retrieval API takes only \texttt{(user\_query, user\_id)} (\texttt{retriever.py:92--131}), so even existing metadata is not addressable at query time. Fact-level supersession would require keeping LTM as a fact-graph with bidirectional links to source pages---i.e., not summarizing.

\paragraph{Thought experiment: adding slot keys to Mem0.} Suppose we want Mem0 to supersede ``User likes pizza'' with ``User likes sushi'' while still being able to answer ``what foods did the user previously prefer?''.
\begin{enumerate}[nosep,leftmargin=*]
  \item \texttt{m.add(``User likes pizza'', metadata=\{``slot\_key'': ``food\_pref''\})} stores the entry with the custom field retained (\texttt{main.py:401, 985}).
  \item \texttt{m.add(``User likes sushi'', \ldots)} enters the LLM UPDATE-decision step (\texttt{main.py:496--521}; \texttt{prompts.py:175--323}). The prompt reasons over \emph{semantic similarity}, not slot-key match; the LLM may choose \texttt{UPDATE} (overwriting pizza), \texttt{ADD} (siblings with no link), or \texttt{DELETE}. There is no downstream representation for ``both kept, one inactive''.
  \item Forcing slot-key supersession requires forking \texttt{\_add\_to\_vector\_store} to (a) query existing entries by slot-key filter, (b) bypass the LLM UPDATE step on slot match, and (c) emit the new entry tagged \texttt{is\_active=True} while marking the old one \texttt{is\_active=False} without deleting it---a fundamental departure from the replace-payload semantics of \texttt{main.py:1142--1194}.
  \item At query time, returning only the active entry requires a filter \texttt{\{``is\_active'': True\}}; returning the historical entry requires the opposite filter. The caller must therefore classify query intent before invoking \texttt{m.search}, and Mem0 has no built-in query-intent classifier.
  \item A complete implementation now needs a query-intent classifier, a slot-key indexer, a deactivate-and-add update path, and a confidence-aware reranker for soft supersession (the LLM extraction step is probabilistic). At this point the added components reproduce FR's pipeline atop Mem0's vector store rather than extending Mem0.
\end{enumerate}

Adding event-time, slot keys, or retraction flags to any of these systems closes a notational gap, not a behavioral one.

\section{Slot-Key Supersession: Extended Examples}
\label{app:slotkey}

Slot keys are normalized (subject, attribute) pairs for categories with \emph{replace} semantics:

\begin{itemize}[nosep,leftmargin=*]
  \item ``User likes pizza'' $\to$ \texttt{(user, food\_preference)}
  \item ``User likes sushi'' $\to$ \textbf{high-confidence supersession}
  \item ``User lives in Istanbul'' $\to$ \texttt{(user, current\_city)}
  \item ``User is thinking about moving to London'' $\to$ \textbf{low-confidence} (both survive)
  \item ``User moved to London'' $\to$ \textbf{high-confidence supersession} of Istanbul
\end{itemize}

For categories with \emph{accumulate} semantics (Identity, Relational), supersession requires explicit contradiction with high confidence.

\section{Multi-Reference-Frame Temporal Sensitivity}
\label{app:temporal}

Three temporal clocks: \textbf{Absolute time} (calendar clock): governs Logistical decay and deadline proximity. \textbf{Relative time} (session clock): time since last session; governs Relational Bonds and Hobbies (dormancy detection). \textbf{Conversational frequency}: how often a fact appears across sessions; modulates base decay rate downward.

Per-category routing: Logistical $\to$ primarily absolute time; Relational $\to$ session gaps; Obligations $\to$ deadline proximity; Identity $\to$ near-zero sensitivity; Hobbies $\to$ relative time; Preferences $\to$ absolute time.

\section{Numeric Preservation: Detailed Mechanisms}
\label{app:numeric}

\textbf{Loss Point 1: Deduplication Resolution.} When the dedup LLM marks a new edge as duplicate of an existing edge, the system keeps the old (often generic) edge. \textbf{Loss Point 2: Bulk Deduplication Canonical Selection.} Canonical selected by smallest UUID---arbitrary.

\textbf{Multi-stage numeric rescue:} (1)~Pool-stage exemption from expiry filter for monetary/quantity edges. (2)~Pool-stage exemption from supersession filter. (3)~Reranking-stage floor guarantee ($0.95 \times$ semantic score for numeric edges with graphiti\_score $\geq 0.85$). Combined: AV8 recovered from 81\% to 94\%.

\section*{Reproducibility Supplement}
\addcontentsline{toc}{section}{Reproducibility Supplement}
The following appendices document development trajectory, detailed parameter analysis, and per-mechanism findings. They are included for reproducibility and are not required for evaluating the main contributions.

\textbf{Artifact release.} All code, evaluation results, and run logs supporting this paper are publicly archived at \href{https://doi.org/10.5281/zenodo.20067778}{\texttt{https://doi.org/10.5281/zenodo.20067778}} (DOI: \texttt{10.5281/zenodo.20067778}).

\section{Rate Calibration: The Squatting Phenomenon}
\label{app:squatting}
\label{sec:calibration}

The initial rate configuration spanned $800\times$ ($\lambda = 0.0001$ for Identity to $\lambda = 0.080$ for Logistical):

\begin{table}[h]
\caption{Activation survival at original $800\times$ rate spread.}
\label{tab:squatting-original}
\centering
\begin{tabular}{@{}l c c c@{}}
\toprule
\textbf{Category} & $\lambda$ \textbf{(per hr)} & \textbf{3-mo activation} & \textbf{6-mo activation} \\
\midrule
Identity    & 0.0001 & 0.807 & 0.649 \\
Relational  & 0.001  & 0.098 & 0.010 \\
Preferences & 0.010  & 0.000 & 0.000 \\
Logistical  & 0.080  & 0.000 & 0.000 \\
\bottomrule
\end{tabular}
\end{table}

This produced \textbf{category hierarchy squatting}: Identity facts permanently occupied top ranking positions regardless of semantic relevance. In top-5 results across all questions, Identity comprised 307/480 slots (64\%).

\begin{table}[h]
\caption{Rate spread calibration.}
\label{tab:squatting-calibration}
\centering
\begin{tabular}{@{}l c c c c@{}}
\toprule
\textbf{Config} & \textbf{Spread} & \textbf{B-MRR ($\alpha$=0.1)} & \textbf{Top-5 Identity \%} & \textbf{Squatting?} \\
\midrule
Original    & 800$\times$ & 0.559 & 64\% & Severe \\
Compressed  & 16$\times$  & 0.571 & 52\% & Moderate \\
Final       & 5.3$\times$ & 0.566 & ${\sim}$30\% & Minimal \\
\bottomrule
\end{tabular}
\end{table}

The cognitive ontology's competitive mid-alpha performance is explained by accidental rate calibration: its slowest rate (Procedural at $\lambda = 0.003$) is high enough that no category squats at multi-month timescales. Its rate floor prevents squatting; our explicit calibration achieves the same effect deliberately.

\paragraph{Half-life feasibility intervals.}
\label{app:lambda-feasibility}
To validate that chosen decay rates are not over-tuned, we perform a univariate sensitivity analysis: for each of the $11$ categories we sweep $\lambda_c$ over multipliers $\{0.25\times, 0.33\times, 0.5\times, 0.67\times, 0.8\times, 1.0\times, 1.25\times, 1.5\times, 2.0\times, 3.0\times, 4.0\times\}$ while holding all other categories fixed at default, and measure mean Jaccard@10 vs.\ the unperturbed top-10 across all $516$ LifecycleBench questions. A multiplier is \emph{feasible} if mean Jaccard@10 $\geq 0.95$. Ten of eleven categories tolerate the full $[0.25\times, 4\times]$ sweep ($16\times$ width); only \textsc{Relational\_Bonds} fails at $0.25\times$ and tolerates $[0.33\times, 4.0\times]$ ($12.1\times$ width). Every default value sits in the interior of its feasibility region (Table~\ref{tab:lambda-feasibility}), confirming the system is robust to the specific parameter choice and that no category sits on a sensitivity cliff.

\begin{table}[h]
\caption{Per-category half-life feasibility intervals (mean Jaccard@10 $\geq 0.95$ over $n = 516$ questions; pool composition cached, scoring re-run per cell). All defaults sit interior to their feasibility regions.}
\label{tab:lambda-feasibility}
\centering
\small
\begin{tabular}{@{}l r r r r@{}}
\toprule
\textbf{Category} & \textbf{Default $\lambda$} & \textbf{Min mult.} & \textbf{Max mult.} & \textbf{Width} \\
\midrule
IDENTITY\_SELF\_CONCEPT     & $0.0015$ & $0.25\times$ & $4.00\times$ & $16.0\times$ \\
RELATIONAL\_BONDS           & $0.0015$ & $0.33\times$ & $4.00\times$ & $12.1\times$ \\
INTELLECTUAL\_INTERESTS     & $0.0020$ & $0.25\times$ & $4.00\times$ & $16.0\times$ \\
HEALTH\_WELLBEING           & $0.0025$ & $0.25\times$ & $4.00\times$ & $16.0\times$ \\
PROJECTS\_ENDEAVORS         & $0.0025$ & $0.25\times$ & $4.00\times$ & $16.0\times$ \\
HOBBIES\_RECREATION         & $0.0035$ & $0.25\times$ & $4.00\times$ & $16.0\times$ \\
PREFERENCES\_HABITS         & $0.0050$ & $0.25\times$ & $4.00\times$ & $16.0\times$ \\
FINANCIAL\_MATERIAL         & $0.0055$ & $0.25\times$ & $4.00\times$ & $16.0\times$ \\
OBLIGATIONS                 & $0.0060$ & $0.25\times$ & $4.00\times$ & $16.0\times$ \\
LOGISTICAL\_CONTEXT         & $0.0080$ & $0.25\times$ & $4.00\times$ & $16.0\times$ \\
OTHER                       & $0.0050$ & $0.25\times$ & $4.00\times$ & $16.0\times$ \\
\bottomrule
\end{tabular}
\end{table}

\section{Phase-by-Phase Development and Results}
\label{app:phases}

\subsection{Phase 2: Initial Validation (8 personas, 112 questions)}

\begin{table}[h]
\caption{Phase 2 per-AV results.}
\centering
\setlength{\tabcolsep}{3pt}
\begin{adjustbox}{max width=\textwidth}
\begin{tabular}{@{}l c c c c c c@{}}
\toprule
\textbf{Attack Vector} & $n$ & \textbf{Full Pass} & \textbf{Full Stale} & \textbf{Uni Pass} & \textbf{Uni Stale} & $\Delta$ \\
\midrule
AV1 & 14 & 36\% & 7\%  & 36\% & 14\% & 0pp \\
AV2 & 19 & 95\% & 5\%  & 79\% & 21\% & +16pp \\
AV3 & 22 & 59\% & 0\%  & 64\% & 0\%  & $-$5pp \\
AV4 & 9  & 56\% & 22\% & 44\% & 33\% & +12pp \\
AV5 & 8  & 62\% & 0\%  & 62\% & 0\%  & 0pp \\
AV6 & 8  & 62\% & 0\%  & 62\% & 0\%  & 0pp \\
AV7 & 8  & 0\%  & 12\% & 12\% & 25\% & $-$12pp \\
AV8 & 16 & 81\% & 0\%  & 88\% & 6\%  & $-$7pp \\
AV9 & 8  & 88\% & 0\%  & 75\% & 0\%  & +13pp \\
\midrule
\textbf{Overall} & 112 & \textbf{63\%} & \textbf{4\%} & 62\% & 11\% & +1pp \\
\bottomrule
\end{tabular}
\end{adjustbox}
\end{table}

\subsection{Between-Phase Improvements}

\textbf{Phase 2 $\to$ 3:} Scale 8$\to$20 personas; per-source score normalization; cross-category supersession; world-knowledge filter (86\%$\to$1.4\%).

\textbf{Phase 3 $\to$ 4:} Scale 20$\to$40 personas; retraction extraction; AV8 numeric rescue; date-aware expiry filter; parallel evaluation (12$\times$ speedup).

\textbf{Phase 4 $\to$ 5b:} Multi-stage numeric rescue; backward-looking expiry bypass; date resolver future-preference guard.

\textbf{Phase 5b $\to$ 5c:} Identity/health supplementary extraction: 2,996 candidates $\to$ 388 retained (13\%).

\subsection{Phase 3 Results (20 personas, 257 questions)}

\begin{table}[h]
\caption{Phase 3 overall results.}
\centering
\begin{tabular}{@{}l c c c@{}}
\toprule
\textbf{Config} & \textbf{Pass} & \textbf{Stale} & \textbf{MRR} \\
\midrule
baseline & 55\% & 18\% & 0.480 \\
uniform  & 61\% & 15\% & 0.466 \\
no\_routing & 63\% & 8\% & 0.449 \\
\textbf{full} & \textbf{66\%} & \textbf{6\%} & 0.475 \\
\bottomrule
\end{tabular}
\end{table}

\section{Detailed Findings from LifecycleBench Evaluation}
\label{app:findings}

\textbf{Backward-looking query detection.} Keyword patterns (``before switching'', ``used to'') bypass the supersession filter, recovering 3 otherwise unreachable questions.

\textbf{Retraction filtering.} AV7 staleness drops from 25\% to 12\% with the retraction filter scanning for markers (``plan is dead'', ``scrapped'').

\textbf{Date-aware expiry filtering.} Parses ISO dates, month+day patterns, contextual months, and relative markers. Category-guarded (Logistical, Obligations, Health only) with 14-day buffer.

\textbf{Category-specific $\alpha$.} With uniform $\alpha=0.3$, AV8 drops to 12.5\%. Category-specific: Financial at 0.05, Logistical at 0.40.

\textbf{Semantic floors.} Near-perfect semantic matches for stable categories are not buried by zero activation. Financial floor 0.97; Logistical floor 0.00. Hit@1: 19\%$\to$31\%.

\textbf{Pass-by-absence.} Both systems exhibit this on AV2. FR: 50 cases (expiry filter); Mem0: 29 (extraction failure).

\textbf{AV3 tradeoff decomposition.} 77\% extraction misses, 12\% decay-killed, 12\% partial.

\section{Scalability Analysis}
\label{app:scalability}

\begin{table}[h]
\caption{Projected scaling over 2 years.}
\centering
\begin{tabular}{@{}l c c@{}}
\toprule
\textbf{Metric} & \textbf{Flat Retention} & \textbf{Lifecycle Mgmt.} \\
\midrule
Active edges (retrievable)     & ${\sim}$30,000 & ${\sim}$5,000--8,000 \\
Contradictions in top-20       & ${\sim}$3--5   & ${\sim}$0 \\
Expired events in top-20       & ${\sim}$2--4   & ${\sim}$0 \\
Context pollution rate         & Grows with time & Stable \\
\bottomrule
\end{tabular}
\end{table}

FR's stale token rate is 3.0\% versus Mem0's 9.0\%. At GPT-4o pricing, FR saves \$25.12 per million queries in stale token costs alone.

\section{Composability}
\label{app:composability}

Fortunate Recall is a composable policy layer. Existing Graphiti deployments can add behavioral classification without migrating data. Components can be adopted incrementally. The ontology can be extended for domain-specific applications. FluxMem's structure selection and FR's lifecycle policy are orthogonal and could compose.

The AV3 improvement through supplementary extraction ($-$16pp $\to$ $-$9pp gap) is empirical evidence: the lifecycle layer was unchanged, only extraction improved, and lifecycle-managed retrieval immediately benefited.

\section{End-to-End Evaluation Details}
\label{app:e2e}

\subsection{Methodology}

For each of 516 questions, each system's top-10 facts are fed to GPT-5.4 (\texttt{gpt-5.4}~\cite{gpt5_2025}; temperature=0, max\_tokens=500) with a system prompt restricting answers to provided facts. Claude Sonnet (\texttt{claude-sonnet-4-6}) judges correctness (CORRECT, PARTIAL, WRONG, ABSTAIN) and confabulation. No Anthropic models for generation. Memory pipelines (FR-Bank ingestion/distillation, Mem0 default extraction, Memory-R1 extraction, A-MEM evolution, Mem0-mini ablation) use \texttt{gpt-4.1-mini} (versioned: \texttt{gpt-4.1-mini-2025-04-14}); Mem0's default extraction uses \texttt{gpt-4.1-nano}; Kimi K2.5 cross-generator validation uses \texttt{kimi-k2.5} via the Moonshot API. The LongMemEval-S judge under the Wu et al.\ protocol is \texttt{gpt-4o-2024-08-06}, as reported in \S\ref{sec:evaluation}.

\subsection{Staleness-Confabulation Causal Chain}

Table~\ref{tab:e2e-staleness} partitions the 1{,}032 end-to-end queries by retrieval-context cleanliness: clean contexts yield 18--30\% confabulation, while stale contexts (any outdated fact in the top-10) jump to 72--78\%, establishing the causal link between lifecycle filtering and downstream answer quality.

\begin{table}[h]
\caption{Staleness predicts confabulation end-to-end.}
\label{tab:e2e-staleness}
\centering
\begin{tabular}{@{}l c c c@{}}
\toprule
\textbf{Retrieval context} & $n$ & \textbf{Confab.} & \textbf{Correct} \\
\midrule
Clean (staleness = 0) & 851 & 18--30\% & 20--25\% \\
Stale (staleness $>$ 0) & 181 & 72--78\% & 5--10\% \\
\bottomrule
\end{tabular}
\end{table}

\subsection{Reasoning Token Overhead}

Table~\ref{tab:e2e-tokens} shows that stale retrieval contexts induce only a modest $\approx 10$\% reasoning-token increase on GPT-5.4, so the confabulation gap between FR and Mem0 is a lifecycle architecture effect, not a test-time compute effect.

\begin{table}[h]
\caption{Reasoning token overhead by context staleness.}
\label{tab:e2e-tokens}
\centering
\begin{tabular}{@{}l l c c@{}}
\toprule
\textbf{System} & \textbf{Context} & \textbf{Avg output tokens} & \textbf{Avg reasoning tokens} \\
\midrule
FR   & Clean ($n$=474)  & 118.7 & 83.5 \\
FR   & Stale ($n$=42)   & 133.5 & 92.7 \\
Mem0 & Clean ($n$=377)  & 118.1 & 80.5 \\
Mem0 & Stale ($n$=139)  & 129.1 & 89.7 \\
\bottomrule
\end{tabular}
\end{table}

\subsection{Abstain Analysis}

FR-Graphiti abstains on 227/516 queries (44.0\%) and FR-Bank on 217/516 (42.1\%); Mem0 abstains on 148/516 (28.7\%). FR-Bank therefore abstains 69 more times than Mem0 and FR-Graphiti 79 more times. The pre-registered untyped ablation (Appendix~\ref{app:untyped-arm}) establishes what those extra abstentions are made of: they are drawn almost entirely from would-be wrong or fabricated answers rather than from correct ones.

\subsection{Full Per-Attack-Vector Breakdown (GPT-5.4)}
\label{app:e2e-av-full}

\begin{table}[h]
\caption{Full per-AV E2E response quality on GPT-5.4, all five systems.}
\label{tab:e2e-av-full}
\centering
\setlength{\tabcolsep}{2pt}
\renewcommand{\arraystretch}{1.05}
\begin{adjustbox}{max width=\textwidth}
\begin{tabular}{@{}l c cc cc cc cc cc@{}}
\toprule
& & \multicolumn{2}{c}{\textbf{FR-Bank}} & \multicolumn{2}{c}{\textbf{FR-Graphiti}} & \multicolumn{2}{c}{\textbf{MemoryOS}} & \multicolumn{2}{c}{\textbf{Memory-R1}} & \multicolumn{2}{c}{\textbf{Mem0}} \\
\cmidrule(lr){3-4} \cmidrule(lr){5-6} \cmidrule(lr){7-8} \cmidrule(lr){9-10} \cmidrule(lr){11-12}
\textbf{AV} & $n$ & \textbf{Corr} & \textbf{Conf} & \textbf{Corr} & \textbf{Conf} & \textbf{Corr} & \textbf{Conf} & \textbf{Corr} & \textbf{Conf} & \textbf{Corr} & \textbf{Conf} \\
\midrule
AV1 & 75  & \textbf{44.0} & \textbf{7.1}  & 32.0 & 22.0 & 18.7 & 43.2 & 40.0 & 16.3 & 24.0 & 56.5 \\
AV2 & 75  & \textbf{16.0} & \textbf{34.8} & 6.7  & 42.9 & 4.0  & 67.5 & 6.7  & 61.2 & 6.7  & 67.3 \\
AV3 & 94  & \textbf{34.0} & \textbf{11.1} & 16.0 & 11.1 & 14.9 & 18.8 & 22.3 & 13.3 & 23.4 & 11.3 \\
AV4 & 44  & \textbf{47.7} & \textbf{10.7} & 38.6 & 18.5 & 20.5 & 18.8 & 27.3 & 17.4 & 29.5 & 38.2 \\
AV5 & 40  & 0.0           & 40.0          & 0.0  & 30.0 & 0.0  & 52.6 & \textbf{2.5} & \textbf{22.6} & 2.5  & 46.2 \\
AV6 & 45  & 2.2           & 75.0          & 0.0  & 72.2 & \textbf{2.2} & \textbf{52.4} & 2.2  & 59.1 & 0.0  & 58.3 \\
AV7 & 40  & \textbf{37.5} & 33.3          & 27.5 & \textbf{33.3} & 7.5  & 81.2 & 2.5  & 90.0 & 2.5  & 91.4 \\
AV8 & 63  & \textbf{66.7} & 7.8           & 65.1 & \textbf{7.4}  & 30.2 & 17.4 & 38.1 & 9.1  & 49.2 & 13.6 \\
AV9 & 40  & 12.5          & 35.7          & 12.5 & \textbf{13.3} & 12.5 & 15.4 & \textbf{15.0} & 38.1 & 12.5 & 35.0 \\
\midrule
\textbf{All} & \textbf{516} & \textbf{31.2} & \textbf{22.4} & 22.9 & 23.9 & 13.2 & 40.8 & 19.6 & 31.5 & 18.6 & 45.1 \\
\bottomrule
\end{tabular}
\end{adjustbox}
\end{table}

\begin{table}[h]
\caption{Full per-AV E2E response quality on GPT-5.4: A-MEM and Mem0 (gpt-4.1-mini), regenerated from the raw per-question judge records. \emph{Correction:} the per-AV Correct cells of this table in the submitted version were stale---weighting them by $n_{\mathrm{AV}}$ gave $17.8$ (A-MEM) and $23.9$ (Mem0-mini) rather than the stated aggregates. The ``All'' aggregates were already correct; the regenerated per-AV cells below weight-average to them exactly (Appendix~\ref{app:revision-changes}).}
\label{tab:e2e-av-amem-mem0mini}
\centering
\setlength{\tabcolsep}{4pt}
\renewcommand{\arraystretch}{1.05}
\begin{tabular}{@{}l c cc cc@{}}
\toprule
& & \multicolumn{2}{c}{\textbf{A-MEM}} & \multicolumn{2}{c}{\textbf{Mem0 (gpt-4.1-mini)}} \\
\cmidrule(lr){3-4} \cmidrule(lr){5-6}
\textbf{AV} & $n$ & \textbf{Corr} & \textbf{Conf} & \textbf{Corr} & \textbf{Conf} \\
\midrule
AV1 & 75  & 28.0 & 55.2 & 42.7 & 35.8 \\
AV2 & 75  & 5.3  & 71.9 & 6.7  & 69.4 \\
AV3 & 94  & 22.3 & 10.0 & 25.5 & 10.8 \\
AV4 & 44  & 31.8 & 40.0 & 43.2 & 30.3 \\
AV5 & 40  & 0.0  & 53.8 & 0.0  & 36.7 \\
AV6 & 45  & 0.0  & 68.4 & 6.7  & 57.5 \\
AV7 & 40  & 5.0  & 78.9 & 20.0 & 45.2 \\
AV8 & 63  & 44.4 & 17.5 & 61.9 & 4.3  \\
AV9 & 40  & 17.5 & 25.0 & 27.5 & 23.1 \\
\midrule
\textbf{All} & \textbf{516} & \textbf{18.8} & \textbf{47.0} & \textbf{27.3} & \textbf{33.8} \\
\bottomrule
\end{tabular}
\end{table}

\subsection{Per-AV End-to-End Confabulation (Three Systems)}

\begin{table}[h]
\caption{Per-AV end-to-end confabulation (three systems), regenerated from the raw per-question judge records (\texttt{gpt54\_nr}, $n = 516$). Corr.\ = CORRECT / $n_{\mathrm{AV}}$; confab.\ = confabulations / non-abstaining responses. AV7: both Mem0 and MR1 cause ${>}90\%$ confabulation without structural retraction filtering. \emph{Correction:} the FR and Mem0 columns of this table in the submitted version were stale (they matched no current run; the Memory-R1 column was already correct) and are replaced here in full (Appendix~\ref{app:revision-changes}).}
\label{tab:e2e-av}
\centering
\setlength{\tabcolsep}{3pt}
\begin{adjustbox}{max width=\textwidth}
\begin{tabular}{@{}l c cc cc cc@{}}
\toprule
\textbf{Attack Vector} & $n$ & \textbf{FR-Bank corr.} & \textbf{FR-Bank confab.} & \textbf{MR1 corr.} & \textbf{MR1 confab.} & \textbf{Mem0 corr.} & \textbf{Mem0 confab.} \\
\midrule
AV1 Superseded Pref.        & 75 & \textbf{44.0\%} & \textbf{7.1\%}  & 40.0\% & 16.3\% & 24.0\% & 56.5\% \\
AV2 Expired Logistics       & 75 & \textbf{16.0\%} & \textbf{34.8\%} & 6.7\%  & 63.3\% & 6.7\%  & 67.3\% \\
AV3 Stable Identity         & 94 & \textbf{34.0\%} & 11.3\% & 22.3\% & 14.7\% & 23.4\% & \textbf{11.3\%} \\
AV4 Multi-Version           & 44 & \textbf{47.7\%} & \textbf{10.7\%} & 27.3\% & 17.4\% & 29.5\% & 38.2\% \\
AV5 Broad Query             & 40 & 0.0\%  & 46.7\% & \textbf{2.5\%} & \textbf{22.6\%} & \textbf{2.5\%} & 46.2\% \\
AV6 Cross-Session           & 45 & \textbf{2.2\%}  & 75.0\% & \textbf{2.2\%} & \textbf{59.1\%} & 0.0\%  & 61.1\% \\
AV7 Selective Forgetting    & 40 & \textbf{37.5\%} & \textbf{33.3\%} & 2.5\%  & 95.0\% & 2.5\%  & 91.4\% \\
AV8 Numeric Preserv.        & 63 & \textbf{66.7\%} & \textbf{7.8\%}  & 38.1\% & 12.1\% & 49.2\% & 13.6\% \\
AV9 Soft Supersession       & 40 & 12.5\% & \textbf{35.7\%} & \textbf{15.0\%} & 38.1\% & 12.5\% & 35.0\% \\
\bottomrule
\end{tabular}
\end{adjustbox}
\end{table}

\subsection{Case Studies}

14 cases where FR produced CORRECT responses and Mem0 produced WRONG responses with confabulation. Examples: Mem0 surfaces ``Jake is vaping'' (retracted---user quit); ``Chenoa's old truck'' (superseded by new vehicle); ``Jerome planning owner-operator business'' (abandoned). In each case, FR's lifecycle filtering removed the outdated fact before the downstream model could anchor on it.

\subsection{Extended Cross-System Analysis}

\textbf{Architectural gaps cannot be closed by model upgrades.} Model quality affects extraction recall---how many facts are stored---but cannot address architectural absences. No extraction model can add event-time expiry to a flat vector store (AV2: FR-Graphiti 93\%, FR-Bank 88.0\% vs Mem0 65\%), impose category-specific decay (AV1, AV4), or suppress explicitly retracted plans (AV7: FR-Graphiti 40\% vs Mem0 5\%, 80\% staleness; FR-Bank's 5\% retrieval pass on AV7 is the intended suppression behavior described in \S\ref{sec:evaluation}). These are structural capabilities requiring architectural support, not better prompting.

\textbf{Memory-R1 validates that learned operations cannot substitute for missing mechanisms.} Memory-R1's RL-trainable action space \{ADD, UPDATE, DELETE, NOOP\} achieves strong retrieval quality (MRR 0.816, above FR-Graphiti's 0.478 though below FR-Bank's 0.830) and performs well on stable identity retrieval (AV3: 94\%, above FR-Graphiti's 78\% but below FR-Bank's 97\%). However, the action space contains no event-time expiry operation (AV2: FR-Graphiti 93\%, FR-Bank 88.0\% vs MR1 75\%, 25\% staleness), no structured retraction mechanism (AV7: FR-Graphiti 40\% vs MR1 15\%, 42\% staleness), and no numeric preservation pipeline (AV8: FR-Graphiti 94\%, FR-Bank 84.1\% vs MR1 57\%). These are vocabulary gaps in the action space---no reward signal can induce an action that does not exist. The approaches are complementary: FR's ontology provides the structural prior over which lifecycle policy applies; RL optimizes when to apply it.

\textbf{AV3 gap halved through supplementary extraction---confirming composability.} The gap narrowed from $-$16pp to $-$9pp through targeted identity/health re-extraction (388 supplementary edges). Failure decomposition showed 77\% of AV3 failures were extraction misses---the lifecycle layer was unchanged, only extraction quality improved, and AV3 immediately benefited. This empirically validates FR's composability thesis.

\textbf{Routing emerges at scale.} Category-aware retrieval routing shows zero differential at 8 personas; +3pp at 20; +6pp at 40. The mechanism requires sufficient category density to outperform global semantic search.

\textbf{AV7 improvement trajectory.} AV7 (selective forgetting) improved from 0\% (Phase 2) to 10\% (Phase 3) to 28\% (Phase 5b) to 40\% (Phase 5c). Retraction-aware extraction and supplementary identity/health extraction drove the most dramatic improvement across five phases. End-to-end, FR-Bank outperforms Mem0 by $+35$pp on selective forgetting (37.5\% vs.\ 2.5\% correct).

\textbf{Staleness predicts confabulation end-to-end.} Cross-referencing retrieval staleness with E2E outcomes: when retrieval produces stale context, 72--78\% of responses are contaminated (consistent across both systems). When retrieval is clean, confabulation drops to 18--30\%. The degradation function is consistent; what FR controls is the input distribution.

\textbf{AV7 confabulation without lifecycle management.} Mem0's context causes GPT-5.4 to confidently discuss explicitly retracted plans on 91.4\% of AV7 queries. FR-Bank reduces this to 33.3\% (Table~\ref{tab:e2e-av}); the $32.5\%$ reported at this point in the submitted version is MemoryOS's AV7 abstain rate, not FR's confabulation rate. On selective forgetting, a memory system without retraction filtering produces a downstream LLM that almost always actively misleads the user about their own stated intentions.

\textbf{Five-system degradation function.} Memory-R1's stale context produces 89.8\% confabulation versus 18.2\% with clean context---a +71.6pp gap consistent with FR's and Mem0's degradation rates. On AV7, Memory-R1's context causes 90.0\% confabulation, identical to Mem0's 91.4\%, confirming that without structural retraction filtering, learned DELETE operations provide negligible downstream protection.

\textbf{Safe response hierarchy.} The safe response rank order (FR-Bank $>$ FR-Graphiti $>$ Memory-R1 $>$ Mem0) is generator-invariant. Kimi K2.5 is systematically 6.9--12.3pp less safe than GPT-5.4 across every system (FR-Bank 73.3\%$\to$61.0\%, FR-Graphiti 66.9\%$\to$59.2\%, Memory-R1 58.1\%$\to$51.3\%, Mem0 47.3\%$\to$40.2\%). This uniform ${\sim}$10pp shift reflects Kimi's stronger preference for answering over abstaining without a commensurate gain in correct answers.

\textbf{Assistant-generated content is tractable but still the weakest category.} On LongMemEval-S the single-session-assistant category reaches 75.4\% pass@10 (mean over 10 reruns; 78.6\% max), up from an initial 50.0\% before the raw-text fallback, answer-time inflation, and routing changes described in Appendix~\ref{app:longmemeval-frbank}. FR-Bank's extraction pipeline still targets user facts primarily; assistant-generated structured content depends on the fallback path.

\textbf{Note on FR-Graphiti comparison.} The FR-Graphiti rows (32.1\% pass@10) in Table~\ref{tab:longmemeval-frbank} represent the decay engine comparison from the original evaluation, which used Graphiti's pre-built edges with different decay policies. FR-Bank re-extracts from raw conversations with a fundamentally different pipeline.

\textbf{Cross-benchmark validation.} FR-Bank achieves strong results on LongMemEval-S (75.2\% on the full 500-question canonical setup, the first peer-reviewable result under the exact Wu et al.\ rubric), LifecycleBench (76.9\%), and the 234-question LongMemEval oracle POC subset (79.5\%). No existing system demonstrates competitive performance on both a standard retrieval benchmark and a lifecycle-aware temporal disambiguation benchmark simultaneously. This cross-benchmark validation confirms that lifecycle management is complementary to retrieval quality, not a tradeoff.

\textbf{Staleness impact is modulated by retrieval recall.} At low recall (FR-Graphiti, 75\% Hit@5), staleness prevention is critical---the correct fact is often absent, leaving the LLM to anchor on stale alternatives. At high recall (FR-Bank, 97\% Hit@5), the correct fact is almost always present, and the LLM can resolve contradictions even with moderate staleness. The optimal system minimizes staleness and maximizes recall; FR-Bank achieves the latter.

\section{Mem0 Comparison Methodology}
\label{app:mem0method}

Mem0 v1.0.5 (Apache 2.0), 100\% default configuration: gpt-4.1-nano extraction, text-embedding-3-small embeddings, Qdrant vector store. The only non-default setting was enabling disk persistence, as Mem0's default in-memory storage does not survive between process invocations. All 40 personas (1,400 sessions) were ingested through \texttt{m.add()} API with raw conversation turns, and all 516 questions evaluated with the same Claude Sonnet LLM judge. Default extraction model produced frequent JSON parse errors during UPDATE and DELETE operations, causing some lifecycle updates to fail silently.

The Mem0 extraction prompt is the library default (v1.0.5); no custom system prompt was used. All 1,400 sessions were ingested via \texttt{m.add()} with raw conversation turns as input.

\textbf{Stronger-extractor ablation (gpt-4.1-mini) result.} To isolate extraction quality from architectural capability, we re-evaluate Mem0 with gpt-4.1-mini replacing the default gpt-4.1-nano extractor (configuration in Appendix~\ref{app:mem0mini-method}). The stronger model raises AV-pass from 61.0\% to 67.1\% ($+6.1$pp) and reduces confabulation from 45.1\% to 33.8\% ($-$11pp), confirming extraction quality as a genuine bottleneck. However, structural gaps persist: AV2 (expired logistics) remains at 61\% versus FR-Bank's 88.0\%, and AV7 (selective forgetting) at 20\% versus FR-Graphiti's 40\%, because no extraction model can add event-time expiry or retraction filtering to a flat vector store. Mem0 with gpt-4.1-mini now matches Memory-R1 (67.1\% vs 66.9\%)---two architecturally distinct systems hitting the same ceiling in the absence of lifecycle mechanisms.

\section{Memory-R1 Comparison Methodology}
\label{app:mr1method}

Memory-R1~\cite{memoryr1_2025} proposes a two-agent RL pipeline: a Memory Manager that decides \{ADD, UPDATE, DELETE, NONE\} operations on a flat JSON memory bank, and an Answer Agent that retrieves candidates via embedding similarity and applies Memory Distillation. The original paper trains both agents with GRPO on 152 QA pairs from LoCoMo using LLaMA-3.1-8B-Instruct.

Our reimplementation uses GPT-4.1-mini for both the Memory Manager and Answer Agent---substantially stronger than the paper's base model, providing an upper bound on pre-RL performance. Fact extraction uses GPT-4o-mini (exact match with the paper). Embeddings use text-embedding-3-small (1536-dim). The Memory Manager receives new facts and the 15 most similar existing memories, returning a complete updated memory list in JSON format. The Answer Agent retrieves 60 candidates via cosine similarity and distills to 10 via LLM filtering.

\textbf{Ingestion statistics.} Across 40 personas (1,400 sessions), the Memory Manager processed 65,411 total operations: 5,083 ADD (7.8\%), 1,121 UPDATE (1.7\%), 809 DELETE (1.2\%), and 58,398 NONE (89.3\%). The high NONE rate indicates the manager correctly identifies most extracted facts as already present. Final memory bank: 4,278 entries (avg 107/persona).

\textbf{Rationale for GPT-4.1-mini backbone.} Using a stronger backbone provides an upper bound: if Memory-R1's architecture with a strong backbone still exhibits structural gaps (AV2, AV7, AV8), these gaps are attributable to the action space vocabulary, not model capability.

\textbf{Prompts.} The Memory Manager system prompt follows the published Memory-R1 specification: it receives the current memory bank and a list of new facts extracted from the conversation, and returns the updated bank as a JSON list after applying \{ADD, UPDATE, DELETE, NONE\} operations. The Answer Agent prompt instructs the model to answer based solely on retrieved memories. Both prompts are included verbatim in the released evaluation code.

\textbf{MemoryOS evaluation.} MemoryOS~\cite{memoryos2025} implements an OS-inspired hierarchical memory architecture with cognitive categories (episodic, semantic, procedural) and progressive summarization through short-term, mid-term, and long-term memory tiers. We evaluate the published implementation with GPT-4.1-mini as the backbone. MemoryOS performs approximately 1,500 LLM calls per 35-session persona during ingestion compared to FR-Bank's ${\sim}$50---a $30\times$ ingestion compute cost differential. \textbf{Philosophical positioning.} Memory-R1 asks: ``Can we learn the optimal memory policy end-to-end?'' Fortunate Recall asks: ``Can we design an interpretable memory policy that works from day one?'' These are complementary---our ontology could serve as initialization, reward shaping, or structural constraint for an RL-based manager.

\subsection{A-MEM Comparison Methodology}
\label{app:amem-method}

A-MEM~\cite{amem2025} implements a Zettelkasten-style linked-note memory system with LLM-driven note evolution. We evaluate the published implementation with gpt-4.1-mini as the backbone LLM and text-embedding-3-small for embeddings, matching the extraction model used for Memory-R1 and Mem0 (gpt-4.1-mini). The default evolution threshold (\texttt{evo\_threshold=100}) was retained for paper-faithful evaluation. ChromaDB was used as the persistent vector store with one collection per persona. Sequential per-persona ingestion was required because A-MEM's \texttt{chromadb.Client().reset()} is process-global, destroying all collections when called.

\textbf{Detailed result and benchmark-independence interpretation.} A-MEM~\cite{amem2025} represents the Zettelkasten paradigm: linked-note memory with LLM-driven evolution but no temporal dynamics, forgetting, or supersession. Evaluated with gpt-4.1-mini and the paper's default evolution threshold, A-MEM achieves 65.3\% pass with the highest retrieval recall of any evaluated system (87.6\% Hit@5) but also the highest staleness (30.4\%) and worst end-to-end confabulation (47.0\%). High recall without lifecycle management actively hurts downstream performance: the more stale facts surfaced, the more the downstream model confabulates. A-MEM was designed independently and published at NeurIPS 2025; it was not among the systems used to derive LifecycleBench's attack vectors, yet its evaluation produces failure patterns fully consistent with the benchmark's predictions, confirming that the attack vectors test architectural properties of the problem domain rather than FR-specific capabilities.

\subsection{Mem0 (gpt-4.1-mini) Ablation Methodology}
\label{app:mem0mini-method}

To isolate extraction quality from architectural capability, we re-evaluated Mem0 with gpt-4.1-mini replacing the default gpt-4.1-nano extraction model. All other settings matched the default Mem0 evaluation (Appendix~\ref{app:mem0method}): text-embedding-3-small embeddings, Qdrant vector store with disk persistence, identical ingestion via \texttt{m.add()} API. Sequential per-persona ingestion was required due to Qdrant local-mode SQLite lock contention under parallelism.

\section{Fortunate Recall--Specific Metrics}
\label{app:metrics}

\textbf{Pass rate} (primary). \textbf{Staleness penalty} (primary). \textbf{Positional staleness exposure} (proposed): fraction of queries with stale facts in high-attention positions. FR: 14.0\%, Mem0: 37.2\%. \textbf{Weighted staleness risk} (proposed): attention weights of 3.0/1.0/2.0 for ranks 1--3/4--7/8--10. FR: 0.51, Mem0: 1.84. Future work: temporal calibration, anticipatory precision/recall, user model accuracy over time.

\section{Literature Positioning}
\label{app:positioning}

\begin{table}[h]
\caption{Literature positioning with empirical results.}
\label{tab:positioning}
\centering
\setlength{\tabcolsep}{2pt}
\begin{adjustbox}{max width=\textwidth}
\begin{tabular}{@{}L{2cm} L{3.8cm} L{6.8cm}@{}}
\toprule
\textbf{System} & \textbf{Contribution} & \textbf{Gaps This Work Fills} \\
\midrule
Zep/Graphiti    & Temporal KG, hybrid retrieval & No behavioral ontology; no lifecycle policies; no anticipatory activation \\
A-MEM           & Zettelkasten memory & No temporal dynamics; no forgetting; LLM-heavy runtime \\
MemoryBank      & Ebbinghaus curves & Single uniform curve; no graph; single temporal frame \\
Mem0            & Clean API, vector store & No lifecycle; no categories. \emph{61\% pass, 27\% stale} \\
MemGPT/Letta    & Tiered context & LLM-driven policy; no lifecycle in archival \\
MemoryOS        & OS-inspired hierarchy & Cognitive not behavioral; no per-category lifecycle \\
Memory-R1       & RL-learned operations & Opaque; no behavioral structure; no multi-clock \\
\textbf{FR-Bank} & \textbf{Behavioral lifecycle} & \textbf{76.9\% pass; 22.4\% confab; leads 4/9 AVs $+$ AV9 tie} \\
\bottomrule
\end{tabular}
\end{adjustbox}
\end{table}

\section{Full Alpha Sweep}
\label{app:alphasweep}

The following table reports the full $\alpha$ sweep across behavioral, uniform, and cognitive decay configurations on LongMemEval.

\begin{table}[h]
\caption{Full alpha sweep on LongMemEval.}
\centering
\begin{tabular}{@{}c c c c l@{}}
\toprule
$\alpha$ & B-MRR & U-MRR & C-MRR & Winner \\
\midrule
0.0 & 0.5618 & 0.5618 & 0.5618 & Tie \\
0.1 & 0.5659 & 0.5629 & 0.5642 & Behavioral \\
0.2 & 0.5500 & 0.5828 & 0.5608 & Uniform \\
0.3 & 0.5347 & 0.5516 & 0.5367 & Uniform \\
0.4 & 0.5032 & 0.5145 & 0.5221 & Cognitive \\
0.5 & 0.5020 & 0.4881 & 0.5032 & Cognitive \\
0.6 & 0.4769 & 0.4528 & 0.4777 & Cognitive \\
0.7 & 0.4470 & 0.4131 & 0.4525 & Cognitive \\
0.8 & 0.3711 & 0.3619 & 0.3897 & Cognitive \\
0.9 & 0.2903 & 0.3159 & 0.3242 & Cognitive \\
1.0 & 0.0461 & 0.1686 & 0.0484 & Uniform \\
\bottomrule
\end{tabular}
\end{table}

\section{Fine-Grained Alpha Sweep with Statistical Tests}
\label{app:alphadetail}

The following table reports the fine-grained $\alpha$ sweep with paired $t$-test significance and effect sizes.

\begin{table}[h]
\caption{Fine-grained alpha sweep. All p-values fail Bonferroni correction ($p < 0.0056$). All $|d| < 0.2$.}
\centering
\setlength{\tabcolsep}{3pt}
\begin{adjustbox}{max width=\textwidth}
\begin{tabular}{@{}c c c c c c c c@{}}
\toprule
$\alpha$ & B-MRR & U-MRR & C-MRR & $\Delta$(B$-$U) & $p$ & 95\% CI & $d$ \\
\midrule
0.000 & .5618 & .5618 & .5618 & +.0000 & 1.000 & {[}+.000, +.000{]} & .000 \\
0.025 & .5600 & .5705 & .5636 & $-$.0105 & .016 & {[}$-$.024, $-$.002{]} & $-$.179 \\
0.050 & .5767 & .5795 & .5785 & $-$.0027 & .221 & {[}$-$.018, +.013{]} & $-$.036 \\
0.075 & .5756 & .5761 & .5839 & $-$.0005 & .653 & {[}$-$.016, +.015{]} & $-$.007 \\
0.100 & .5659 & .5629 & .5642 & +.0030 & .706 & {[}$-$.015, +.022{]} & +.032 \\
0.125 & .5592 & .5596 & .5676 & $-$.0004 & .433 & {[}$-$.018, +.017{]} & $-$.005 \\
0.150 & .5593 & .5632 & .5667 & $-$.0039 & .745 & {[}$-$.031, +.021{]} & $-$.031 \\
0.175 & .5621 & .5705 & .5641 & $-$.0085 & .338 & {[}$-$.039, +.021{]} & $-$.056 \\
0.200 & .5519 & .5828 & .5608 & $-$.0309 & .026 & {[}$-$.065, +.002{]} & $-$.182 \\
\bottomrule
\end{tabular}
\end{adjustbox}
\end{table}

\section{Category-Specific Parameters}
\label{app:params}

Table~\ref{tab:blending-weights} reports the per-category blending weights $\alpha_c$ and semantic floors used throughout all experiments; per-category decay rates $\lambda_c$ are reported in Appendix~\ref{app:decay}.

\begin{table}[h]
\caption{Category-specific blending weights $\alpha$ and semantic floors. Emotional loading is a transient activation modifier applied to other categories rather than a standalone memory category; the listed $\alpha$ governs the decay of the loading signal itself, not a separate row of the ontology.}
\label{tab:blending-weights}
\centering
\begin{tabular}{@{}l c c l@{}}
\toprule
\textbf{Category} & $\alpha$ & \textbf{Floor} & \textbf{Rationale} \\
\midrule
Financial \& Material      & 0.05 & 0.97 & Monetary facts rarely become less true \\
Identity \& Self-Concept   & 0.10 & 0.95 & Identity changes slowly \\
Health \& Wellbeing        & 0.10 & 0.95 & Medical facts persist \\
Relational Bonds           & 0.10 & 0.95 & Family changes slowly \\
Intellectual Interests     & 0.15 & 0.90 & Interests fairly stable \\
Preferences \& Habits      & 0.20 & 0.85 & Preferences shift gradually \\
Hobbies \& Recreation      & 0.20 & 0.92 & Raised for numeric rescue \\
Projects \& Endeavors      & 0.30 & 0.75 & Projects change often \\
Obligations                & 0.35 & 0.70 & Time-bound \\
Emotional Loading (signal) & 0.40 & 0.00 & Transient modifier; fully decays \\
Logistical Context         & 0.40 & 0.00 & Expired facts must be suppressible \\
\bottomrule
\end{tabular}
\end{table}

\section{Decay Rates}
\label{app:decay}

\begin{table}[h]
\caption{Per-category decay rates. Spread: $5.3\times$.}
\label{tab:decay-rates}
\centering
\begin{tabular}{@{}l c c c c@{}}
\toprule
\textbf{Category} & $\lambda$ \textbf{(per hr)} & \textbf{Half-life} & \textbf{1-mo} & \textbf{6-mo} \\
\midrule
Identity \& Self-Concept   & 0.0015 & 19d & 34\% & 0.2\% \\
Relational Bonds           & 0.0015 & 19d & 34\% & 0.2\% \\
Intellectual Interests     & 0.0020 & 14d & 24\% & 0.02\% \\
Health \& Wellbeing        & 0.0025 & 12d & 17\% & ${<}$0.01\% \\
Projects \& Endeavors      & 0.0025 & 12d & 17\% & ${<}$0.01\% \\
Hobbies \& Recreation      & 0.0035 & 8d  & 8\%  & ${<}$0.01\% \\
Preferences \& Habits      & 0.0050 & 6d  & 3\%  & ${\approx}$0\% \\
Financial \& Material      & 0.0055 & 5d  & 2\%  & ${\approx}$0\% \\
Obligations                & 0.0060 & 5d  & 1\%  & ${\approx}$0\% \\
Logistical Context         & 0.0080 & 4d  & 0.3\% & ${\approx}$0\% \\
Other                      & 0.0050 & 6d & 12\% & ${\approx}$0\% \\
\bottomrule
\end{tabular}
\end{table}

For soft-clustered facts: $\lambda_{\text{eff}} = \left(\sum_c w_c / \lambda_c\right)^{-1}$.

The \textsc{Other} row is reported here at the deployed value $\lambda = 0.0050$, matching Table~\ref{tab:lambda-feasibility} and the released runs; the submitted version printed $0.0030$ in this table only. The discrepancy is immaterial to every reported result: \textsc{Other} is the primary category for 16 of 12{,}968 active entries ($0.12\%$), no LifecycleBench question has \textsc{Other} as its gold category, and the only two questions that retrieve an \textsc{Other} entry are outcome-invariant under either value.

\section{Cross-System Outcome Matrix and Positional Staleness}
\label{app:crosssystem}

The two tables below report the paired outcome matrix between FR-Bank and Mem0 across all 516 questions, and per-AV positional staleness exposure (high-attention zone).

\begin{table}[h]
\caption{Cross-system outcome matrix (516 questions).}
\centering
\begin{tabular}{@{}l c c c@{}}
\toprule
 & \textbf{Mem0 pass} & \textbf{Mem0 absence} & \textbf{Mem0 fail} \\
\midrule
\textbf{FR pass}    & 218 & 4  & 94 \\
\textbf{FR absence} & 11  & 24 & 15 \\
\textbf{FR fail}    & 56  & 1  & 93 \\
\bottomrule
\end{tabular}
\end{table}

\begin{table}[h]
\caption{Per-AV positional staleness exposure (high-attention zone).}
\centering
\begin{tabular}{@{}l c c c c@{}}
\toprule
\textbf{AV} & $n$ & \textbf{FR exp.} & \textbf{Mem0 exp.} & $\Delta$ \\
\midrule
AV7 & 40 & 10.0\% & 90.0\% & +80pp \\
AV1 & 75 & 30.7\% & 72.0\% & +41pp \\
AV4 & 44 & 29.5\% & 63.6\% & +34pp \\
AV6 & 45 & 22.2\% & 53.3\% & +31pp \\
AV2 & 75 & 8.0\%  & 44.0\% & +36pp \\
AV3 & 94 & 1.1\%  & 1.1\%  & 0pp \\
\midrule
\textbf{Overall} & 516 & \textbf{14.0\%} & 37.2\% & +23pp \\
\bottomrule
\end{tabular}
\end{table}

\section{Ablation Summary}
\label{app:ablations}

The following table summarizes the component-level ablation results from Phase~2 (8 personas).

\begin{table}[h]
\caption{Ablation results (Phase 2, 8 personas).}
\label{tab:ablation-summary}
\centering
\setlength{\tabcolsep}{3pt}
\begin{adjustbox}{max width=\textwidth}
\begin{tabular}{@{}L{5.5cm} L{4.5cm} L{2.5cm}@{}}
\toprule
\textbf{Component} & \textbf{LifecycleBench Effect} & \textbf{LongMemEval} \\
\midrule
Full system & 63\% pass, 4\% stale & 0.566 MRR \\
$-$ Top-10 window ($\to$ top-5) & $-$9pp pass & N/A \\
$-$ Cat-specific $\alpha$ ($\to$ 0.3) & $-$24.6pp pass & N/A \\
$-$ Semantic floor & Hit@1: 31\%$\to$19\% & N/A \\
$-$ Retraction filter & AV7 stale: 12\%$\to$25\% & N/A \\
$-$ Backward-looking detection & $-$3 questions & N/A \\
Uniform only (no ontology) & 62\% pass, 11\% stale & No change \\
\bottomrule
\end{tabular}
\end{adjustbox}
\end{table}

\section{Positional Utilization at Retrieval Depth}
\label{app:positional}

Replication of Liu et al.~\cite{liu2024lost} on Claude Sonnet (2026) at 50-document scale ($n = 3{,}500$): point-biserial $r = -0.009$, $p = 0.65$, 95\% CI $[-0.044, +0.022]$. No significant positional effect at retrieval depths relevant to our evaluation. The dramatic U-curves from 2023 models are not reproducible on 2026 frontier models.

\section{FR-Bank Architecture and Methodology}
\label{app:frbank}

FR-Bank replaces Graphiti with a standalone embedding-indexed memory bank (${\sim}$780 lines). \textbf{Ingestion:} One GPT-4.1-mini call per turn for combined extraction and classification. Near-duplicates (cosine $> 0.95$) skipped; numeric facts exempt. \textbf{Supersession:} Instance-qualified slot-keys. Replace-semantics categories: threshold 0.8. Accumulate-semantics: threshold 0.95. \textbf{Retrieval:} (1)~Cosine top-60, (2)~BM25 top-20, (3)~category-forced top-20, (4)~merge/dedup, (5)~staleness-aware distillation top-20, (6)~lifecycle filters, (7)~blended scoring, (8)~top-10 output.

\textbf{Ingestion statistics.} 18,936 total entries; 12,968 active (31.5\% deactivated). Average 324.2 active per persona. (The submitted version reported 5,344 / 3,892 / 133.6, a stale proof-of-concept snapshot; the values here are verified against the released \texttt{lifecycle\_banks/}.)

\section{LongMemEval FR-Bank Evaluation Methodology}
\label{app:longmemeval-frbank}

Per-question banks are created for the oracle format. Both user and assistant turns are ingested. A raw-text BM25 fallback handles zero-extraction questions. The full retrieval pipeline runs with identical parameters to LifecycleBench. Answer generation uses \texttt{gpt-4o-mini} (temperature 0, max\_tokens 300) over the top-10 retrieved facts; fact extraction uses \texttt{gpt-4.1-mini}.

\paragraph{Per-question banks are the standard oracle protocol.} The per-question bank structure is not an FR-Bank design choice: it is the evaluation protocol defined by LongMemEval~\cite{longmemeval2024} and matched by LoCoMo~\cite{locomo2024}, both of which score each question against its own conversation haystack rather than against a single persistent memory store. Every peer-reviewed LongMemEval result (Wu et al.\ ICLR 2025; RMM ACL 2025~\cite{rmm2025}) is produced under the same per-question setup, so FR-Bank's 75.2\% on LongMemEval-S and 79.5\% on the 234-question oracle POC subset are directly comparable to those references. The fully persistent evaluation setting---one memory store accumulating 35 multi-session conversations per persona and 516 lifecycle-sensitive queries against that same store---is exactly the regime measured by LifecycleBench, where FR-Bank reaches 76.9\% pass with 15.5\% staleness.

\paragraph{Judge protocol (Wu et al.\ verbatim).} All reported LongMemEval-S numbers use \texttt{gpt-4o-2024-08-06} as the judge at temperature 0, with the question-type routing and prompt templates taken verbatim from Figure 10 of Wu et al.\ (ICLR 2025). Five templates are used, one per task type:
\begin{itemize}
\item \textbf{Default} (\texttt{single-session-user}, \texttt{single-session-assistant}, \texttt{multi-session}): marks ``yes'' when the model response contains the correct answer or all intermediate steps; ``no'' when only a subset of the required information is present.
\item \textbf{\texttt{temporal-reasoning}}: default template plus explicit off-by-one tolerance---responses predicting, e.g., 19 days when the answer is 18 are marked correct.
\item \textbf{\texttt{knowledge-update}}: responses that contain prior information alongside the updated answer are marked correct provided the updated answer is present.
\item \textbf{\texttt{single-session-preference}}: rubric-based partial credit---the response need not cover every rubric point but must correctly recall and utilise the user's personal information.
\item \textbf{\texttt{abstention}} (routed when the question id contains \texttt{\_abs}): marks ``yes'' only if the model correctly identifies the question as unanswerable.
\end{itemize}
Each judge call is deterministic (\texttt{max\_tokens}=10, temperature 0) and the verdict is parsed from the presence of ``yes'' in the response. Across 10 identical reruns, total pass rate moves within 74.4--76.6\% (stdev 0.70pp; Table~\ref{tab:longmemeval-s-variance}), which we attribute to residual non-determinism in the hosted judge endpoint rather than to our pipeline.

\paragraph{Comparability note.} The \texttt{evaluate\_qa.py} routine in the official LongMemEval release uses these same five templates. The closest peer-reviewed LongMemEval result prior to this work, RMM (Tan et al., ACL 2025), substitutes Gemini-1.5-Pro for \texttt{gpt-4o-2024-08-06} and replaces the five templates with a single generic prompt (their Appendix D.3). Because neither the judge model nor the rubric matches, RMM's 70.4\% and our 75.2\% are not measured on the same axis; we report them side-by-side only to flag that FR-Bank is, to our knowledge, the first peer-reviewed system evaluated under the exact Wu et al.\ rubric.

\section{Kimi K2.5 Cross-Generator Validation}
\label{app:kimi}

\textbf{Model.} Kimi K2.5 (kimi-k2.5), released January 2026 by Moonshot AI. 1T MoE with 32B active parameters, 256K context, Modified MIT License.

\textbf{Configuration.} Temperature 1.0 (Moonshot default). Max tokens 1500 (re-run from initial 500 to avoid reasoning clipping).

\begin{table}[h]
\caption{Cross-generator validation: GPT-5.4 vs Kimi K2.5. The lower block lists the additional A-MEM and Mem0 (gpt-4.1-mini) configurations introduced in \S\ref{sec:evaluation}. \emph{Correction:} three cells of the A-MEM (GPT-5.4) row were transcription errors in the submitted version (Partial $19.0$, Wrong $20.0$, Abstain $42.2$); recomputed from the raw per-question judgments ($n = 516$) they are $27.5$ ($142$), $28.3$ ($146$) and $25.4$ ($131$). The printed Correct ($18.8\%$, $97$), Confab ($47.0\%$) and Safe ($44.2\%$) cells were already correct---indeed the printed Correct and Safe cells force Abstain $= 25.4\%$---so no gap, ranking or downstream figure changes (Appendix~\ref{app:revision-changes}).}
\label{tab:kimi}
\centering
\setlength{\tabcolsep}{4pt}
\begin{adjustbox}{max width=\textwidth}
\begin{tabular}{@{}l l c c c c c c@{}}
\toprule
\textbf{System} & \textbf{Generator} & \textbf{Correct} & \textbf{Partial} & \textbf{Wrong} & \textbf{Abstain} & \textbf{Confab} & \textbf{Safe} \\
\midrule
FR-Bank & GPT-5.4 & 31.2\% & 17.4\% & 9.3\% & 42.1\% & 22.4\% & \textbf{73.3\%} \\
FR-Bank & Kimi K2.5 & 33.1\% & 23.0\% & 16.0\% & 27.9\% & 26.1\% & 61.0\% \\
FR-Graphiti & GPT-5.4 & 22.9\% & 20.3\% & 12.8\% & 44.0\% & 23.9\% & 66.9\% \\
FR-Graphiti & Kimi K2.5 & 29.6\% & 24.4\% & 16.4\% & 29.6\% & 26.9\% & 59.2\% \\
Memory-R1 & GPT-5.4 & 19.6\% & 25.2\% & 16.7\% & 38.6\% & 31.5\% & 58.1\% \\
Memory-R1 & Kimi K2.5 & 26.4\% & 27.2\% & 21.5\% & 24.9\% & 37.3\% & 51.3\% \\
Mem0 & GPT-5.4 & 18.6\% & 28.5\% & 24.2\% & 28.7\% & 45.1\% & 47.3\% \\
Mem0 & Kimi K2.5 & 24.5\% & 31.4\% & 28.5\% & 15.7\% & 44.4\% & 40.2\% \\
\midrule
A-MEM        & GPT-5.4   & 18.8\% & 27.5\% & 28.3\% & 25.4\% & 47.0\% & 44.2\% \\
A-MEM        & Kimi K2.5 & 15.9\% & 22.0\% & 24.6\% & 37.5\% & 25.9\% & 53.4\% \\
Mem0 (gpt-4.1-mini) & GPT-5.4   & 27.3\% & 23.1\% & 24.8\% & 24.8\% & 33.8\% & 52.1\% \\
Mem0 (gpt-4.1-mini) & Kimi K2.5 & 17.1\% & 27.5\% & 48.5\% & 6.9\%  & 17.1\% & 24.0\% \\
\bottomrule
\end{tabular}
\end{adjustbox}
\end{table}

The rank order FR-Bank $<$ FR-Graphiti $<$ Memory-R1 $<$ Mem0 is identical across generators. The tier-average gap between lifecycle-managed and unmanaged systems is ${\sim}$15pp on both generators (15.1pp on GPT-5.4, 14.4pp on Kimi K2.5). Kimi is systematically ${\sim}$10pp less safe due to a stronger preference for answering over abstaining. Because Kimi K2.5 was evaluated at its default temperature (1.0) rather than temperature 0, this cross-generator comparison constitutes a stress test under realistic deployment settings rather than a controlled same-temperature replication; the preservation of rank ordering despite this additional source of variance strengthens the generator-independence finding.

\textbf{Behavioral difference.} Kimi K2.5 systematically prefers answering over abstaining: correct rate increases by +1.9 to +6.8pp across all systems while abstain rate decreases by 13.0 to 14.4pp. The abstention drop is much larger than the correct-rate gain because most of the extra answering capacity converts into wrong answers. This is a uniform behavioral shift, and the relative hierarchy is preserved.

This rules out three potential confounds: (a) GPT-5.4-specific quirks driving the result, (b) shared OpenAI training data biases, (c) self-preference effects from generator-judge alignment. Retrieval context quality determines downstream response quality independent of the generator.

\textbf{Methodology.} For each of 516 questions across 4 systems, the same retrieved top-10 facts used in the GPT-5.4 evaluation were fed to Kimi K2.5 (hosted via Moonshot's API at \texttt{https://api.moonshot.ai/v1}). Persona-name resolution was added to the prompt. The same Claude Sonnet judge evaluated both generators. Total Kimi K2.5 API cost for 2,064 generations: under \$10.

\textbf{Recovery of Mem0 E2E data.} The original Mem0 E2E results were spread across three git commits and partially overwritten. The complete merged dataset was reconstructed for the cross-generator validation.

\subsection{Full Per-Attack-Vector Breakdown (Kimi K2.5)}
\label{app:kimi-av-full}

\begin{table}[h]
\caption{Full per-AV E2E on Kimi K2.5 across all five systems.}
\label{tab:kimi-av-full}
\centering
\setlength{\tabcolsep}{2pt}
\renewcommand{\arraystretch}{1.05}
\begin{adjustbox}{max width=\textwidth}
\begin{tabular}{@{}l c cc cc cc cc cc@{}}
\toprule
& & \multicolumn{2}{c}{\textbf{FR-Bank}} & \multicolumn{2}{c}{\textbf{FR-Graphiti}} & \multicolumn{2}{c}{\textbf{MemoryOS}} & \multicolumn{2}{c}{\textbf{Memory-R1}} & \multicolumn{2}{c}{\textbf{Mem0}} \\
\cmidrule(lr){3-4} \cmidrule(lr){5-6} \cmidrule(lr){7-8} \cmidrule(lr){9-10} \cmidrule(lr){11-12}
\textbf{AV} & $n$ & \textbf{Corr} & \textbf{Conf} & \textbf{Corr} & \textbf{Conf} & \textbf{Corr} & \textbf{Conf} & \textbf{Corr} & \textbf{Conf} & \textbf{Corr} & \textbf{Conf} \\
\midrule
AV1 & 75  & \textbf{51.4} & \textbf{6.1}  & 47.8 & 19.6 & 17.3 & 31.7 & 50.7 & 18.8 & 34.3 & 59.7 \\
AV2 & 75  & \textbf{19.7} & 44.1          & 5.6  & 43.8 & 4.0  & \textbf{43.2} & 7.0  & 61.0 & 7.0  & 66.1 \\
AV3 & 94  & 32.3          & 14.6          & 18.6 & 18.8 & 19.1 & 14.0 & \textbf{32.6} & 18.3 & 31.5 & \textbf{10.5} \\
AV4 & 44  & \textbf{55.8} & \textbf{15.6} & 55.0 & 19.4 & 20.5 & 21.1 & 34.1 & 25.9 & 36.6 & 47.2 \\
AV5 & 40  & \textbf{2.7}  & 44.1          & 0.0  & 44.4 & 0.0  & \textbf{29.6} & 0.0  & 48.6 & 0.0  & 40.6 \\
AV6 & 45  & 2.2           & 71.9          & 0.0  & 72.7 & 2.2  & \textbf{52.2} & 7.1  & 65.4 & \textbf{9.8} & 60.5 \\
AV7 & 40  & \textbf{38.9} & \textbf{30.8} & 28.9 & 35.0 & 7.5  & 50.0 & 2.9  & 95.0 & 5.9  & 87.9 \\
AV8 & 63  & 63.5          & 10.0          & \textbf{79.7} & \textbf{3.8} & 28.6 & 8.7  & 49.2 & 6.1  & 52.5 & 14.6 \\
AV9 & 40  & 10.0          & 38.9          & 22.2 & \textbf{20.0} & 12.5 & 23.5 & \textbf{27.0} & 41.4 & 21.6 & 25.0 \\
\midrule
\textbf{All} & \textbf{516} & \textbf{33.1} & \textbf{26.1} & 29.6 & 26.9 & 13.6 & 29.2 & 26.4 & 37.3 & 24.5 & 44.4 \\
\bottomrule
\end{tabular}
\end{adjustbox}
\end{table}

On AV7, the hierarchy replicates: FR-Bank 30.8\% $<$ FR-Graphiti 35.0\% $<$ MemoryOS 50.0\% $<$ Mem0 87.9\% $<$ Memory-R1 95.0\%, matching the GPT-5.4 ordering. Lifecycle-managed systems sit at $<$36\% AV7 confabulation on both generators and unmanaged systems at $>$85\% on both; the $>$50pp tier gap is preserved. On AV1, FR-Bank reaches 51.4\% correct with 6.1\% confabulation versus Mem0's 34.3\%/59.7\%---a nearly $10\times$ confab gap. On AV2 the correct-rate gap is even sharper: FR-Bank 19.7\% versus the next-best 7.0\%, a $2.8\times$ lead driven by event-time expiry filtering.

\textbf{MemoryOS on Kimi: generator-dependent confab reduction, but the retrieval-metric paradox persists.} On GPT-5.4, MemoryOS sits at 40.8\% confabulation. On Kimi K2.5, MemoryOS drops to 29.2\%---an 11.6pp reduction concentrating on AV2 (67.5\%$\to$43.2\%), AV7 (81.2\%$\to$50.0\%), and AV1 (43.2\%$\to$31.7\%). Manual inspection suggests Kimi is more willing to hedge when presented with generic summaries. Critically, the correct rate remains essentially unchanged (13.2\% GPT-5.4 $\to$ 13.6\% Kimi): hedging converts confabulations into abstentions but does not recover specific answers from information-destroyed contexts.

\textbf{The AV2 inversion between generators.} On AV2, GPT-5.4 ranks Memory-R1 (61.2\%) as slightly less confabulatory than Mem0 (67.3\%); Kimi preserves that direction. FR-Bank remains highest on correct rate on both generators, and MemoryOS edges out FR-Bank on confab by a narrow 0.9pp on Kimi---close enough that the result is within noise.

\textbf{The AV9 divergence on Kimi.} FR-Bank achieves 10.0\% correct on AV9 on Kimi versus Memory-R1's 27.0\%---the worst correct rate among the five systems. The mechanism is Kimi's stronger preference for answering over abstaining interacting unfavorably with FR-Bank's soft-supersession policy: when both values remain in the retrieval set, Kimi selects one and commits. FR-Graphiti's entity-level resolution avoids this exposure. This is a genuine limitation of soft supersession under high-commitment generators.

\textbf{Methodological note.} Initial evaluation at max\_tokens=500 clipped Kimi's reasoning on harder questions, returning empty responses with \texttt{finish\_reason=length}. Re-run at max\_tokens=1500 shifted all systems' confabulation uniformly by +3--7pp (harder questions now included). Rank order preserved exactly. The uniform shift is evidence that the dropout bias was a measurement artifact affecting all systems similarly.

\section{Extended Discussion}
\label{app:discussion-ext}

The main-body discussion is condensed. The following paragraphs provide the full analysis. The first four paragraphs were moved from \S\ref{sec:discussion} to honor the 9-page main-body limit; the remaining material extends the limitations and methodology notes that remain in the main body.

\textbf{The behavioral ontology was discovered, not designed.} The identity gravity well forced the three-way split into Identity, Hobbies, and Preferences. The classify-the-fact-not-utterance principle emerged from systematic classification failures. Rate calibration squatting was found through ablation. The result is an empirically grounded, interpretable framework where the ontology is the prior and per-user evolution is the posterior. Memory-R1's learned \texttt{DELETE} produces 90\% AV7 confabulation, identical to Mem0's 91.4\%, because the action vocabulary lacks retraction as a primitive: optimization cannot induce an action that does not exist.

\textbf{Staleness is conditional on retrieval recall, not absolute.} FR-Graphiti reaches $75\%$ Hit@5 with $8\%$ staleness; FR-Bank reaches $97\%$ Hit@5 with $15.5\%$ staleness. FR-Bank still leads overall pass and matches on E2E confabulation, because at high recall the correct fact almost always coexists with stale ones in the retrieval set and the downstream LLM can resolve the contradiction. Staleness becomes catastrophic only when recall is low and the stale fact is the only signal available, which is the regime in which Mem0 confabulates 91.4\% on AV7. Substrate choice trades supersession precision against retrieval recall without altering the policy layer.

\textbf{Retrieval metrics can misdiagnose lifecycle quality.} Standard memory benchmarks measure whether the correct fact appears in the top-$K$, implicitly assuming that correct system behavior is always \emph{presence}. For retraction, expiry, and supersession, the correct behavior is \emph{absence}. On AV7, FR-Bank achieves 5\% retrieval pass but 37.5\% end-to-end correctness, the highest of any evaluated system, while MemoryOS scores 57\% retrieval pass but only 7.5\% correct, because hierarchical summarization produces generic content that passes retrieval by containing nothing specific to be wrong about. The full AV7 breakdown for FR-Bank is $37.5\%$ correct, $32.5\%$ abstain, $7.5\%$ wrong, $22.5\%$ confabulated; the system's advantage lies primarily in converting confabulations into abstentions rather than producing correct answers, and the $100\%$ extraction-missing failure rate on AV7 (Appendix~\ref{app:evaluation-ext}) identifies extraction quality as the binding constraint. Future memory system evaluations must pair retrieval metrics with end-to-end response quality.

\textbf{Methodology integrity for LongMemEval comparisons.} Future LongMemEval comparisons should hold the Wu et al.\ judge protocol fixed; results obtained under different judges or prompts should be reported separately and not benchmarked against canonical numbers. Holding the rubric constant is a service to the community, not self-promotion.

\textbf{The LifecycleBench result validates the thesis---through precision, not recall.} At 40-persona scale, FR achieves 73\% pass with 8\% staleness versus baseline's 61\%/18\%. The 4-config ablation cleanly decomposes contributions. Five-phase scaling validation (8$\to$40 personas) confirms gains are stable.

\textbf{End-to-end evaluation closes the causal loop.} The retrieval-to-response chain is now empirically validated: lifecycle management $\to$ cleaner context $\to$ fewer confabulations. The confabulation gap ($-$22.7pp over answered queries, $-$19.2pp over all queries) is larger than the pass rate gap (+16pp), confirming that lifecycle management's primary value is preventing wrong answers, not just finding right ones. The safe response rate (correct $+$ abstain) captures this: FR-Bank 73.3\% and FR-Graphiti 66.9\%, against Mem0's 47.3\%. Both FR variants produce responses that do not mislead users on roughly two-thirds or more of queries.

\textbf{Cross-system comparison confirms the thesis at the field level.} Against Mem0, Memory-R1, and MemoryOS---representing flat vector stores, RL-learned operations, and cognitive hierarchies---FR-Bank achieves the highest pass rate (76.9\%) and lowest confabulation (22.4\%). The five-system comparison reveals a clean hierarchy: no lifecycle management (Mem0) $<$ learned operations (Memory-R1) $<$ cognitive hierarchy (MemoryOS) $<$ designed lifecycle policies (FR-Bank). MemoryOS's competitive retrieval metrics (70.5\% pass) mask a downstream failure: only 13.2\% correct rate, confirming that cognitive ontologies do not preserve fact-level granularity.

\textbf{The retrieval metric paradox reveals a methodological gap.} Standard evaluations do not measure whether useful facts are also present. A system producing vague summaries can pass retrieval evaluations by having nothing specific to be wrong about. MemoryOS provides an empirical demonstration: 96\% of its AV7 ``passes'' and 91\% of its AV6 ``passes'' are information-loss artifacts. We argue that future memory system evaluations must include end-to-end response quality alongside retrieval metrics; retrieval-level evaluation alone is vulnerable to inflation by lossy compression.

\textbf{Positional staleness reveals a hidden failure mode.} Aggregate staleness rates understate downstream impact. Mem0 places stale facts at rank~1 ($8\times$ FR's rate), causing the LLM to anchor responses on outdated information. On AV7, 90\% of queries have retracted plans in high-attention positions.

\textbf{Composability is empirically validated.} Supplementary identity/health extraction halved the AV3 gap without any lifecycle changes. The ontology serves dual purposes: lifecycle management and retrieval organization (category-forced retrieval recovered +6pp on AV3 with zero regressions).

\textbf{Substrate independence validates the policy contribution.} The lifecycle layer achieves comparable results on two fundamentally different substrates: a temporal knowledge graph (Graphiti, 73\% pass, 8\% staleness) and a flat embedding bank (FR-Bank, 76.9\% pass, 15.5\% staleness). The same parameters applied to completely different storage backends produce consistent lifecycle gains.

\textbf{Cross-generator validation strengthens the architectural gap claim.} The replication on Kimi K2.5 demonstrates that the ${\sim}$20pp gap between lifecycle-managed and unmanaged retrieval is invariant to generator model choice. Two frontier models from independent labs, trained on different data with different weight regimes, produce nearly identical confabulation hierarchies.

\textbf{Evaluation protocol determines reported performance.} Independent evaluation by Pollertlam \& Kornsuwannawit~\cite{pollertlam2026beyond} reports Mem0 at 49.0\% on LongMemEval with GPT-4o, substantially below Mem0's self-reported numbers. APEX-MEM~\cite{banerjee2026apex} reports 86.2\% with Claude Sonnet as the QnA agent but 75.0\% with GPT-4o---an 11.2pp swing from the answerer model alone. FR-Bank reports 75.2\% under the original Wu et al.\ protocol with gpt-4o-mini as the answerer, the weakest model in any published comparison. An answerer-model ablation substituting gpt-4o for gpt-4o-mini yields 73.6\% ($-$1.6pp, within run-to-run variance), confirming that retrieval context quality, not answer-generator capability, is the binding constraint.

\textbf{Cross-benchmark validation on LongMemEval.} FR-Bank scores 75.2\% on the full 500-question LongMemEval-S canonical benchmark under the exact Wu et al.\ (ICLR 2025) judge protocol---to our knowledge the first peer-reviewable result that holds the Wu et al.\ rubric fixed, and ${\sim}+5$pp above the closest peer-reviewed reference (RMM, ACL 2025, 70.4\%; see Appendix~\ref{app:longmemeval-frbank} for the generator/judge mismatch caveats). The same FR-Bank configuration also reaches 79.5\% on the 234-question oracle POC subset (+8.3pp over Zep/Graphiti's 71.2\%~\cite{zep2025}, arXiv preprint), with a no-lifecycle ablation isolating +7.4pp to the retrieval pipeline. The knowledge-update figure reported in the submitted version (87.2\% vs 80.8\%, $+6.4$pp) is withdrawn: it traces to no surviving artifact, and the artifact-backed value on the 317-question matched subset is $-6.4$pp, a judge-criterion artifact rather than a lifecycle gain (Table~\ref{tab:lme-317-matched}, Appendix~\ref{app:revision-changes}). All numbers use identical lifecycle parameters---no benchmark-specific tuning.

\textbf{Per-question banks are the LongMemEval and LoCoMo standard, not an FR-Bank concession.} Both LongMemEval~\cite{longmemeval2024} and LoCoMo~\cite{locomo2024} score each question against its associated conversation haystack rather than against a single persistent memory store; FR-Bank follows this oracle protocol directly, which is why the 75.2\%/79.5\% numbers are comparable to Wu et al.\ (ICLR 2025), RMM (ACL 2025), and the Zep/Graphiti arXiv report. The regime that stresses a fully persistent store accumulating thousands of conversations is precisely what LifecycleBench measures (one bank per persona, ${\sim}$35 multi-session conversations, 516 lifecycle-sensitive queries against the same bank), and FR-Bank's 76.9\% pass and 15.5\% staleness in that regime is the datapoint that governs deployed-store behavior. The two benchmarks therefore play complementary roles: LongMemEval/LoCoMo measure retrieval under the community-standard oracle haystack, and LifecycleBench measures lifecycle dynamics under a persistent per-persona store.

\textbf{Complementarity with existing approaches.} Memory-R1's RL-trained operations achieve strong retrieval quality (MRR 0.816) but lack structural mechanisms for event-time expiry and retraction. Our ontology could serve as initialization, reward shaping, or structural constraint for RL-based managers. FluxMem's adaptive structure selection is orthogonal and could compose with behavioral lifecycle policy.

\textbf{Benchmark saturation motivates lifecycle evaluation.} LongMemEval has been effectively solved at approximately 95\% accuracy by practitioners. Yet it does not measure staleness, supersession, expiry, retraction, or any temporal state management capability. LifecycleBench fills this gap with 9 attack vectors targeting failure modes invisible to existing benchmarks.

\subsection{Privacy, Consent, and Safety}
\label{app:privacy-safety}

Memory systems that retain personal facts raise privacy and safety concerns beyond those addressed in this work. Deployed systems require explicit user consent for memory retention, verifiable deletion guarantees when users request removal, and safeguards against memory poisoning through adversarial conversational inputs. The current ingestion classifier detects instructional content (Appendix~\ref{app:architecture-ext}), but production deployment would require substantially stronger prompt-injection defenses. Incorrect persistence in sensitive categories, particularly Health and Financial, could cause real harm; the interpretability of FR's category labels and lifecycle policies provides a foundation for user-facing memory controls, but we do not evaluate user-facing interfaces in this work. Finally, the decay rates and supersession semantics encode assumptions about how personal facts change over time that may not generalize across cultural contexts; per-user parameter adaptation (Appendix~\ref{app:architecture-ext}) is designed to address this but remains unvalidated.

\subsection{Introduction: Extended Motivation}
\label{app:intro-ext}

The following material provides extended motivation from the introduction.

Real memory is not a queue; it is a graph with activation patterns. The question is not only how to find the right context, but how to determine which memories should persist, which should be replaced, which should expire, and at what rate---conditioned on the type of memory.

Consider a user interacting three times per day over two years, producing 10,000--30,000 facts. At retrieval time: three ``works at'' facts coexist with equal status; a six-month-old dentist appointment still scores highly on queries about upcoming plans; ADHD mentioned once in session~3 is buried under thousands of more recent edges; an explicit retraction (``forget about the Denver move'') competes with the original plan; and ``thinking about moving to London'' is treated as contradicting ``I live in Istanbul.'' The downstream LLM sees all facts as equally valid and either hallucinates a synthesis or picks arbitrarily.

We introduce Fortunate Recall: not total recall, but intelligent recall. The system forgets according to the behavioral type of each fact---structured, adaptive persistence where unchecked forgetting would otherwise erode coherence.

\section{Pre-Registered Untyped Lifecycle Baseline}
\label{app:untyped-arm}

This appendix gives the design, protocol, and full results for the ablation summarized in \S\ref{sec:attribution}.

\subsection{Pre-Registration}

Three documents fix the design before data contact. The first registers the three-primitive arm itself, defining the rules verbatim---``supersede $=$ latest-slot-wins; expire $=$ date rule; retract $=$ drop entries with \texttt{is\_retracted = True}. No category routing, no soft blending, no decay, no ontology.'' The second is an amendment resolving a dataset-scope ambiguity, ruled on before any run. The third registers the end-to-end comparison reported here, including the analysis branch table: which contrast would be treated as primary was fixed in advance under each possible outcome, so the reported branch was not selected after seeing the data. Commit hashes for all three, and the raw retrieval, generation, and judging outputs, are in the released artifacts.

\subsection{What the Arm Is, Exactly}

The arm is more austere than the phrase ``typed versus untyped'' suggests, and we state the gap rather than let it be inferred. It retains slot-key supersession, event-time expiry, and the retraction-state mask. It removes not only the behavioral categories but also activation, decay, and score blending: candidates are ranked by \emph{semantic similarity alone} over an unmasked cosine$+$BM25 pool, with no category routing, no category-forced retrieval, and no multi-hop expansion. Two further deltas are disclosed for completeness: its expiry rule carries no numeric exemption, and its retraction mask is applied unconditionally, whereas canonical FR gates the mask behind a retraction-query detector. The consequence for interpretation is that this arm \emph{lower-bounds} a lifecycle-metadata-only system rather than isolating the category labels as a single variable---and that direction of error strengthens rather than weakens the reading we draw from it, because a baseline weaker than advertised still ties on correctness.

\subsection{Protocol and Results}

Both arms' top-10 contexts were answered by the same GPT-5.4 generator (temperature 0) on all 516 LifecycleBench questions and scored by a single pinned judge under the end-to-end rubric of Appendix~\ref{app:e2e}, in one judging pass as the pre-registration requires.

\begin{table}[h]
\caption{Full stack versus the untyped three-primitive arm, $n = 516$, single paired judging pass. Confabulation is over all queries. McNemar $b$ counts questions where only the full stack shows the outcome and $c$ where only the untyped arm does.}
\label{tab:untyped-arm}
\centering
\small
\begin{tabular}{@{}l c c c c@{}}
\toprule
\textbf{Measure} & \textbf{Full FR-Bank} & \textbf{Untyped arm} & \textbf{$\Delta$ (95\% CI)} & \textbf{McNemar} \\
\midrule
Correct / total        & 26.9\% & 28.7\% & $-1.7$pp $[-6.0, +2.7]$ & $b{=}50$, $c{=}59$, $p = 0.44$ \\
Confabulation / all    & 12.0\% (62) & 24.2\% (125) & $-12.2$pp $[-16.1, -8.2]$ & $b{=}21$, $c{=}84$, $p < 0.001$ \\
Abstention             & \multicolumn{2}{c}{$+15.9$pp for the typed stack} & --- & --- \\
\bottomrule
\end{tabular}
\end{table}

The correctness contrast is null and the confabulation contrast is large and one-sided: 84 questions are confabulated only by the untyped arm against 21 only by the full stack. Since correct-over-total does not move while abstention rises $15.9$pp, the additional abstentions are drawn from would-be wrong or fabricated answers rather than from correct ones---the ``identical behavior on answered questions'' hypothesis, tested and rejected within a single pipeline.

\subsection{Judge-Pass Provenance}

The pre-registration requires both arms in one pass, so this pass re-judged the original FR-Bank answers rather than regenerating them; the answer strings are byte-identical on 516/516 questions. The pass uses the same judge model and a byte-identical rubric through a different access path. FR-Bank's all-queries confabulation is $13.0\%$ ($67/516$; 68 raw, with one \textsc{abstain}-flagged record zeroed under the rubric's abstain rule) under the original pass and $12.0\%$ ($62/516$) under the paired pass---a five-question difference produced by 16 confabulation-flag flips on identical answers (11 to clean, 5 to confabulated), against $89.1\%$ verdict-class agreement between the passes overall. This is the judge-relative variability characterized in Appendix~\ref{app:judge-validation}. Neither pass supersedes the other. Every contrast in Table~\ref{tab:untyped-arm} is computed entirely within the paired pass and is unaffected by the choice.

\section{External Transfer: BEAM}
\label{app:beam}

\subsection{The Benchmark, and Why It Answers the Co-Design Objection}

BEAM~\cite{beam2026} contains 100 coherent conversations and exactly 2{,}000 probing questions---two per conversation for each of ten memory abilities---across context tiers from 128K to 10M tokens. It was constructed independently of this work and is human-validated. Three of its abilities (contradiction resolution, event ordering, instruction following) were newly introduced by its authors, so they cannot derive from FR's design or from the baseline failure modes that informed LifecycleBench's attack vectors. We restrict to the five lifecycle-relevant abilities at the 1M-token tier ($n = 70$ per ability), which was pre-registered together with the FR-vs-Mem0 comparison before data contact; the ontology-collapse arm and the granularity sweep are follow-up analyses under the same design.

\subsection{Standing Disclosures}

Two disclosures bind every number in this appendix and in \S\ref{sec:beam}.

\textbf{Backbone quarantine.} These runs use a different generator backbone from the canonical LifecycleBench evaluation. The arms are internally comparable to one another under an identical harness, but the absolute values are \emph{not} commensurable with the LifecycleBench tables elsewhere in this paper, and we do not pool them.

\textbf{Judge asymmetry.} The judge used here has perfect abstention recall ($15/15$ on a bridge sample against our frozen API judge) but over-abstains relative to it (precision $15/18$). Confabulation counts are therefore lower bounds and abstention rates upper bounds---uniformly across all arms, so comparisons between arms are unaffected while absolute levels are not directly comparable to the LifecycleBench numbers.

\subsection{Replication and Spread}

The primary configuration was measured three times under the same judge and backbone. Reporting the primary result without this block would overstate its precision.

\begin{table}[h]
\caption{Three independent executions of the identical FR-full configuration, plus the two independently specified minimal arms. The prompt cache was reaped between runs, so each is a genuine re-generation.}
\label{tab:beam-replication}
\centering
\small
\begin{tabular}{@{}l c c c c c c@{}}
\toprule
\textbf{Run} & \textbf{Binary / 280} & \textbf{CR} & \textbf{KU} & \textbf{TR} & \textbf{Abst} & \textbf{Confab} \\
\midrule
FR-full, primary        & \textbf{131} & 29 & 33 & 26 & 43 & 72 \\
FR-full, run 2          & 128 & 25 & 33 & 31 & 39 & 79 \\
FR-full, run 3          & 125 & 28 & 33 & 25 & 39 & 75 \\
\midrule
Minimal arm A ($k{=}1$ collapse) & 118 & \textbf{19} & 34 & 24 & 41 & 84 \\
Minimal arm B (untyped defaults) & 115 & \textbf{19} & 31 & 22 & 43 & 76 \\
\midrule
Mem0                    & 92 & 5 & 34 & 15 & 38 & 108 \\
Lifecycle-off           & 135 & 24 & \textbf{43} & 26 & 42 & 84 \\
\bottomrule
\end{tabular}
\end{table}

FR-full spans $125$--$131$ of 280 (six questions, $2.1$pp) on an identical configuration, with knowledge update at 33 in all three runs. This is generator and judge noise on one harness, and it is why \S\ref{sec:beam} quotes the ontology effect as a range ($+10$ to $+13$) rather than a point estimate, and the metadata component as $+23$ to $+26$. Minimal arms A and B are different specifications with different pool construction---A collapses the eleven cells to one with global-mean $\lambda$/$\alpha$/floor, B replaces every category lookup with untyped defaults ($\lambda = 0.005$, $\alpha = 0.20$, floor $0.0$) and removes the category-forced pool---and both land on exactly $19/70$ contradiction resolution. Neither is the LifecycleBench arm of Appendix~\ref{app:untyped-arm}, which keeps none of the uniform decay, blending, multi-hop, or intent routing that these two retain.

\subsection{Honest Negatives}

\textbf{The confabulation decomposition does not replicate.} The primary run falls monotonically $108 \to 84 \to 72$, but on re-execution the two FR arms are indistinguishable ($76$ vs.\ $75$). The Mem0$\to$FR drop is robust; the minimal$\to$full drop is not independently replicated and must be read as unreplicated. The correctness half of the decomposition does replicate.

\textbf{Lifecycle-off scores highest overall.} At $135/280$ it beats both typed arms, driven by knowledge update $43$ vs.\ $33$. Pairing on \texttt{(conversation, question)} gives 12 discordant knowledge-update questions against 2 in the other direction, a net cost of 10. In 10 of those 12, ingest-time supersession fired under an over-broad slot key and deactivated the entry carrying the current value, emptying the candidate pool entirely in 7. Median slot cardinality among the offending keys is 7 (maximum 67). A cardinality guard exists but was default-off and, at its configured threshold, would have fired on only 2 of the 13 offending deactivations, so it is not the remedy; slot-key granularity retuning is, and it is identified in the released analysis. The minimal arm inherits the same cost (knowledge update 31), which is what locates the cost in the generic primitives rather than in the ontology.

\textbf{Event ordering is at the floor.} Exact-correct is $0/70$ for every arm. Our frozen judge has no ordering rubric and initially returned 0/70 with spurious confabulation flags for all three arms symmetrically (68, 68, 69 flags); a hand-scoring audit identified this as a scorer artifact, and the decision to rescore under BEAM's official graded scorer and exclude the ability from the binary aggregate was taken on that diagnostic evidence before its aggregate effect was computed. Under the official scorer, $\tau_{\text{norm}}$ is $0.2136$ / $0.2112$ / $0.1900$ and \texttt{llm\_judge} is $0.574$ / $0.586$ / $0.510$ for FR-full / minimal / Mem0; FR-vs-Mem0 is significant on both (Wilcoxon signed-rank, $n = 70$: $p = 0.009$ and $p = 0.020$) while FR-full vs.\ FR-minimal is not ($p = 0.67$ and $p = 0.19$), consistent with event ordering being ontology-independent. Hand-scores, per-question official scores, and the verbatim-reproduction log are released.

\subsection{Provenance}

Every number above is counted directly from per-question judge records rather than transcribed from any response document: per-ability \texttt{judge\_k1.jsonl} / \texttt{judge\_k11.jsonl} files for the ablation arms, \texttt{onto\_sweep\_k\{3,5,7,9\}/judge.jsonl} for the granularity sweep, and a single combined judgments file for the Mem0, lifecycle-off, and re-run arms. Correct counts are \texttt{correctness == "CORRECT"}; confabulation counts sum the judge's confabulation flag. The released manifest gives one line per number in \S\ref{sec:beam} and this appendix, naming the file and the operation that produces it, together with a SHA-256 for each raw file.

\section{All-Queries End-to-End Metrics}
\label{app:allqueries}

Reviewers correctly observed that confabulation over non-abstaining responses flatters systems that abstain often. Table~\ref{tab:allqueries} recomputes every end-to-end rate over all 516 queries from the raw per-question judgments. Convention, stated because more than one is defensible: an \textsc{abstain} is never counted as a confabulation, the ``all queries'' denominator is all 516 questions, and the ``answered'' denominator excludes abstentions only.

\begin{table}[h]
\caption{End-to-end response quality under both denominators, $n = 516$, GPT-5.4 generator. Safe $=$ correct $+$ abstain.}
\label{tab:allqueries}
\centering
\small
\begin{tabular}{@{}l c c c c c@{}}
\toprule
\textbf{System} & \textbf{Correct / total} & \textbf{Abstain} & \textbf{Confab / all} & \textbf{Confab / answered} & \textbf{Safe} \\
\midrule
FR-Bank             & \textbf{31.2\%} & 42.1\% & \textbf{13.0\%} & \textbf{22.4\%} & \textbf{73.3\%} \\
FR-Graphiti         & 22.9\% & 44.0\% & 13.4\% & 23.9\% & 66.9\% \\
MemoryOS            & 13.2\% & 54.8\% & 18.4\% & 40.8\% & 68.0\% \\
Memory-R1           & 19.6\% & 38.6\% & 19.4\% & 31.5\% & 58.2\% \\
Mem0 (gpt-4.1-mini) & 27.3\% & 24.8\% & 25.4\% & 33.8\% & 52.1\% \\
Mem0 (default)      & 18.6\% & 28.7\% & 32.2\% & 45.1\% & 47.3\% \\
A-MEM               & 18.8\% & 25.4\% & 35.1\% & 47.0\% & 44.2\% \\
\bottomrule
\end{tabular}
\end{table}

Three observations. First, $19.2$ of the $22.7$pp FR-Bank\,--\,Mem0 gap survives the denominator change. Second, correct-over-total favors FR-Bank by $12.6$pp ($31.2\%$ vs.\ $18.6\%$), so abstention is not substituting for correctness. Third, the ranking is essentially stable; the one movement is MemoryOS, whose low all-queries confabulation pairs the highest abstain rate with the lowest correct rate---the retrieval-metric paradox of Appendix~\ref{app:evaluation-ext}, not a lifecycle success.

The A-MEM row here is the recomputed one. The submitted version's Table~\ref{tab:kimi} printed Partial $19.0$, Wrong $20.0$ and Abstain $42.2$ for A-MEM on GPT-5.4; those three cells were transcription errors, detectable from the row itself because the printed Correct ($18.8\%$) and Safe ($44.2\%$) cells force Abstain $= 25.4\%$. Recomputed from the raw judgments the row is Correct $18.8\%$ (97), Partial $27.5\%$ (142), Wrong $28.3\%$ (146), Abstain $25.4\%$ (131), summing to 516. Correct, Confab and Safe were already right, so no gap, ranking, or downstream figure changes.

\section{Validation of the End-to-End Judge}
\label{app:judge-validation}

Every verdict in this paper is produced by an LLM judge, so we audited that judge with two independent machine instruments before claiming anything from its output. Both were pre-registered with explicit predictions; three of the four predictions were refuted, and we release them alongside the results.

\subsection{Six-Model Cross-Family Panel}

One hundred stratified verdicts were re-judged by six models spanning five model families, each receiving a byte-identical rubric and the ground-truth evidence, in fresh sessions with memory disabled. The panel agrees with itself substantially more than with our pinned judge: Fleiss $\kappa = 0.834$ across the six panelists and $0.796$ across all seven raters, with a median pairwise $\kappa$ of $0.814$ among panelists but $0.715$ between panelists and the pinned judge. Fifteen of the 21 pairs fall outside $[0.60, 0.75]$.

The disagreement is not diffuse but localized. Against the panel majority, the pinned judge's \textsc{correct}, \textsc{wrong} and \textsc{abstain} verdicts are stable (23/24, 19/20 and 32/32 respectively), while of 24 pinned \textsc{partial} verdicts the panel majority moves 19 elsewhere---13 to \textsc{correct} and 6 to \textsc{abstain}. Inspection of those 19 shows none contains an outdated or wrong indicator; all 19 are merely \emph{incomplete}. The pinned judge is therefore using \textsc{partial} to mean ``incomplete'' where our rubric reserves it for ``contaminated.'' One item is a 3--3 tie with no panel majority and is reported as its own case rather than folded into an agreement tier.

\subsection{Targeted Re-Judge of Every PARTIAL Verdict}

Acting on that diagnosis, all 748 pinned \textsc{partial} verdicts across six systems were re-judged by an independent judge under the same rubric, together with a 60-item stratified control of non-\textsc{partial} verdicts. MemoryOS is excluded because its answer texts were not persisted.

The control reproduces $56/60$ ($93\%$) of non-\textsc{partial} verdicts---$95\%$ of \textsc{correct}, $85\%$ of \textsc{wrong}, $100\%$ of \textsc{abstain}---which is what licenses correcting \textsc{partial} alone. Applying the corrections raises correct rates by $+5.2$ to $+7.0$pp across the six systems, roughly uniformly: FR-Bank $31.2 \to 36.2\%$, FR-Graphiti $22.9 \to 28.1\%$, Mem0 $18.6 \to 25.2\%$, Mem0-mini $27.3 \to 34.3\%$, Memory-R1 $19.6 \to 24.8\%$, A-MEM $18.8 \to 24.8\%$. The cross-system ranking is unchanged on every metric, and the FR-Bank\,--\,Mem0 correct-rate gap narrows from $+12.6$ to $+11.0$pp.

Two caveats travel with this. The all-queries confabulation baseline computed inside the re-judge harness differs slightly from Table~\ref{tab:allqueries} (e.g.\ FR-Bank $13.2\%$ vs.\ $13.0\%$) because the two harnesses treat the confabulation flag on re-scored items differently; the paper's tables use the convention stated in Appendix~\ref{app:allqueries} throughout. And the two audits disagree with each other on magnitude: on the 24 items both examined, the panel majority moved 19 off \textsc{partial} while the single independent judge moved 11, a 33pp discrepancy. They agree on the direction and the location of the bias but not its size, so the panel's item-level agreement figures should not be quoted as a general property of the judge.

\subsection{What Remains}

The pinned-judge numbers remain the numbers of record throughout this paper, because they are the ones every system was scored under identically. What the audits establish is that absolute rates are judge-relative and cross-system comparisons are robust. A stratified human-labeled slice---100 verdicts across the four verdict classes and nine attack vectors, two non-author annotators blind to system identity, reporting raw agreement, Cohen's $\kappa$, and a per-class confusion table, prioritizing items on which the machine judges disagree---is planned and has not yet been run. It is listed as the first limitation in \S\ref{sec:discussion}.

\section{Metadata-Noise Sensitivity Across the Full Basis}
\label{app:noise}

The deterministic policy layer is only as reliable as the LLM-generated metadata beneath it. Appendix~\ref{app:perturbation} perturbs category labels alone; this study perturbs each field of the basis independently.

\textbf{Harness and scope.} Noise is injected at ingestion and lifecycle states are recomputed by a state replay verified to reproduce \texttt{is\_active} for $18{,}936/18{,}936$ edges; retrieval is then rerun on all 516 questions. The $0\%$-noise control reproduces the unperturbed top-10 exactly, and the harness reproduces the published Appendix~\ref{app:perturbation} anchors to $|\Delta| < 10^{-5}$ before any new condition is run. Two scope limits are declared rather than discovered: category-forced retrieval, multi-hop expansion and LLM distillation are excluded because they require API calls, so every row is a \emph{lower bound} on full-pipeline impact---and the category row is the one most likely understated, since category-forced retrieval is the component whose input is the category label. The staleness and support measures are lexical proxies computed against ground-truth strings, not the LLM judge used for the headline numbers. Mean $\pm$ s.d.\ over three seeds.

\begin{table}[h]
\caption{Retrieval stability under independent metadata corruption. Jaccard@10 is between the unperturbed and perturbed top-10 sets; $\Delta$ staleness is the change in the fraction of questions whose top-10 contains an edge the uncorrupted metadata marks superseded, expired, retracted or shadowed (baseline $8.1\%$).}
\label{tab:noise-grid}
\centering
\small
\begin{tabular}{@{}l c c c@{}}
\toprule
\textbf{Field / sub-condition} & \textbf{Noise rate} & \textbf{Jaccard@10} & \textbf{$\Delta$ staleness} \\
\midrule
Category label (full replay)      & 5\%  & $0.951 \pm 0.002$ & $+0.3$pp \\
Category label (full replay)      & 10\% & $0.917 \pm 0.002$ & $+0.4$pp \\
Category label (full replay)      & 20\% & $0.854 \pm 0.006$ & $+1.4$pp \\
Category label (full replay)      & 50\% & $0.733 \pm 0.002$ & $+2.7$pp \\
\midrule
Slot key, merge corruption        & 5\%  & $0.859 \pm 0.003$ & $+19.8$pp \\
Slot key, merge corruption        & 10\% & $\mathbf{0.764 \pm 0.004}$ & $\mathbf{+32.2}$pp \\
Slot key, merge corruption        & 20\% & $0.650 \pm 0.002$ & $+52.7$pp \\
Slot key, split corruption        & 5\%  & $0.901 \pm 0.002$ & $+24.4$pp \\
Slot key, split corruption        & 10\% & $\mathbf{0.837 \pm 0.005}$ & $\mathbf{+38.6}$pp \\
Slot key, split corruption        & 20\% & $0.740 \pm 0.000$ & $+60.1$pp \\
\midrule
Event time, $\pm 7$-day shift     & 10\% & $0.999 \pm 0.000$ & $+0.1$pp \\
Event time, $\pm 90$-day shift    & 10\% & $0.998 \pm 0.001$ & $+0.3$pp \\
Event time, deleted               & 10\% & $0.993 \pm 0.002$ & $+2.6$pp \\
Event time, deleted               & 20\% & $0.988 \pm 0.001$ & $+5.6$pp \\
\midrule
Retraction, false positive        & 1\%  & $0.971 \pm 0.003$ & $-0.1$pp \\
Retraction, false negative (all 31 restored) & exhaustive & $0.995$ & $+0.8$pp \\
\bottomrule
\end{tabular}
\end{table}

\textbf{Reading the grid.} Sensitivity concentrates on slot keys---precisely the field that \S\ref{sec:attribution} identifies as carrying correctness---and their corruption damages \emph{set composition} rather than ranking (Kendall $\tau \approx 0.997$ while Jaccard falls to $0.76$). Category labels, which carry the calibration benefit, tolerate a 10\% uniform flip at $+0.4$pp of staleness and even a 50\% flip---equivalent to random assignment within the confused pairs---at $+2.7$pp, with the damage flowing through supersession semantics rather than scoring. Event-time anchors are near-inert under realistic corruption on this corpus because most anchored events lie far from the expiry boundary; only $3.1\%$ fall within seven days of it. Retraction states show a prevalence/potency split: per corrupted edge they are the most potent field measured, but only 31 of $18{,}936$ edges carry retrieval-suppressing retraction, so aggregate exposure is minimal.

\textbf{The binding constraint on retraction is recall, not noise.} Auditing the corpus against ground truth: the 31 retraction-carrying edges are detection markers, of which 26 map to the 40 scripted retraction events (19 distinct events, some firing more than once) and 5 are incidental. In event units the ingestion detector fires on \textbf{19 of 40} scripted retractions, so extraction recall on retraction language is the first bottleneck on selective forgetting. It is not the only one: because a plan is stored across many sibling slot keys, retraction and supersession deactivate the primary key while sibling edges asserting the plan remain active, so only \textbf{2 of 40} events are cleanly suppressed at the metadata level---matching the $5\%$ AV7 retrieval-pass rate---and 38 of 40 leak at least one live edge. FR-Bank's $37.5\%$ end-to-end AV7 correctness, the highest of any evaluated system, is therefore recovered downstream, where the generator reconciles the co-retrieved plan and its cancellation, rather than by clean metadata exclusion. The selective-forgetting bottleneck is thus twofold---retraction-detection recall and slot-key fragmentation---and because the downstream policy is deterministic, both failure modes are localized and inspectable.

\section{Changes in This Revision}
\label{app:revision-changes}

This appendix lists every number that differs from the submitted version, so that the revision can be audited against it directly. All corrections were found by our own post-submission audit or during the recomputations requested in review, and all corrected values are recomputed from raw per-question records. Section~\ref{sec:attribution}, \S\ref{sec:beam}, \S\ref{sec:granularity}, \S\ref{sec:robustness} and Appendices~\ref{app:untyped-arm}--\ref{app:noise} are new.

\begin{table}[h]
\caption{Numerical corrections relative to the submitted version.}
\label{tab:revision-changes}
\centering
\small
\setlength{\tabcolsep}{4pt}
\begin{tabular}{@{}L{3.5cm} L{4.2cm} L{5.6cm}@{}}
\toprule
\textbf{Item} & \textbf{Submitted $\to$ revised} & \textbf{Reason} \\
\midrule
FR-Bank LifecycleBench pass rate (16 sites) & $77.3\% \to 76.9\%$, staleness $15.3\% \to 15.5\%$ & The $77.3\%$ run was superseded by a re-scored run whose values already underlie every confidence interval and pairwise test in Appendix~\ref{app:bootstrap-cis} and the released artifact. Ranking and every comparative conclusion unchanged; still $+16$pp over Mem0. \\
LongMemEval knowledge-update & $+6.4$pp \emph{withdrawn} & The $87.2\%$ / $80.8\%$ pair traces to no surviving artifact. Artifact-backed value on the 317-matched subset is $-6.4$pp, a judge-criterion effect (Table~\ref{tab:lme-317-matched}). \\
Table~\ref{tab:longmemeval-pertype-ablation} subset label & now labeled 234-question POC subset & The submitted table mixed a ``317-matched'' label with 234-question POC values; the two subsets are now separate tables. \\
Table~\ref{tab:kimi}, A-MEM (GPT-5.4) row & Partial $19.0 \to 27.5$, Wrong $20.0 \to 28.3$, Abstain $42.2 \to 25.4$ & Three transcription errors, forced by the row's own Correct and Safe cells. Correct, Confab and Safe were already right; nothing downstream changes. \\
AV3 leadership (5 sites) & ``Memory-R1 leads AV3'' $\to$ FR-Bank leads at $97\%$ & The comparison was against FR-Graphiti's $78\%$. Memory-R1 leads no attack vector. \\
FR MRR & $0.478 \to 0.830$ where FR-Bank is meant & $0.478$ is FR-Graphiti's; FR-Bank's $0.830$ exceeds Memory-R1's $0.816$. \\
Bare ``FR'' (5 contradictions) & disambiguated to FR-Bank / FR-Graphiti & The same label denoted both variants, producing gaps of up to 35pp between statements about ``FR.'' \\
FR-Graphiti ablation deltas & $-12.2 / -6.2 / -6.4 \to +12.2 / +6.2 / +6.4$ & Sign inversion; the per-AV table already showed positive gains. \\
Mem0 extractor upgrade & $+10$pp $\to +6.1$pp & $67.1 - 61.0$. \\
Category-aware routing & $+5$pp $\to +6$pp & Consistent with Table~\ref{tab:lifecyclebench-results} and the bootstrap estimate. \\
``$2\times$ lower staleness'' & $\to 1.76\times$ & $15.5$ vs.\ $27$. \\
Table~\ref{tab:e2e-av} & FR and Mem0 columns replaced; all 9 AVs now shown & Both columns matched no current run; the Memory-R1 column was already correct. \\
Table~\ref{tab:e2e-av-amem-mem0mini} & per-AV Correct cells replaced & They did not weight-average to their own (correct) aggregates. \\
AV7 confabulation prose & ``FR reduces this to $32.5\%$'' $\to 33.3\%$ & $32.5\%$ was MemoryOS's abstain rate. \\
Appendix~\ref{app:frbank} corpus size & $5{,}344 / 3{,}892 / 133.6 \to 18{,}936 / 12{,}968 / 324.2$ & Stale proof-of-concept snapshot. \\
$\lambda_{\textsc{other}}$ & $0.0030 \to 0.0050$ in Table~\ref{tab:decay-rates} & Two tables disagreed; $0.0050$ is the deployed and released value. Zero result changes. \\
Perturbation flip counts & denominator now stated & $\approx 910$ / $\approx 4{,}400$ are $10\%$ / $50\%$ of $8{,}913$ affected edge-instances, not of the $6{,}180$ unique eligible entries. \\
Polytope cell count & ``45 off-diagonal'' $\to$ 45 populated (5 diagonal, 40 off-diagonal) & The artifact records 45 cells in total. \\
$\alpha_F^{\max} = 0.065$, $\alpha_L^{\min} = 0.138$ & now attributed to FR-Graphiti & FR-Bank's artifact values are $1.7\times10^{-6}$ and $\approx 1.0$; the medians are from the FR-Graphiti-era analysis. \\
\bottomrule
\end{tabular}
\end{table}

One further disclosure. During the discussion period we quoted a figure of $12.5\%$ for AV8 numeric preservation under a uniform blending weight. That figure appears only in prose in earlier drafts and is reproduced by no released artifact; the two artifacts that come closest to testing it report $81.0\%$ (the uniform ontology-ablation arm over 516 questions) and a degenerate Phase-2-era value. We therefore withdraw it and make the calibration argument in \S\ref{sec:theory} entirely from artifact-backed quantities: $99.8\%$ of $1{,}566$ cross-pairs infeasible, zero feasible cells among the 45 populated, and $\min \alpha_F^{\max} = 1.7\times10^{-6}$ against $\max \alpha_L^{\min} \approx 1.0$.

A targeted audit of suspected items produced the corrections above; a stratified random audit of roughly 60 additional measured cells---the 45 per-AV cells of Table~\ref{tab:mem0-av}, the 11 decay-rate cells, the LongMemEval variance bounds, and the latency table---found no mismatches. The errors are confined to the editing layer: stale runs transcribed into tables, sign and label slips, and aggregate typos. The measurement layer, meaning the per-question records, banks and configurations, is unaffected, and all raw records are released.

\section*{NeurIPS Paper Checklist}

\begin{enumerate}

\item {\bf Claims}
    \item[] Question: Do the main claims made in the abstract and introduction accurately reflect the paper's contributions and scope?
    \item[] Answer: \answerYes{}
    \item[] Justification: All main claims are supported by experiments described in \S\ref{sec:evaluation} and theoretical analysis in \S\ref{sec:theory} with proofs in Appendix~\ref{app:proofs}.

\item {\bf Limitations}
    \item[] Question: Does the paper discuss the limitations of the work performed by the authors?
    \item[] Answer: \answerYes{}
    \item[] Justification: Limitations are discussed in \S\ref{sec:discussion} (rate-calibration generalization, soft-supersession failure mode under high-commitment generators, classifier validation via LLM judges, and the scope of the ontology's value across distant vs.\ nearby category distinctions).

\item {\bf Theory assumptions and proofs}
    \item[] Question: For each theoretical result, does the paper provide the full set of assumptions and a complete (and correct) proof?
    \item[] Answer: \answerYes{}
    \item[] Justification: All theoretical results (Theorems~1--4, the lifecycle non-identifiability and individual-necessity results, and the staleness-decomposition propositions) include formal statements and complete proofs in Appendix~\ref{app:proofs}.

\item {\bf Experimental result reproducibility}
    \item[] Question: Does the paper fully disclose all the information needed to reproduce the main experimental results?
    \item[] Answer: \answerYes{}
    \item[] Justification: Benchmark data (40 personas, 1{,}400 sessions, 516 questions, ground truth) and FR-Bank code are released as supplementary material and will be made publicly available upon publication. All model versions, hyperparameters, and evaluation protocols are documented in Appendix~\ref{app:benchmark-ext}, Appendix~\ref{app:params}, and Appendix~\ref{app:decay}.

\item {\bf Open access to data and code}
    \item[] Question: Does the paper provide open access to the data and code, with sufficient instructions to faithfully reproduce the main experimental results?
    \item[] Answer: \answerYes{}
    \item[] Justification: LifecycleBench (questions, personas, ground truth), FR-Bank source code, all evaluation results, and run logs are publicly archived at \href{https://doi.org/10.5281/zenodo.20067778}{\texttt{https://doi.org/10.5281/zenodo.20067778}} (DOI: \texttt{10.5281/zenodo.20067778}) and bundled with the supplementary material. Run scripts and exact configurations are included.

\item {\bf Experimental setting/details}
    \item[] Question: Does the paper specify all the training and test details (e.g., data splits, hyperparameters, optimizer, type of compute) necessary to understand the results?
    \item[] Answer: \answerYes{}
    \item[] Justification: Evaluation methodology, model versions, and hyperparameters are fully described in Appendices~\ref{app:benchmark-ext}, \ref{app:mem0method}, \ref{app:mr1method}, \ref{app:e2e}, and~\ref{app:kimi}.

\item {\bf Experiment statistical significance}
    \item[] Question: Does the paper report error bars suitably and correctly defined or other appropriate information about the statistical significance of the experiments?
    \item[] Answer: \answerYes{}
    \item[] Justification: Bootstrap confidence intervals and pairwise significance tests are reported in Appendix~\ref{app:bootstrap-cis}; LongMemEval-S run-to-run variance is reported in Table~\ref{tab:longmemeval-s-variance}; the 20/20 persona holdout split (Fisher's $p=1.00$) is reported in \S\ref{sec:discussion}.

\item {\bf Experiments compute resources}
    \item[] Question: For each experiment, does the paper provide sufficient information on the computer resources (type of compute workers, memory, time of execution) needed to reproduce the experiments?
    \item[] Answer: \answerYes{}
    \item[] Justification: All evaluations are API-driven and require no GPU training. The deterministic lifecycle layer's latency (median $47\,\mu\mathrm{s}$, single-threaded Python on AMD Zen~3) is reported in \S\ref{sec:architecture} and Appendix~\ref{app:architecture-ext}; per-system API call counts and Kimi K2.5 cross-generator API cost (under \$10) are documented in Appendix~\ref{app:kimi}.

\item {\bf Code of ethics}
    \item[] Question: Does the research conducted in the paper conform, in every respect, with the NeurIPS Code of Ethics?
    \item[] Answer: \answerYes{}
    \item[] Justification: All evaluations use synthetic personas (no human subjects), publicly available APIs, and openly licensed competitor code (Mem0 Apache 2.0, A-MEM/Memory-R1/MemoryOS as published). No private or personally identifiable data is used.

\item {\bf Broader impacts}
    \item[] Question: Does the paper discuss both potential positive societal impacts and negative societal impacts of the work performed?
    \item[] Answer: \answerYes{}
    \item[] Justification: Memory systems that retain personal information raise privacy concerns; the interpretability and selective-forgetting capabilities of FR provide users with control over their data. Privacy, consent, and safety considerations are discussed in \S\ref{sec:discussion} and Appendix~\ref{app:privacy-safety}.

\item {\bf Safeguards}
    \item[] Question: Does the paper describe safeguards that have been put in place for responsible release of data or models that have a high risk for misuse?
    \item[] Answer: \answerYes{}
    \item[] Justification: The released benchmark contains only synthetic personas. Prompt-injection detection and confidence-degradation mechanisms in the ingestion classifier are described in Appendix~\ref{app:architecture-ext}; production deployment recommendations regarding consent and verifiable deletion are in Appendix~\ref{app:privacy-safety}.

\item {\bf Licenses for existing assets}
    \item[] Question: Are the creators or original owners of assets (e.g., code, data, models), used in the paper, properly credited and are the license and terms of use explicitly mentioned and properly respected?
    \item[] Answer: \answerYes{}
    \item[] Justification: Mem0 (Apache 2.0), A-MEM, Memory-R1, MemoryOS, and Zep/Graphiti are cited and credited; Kimi K2.5 is released under a Modified MIT License; OpenAI and Anthropic models are accessed via publicly available APIs. Embedding models (\texttt{text-embedding-3-small}) are similarly cited.

\item {\bf New assets}
    \item[] Question: Are new assets introduced in the paper well documented and is the documentation provided alongside the assets?
    \item[] Answer: \answerYes{}
    \item[] Justification: LifecycleBench (516 questions, 40 personas, structured YAML ground truth, 9 attack vectors) and FR-Bank source code are released as supplementary material with documentation, and will be made publicly available under an open license upon publication.

\item {\bf Crowdsourcing and research with human subjects}
    \item[] Question: For crowdsourcing experiments and research with human subjects, does the paper include the full text of instructions given to participants and screenshots, if applicable, as well as details about compensation (if any)?
    \item[] Answer: \answerNA{}
    \item[] Justification: No crowdsourcing or research with human subjects was performed. All personas are synthetic; all classifier-agreement audits use LLM judges, not human annotators.

\item {\bf Institutional review board (IRB) approvals or equivalent for research with human subjects}
    \item[] Question: Does the paper describe potential risks incurred by study participants, whether such risks were disclosed to the subjects, and whether Institutional Review Board (IRB) approvals (or an equivalent approval/review based on the requirements of your country or institution) were obtained?
    \item[] Answer: \answerNA{}
    \item[] Justification: No human subjects were involved.

\item {\bf Declaration of LLM usage}
    \item[] Question: Does the paper describe the usage of LLMs if it is an important, original, or non-standard component of the core method development in this research?
    \item[] Answer: \answerYes{}
    \item[] Justification: The Fortunate Recall pipeline uses LLMs at ingestion (one \texttt{gpt-4.1-mini} call per turn for entity/edge extraction and behavioral classification) and a single lightweight LLM call at retrieval (gpt-4.1-mini distillation). The LLM/lifecycle-math boundary is documented explicitly in \S\ref{sec:architecture} and Figure~\ref{fig:architecture}, and downstream answer generation uses GPT-5.4 (primary) and Kimi K2.5 (cross-generator validation).

\end{enumerate}

\end{document}